\documentclass{article}
\PassOptionsToPackage{numbers, compress}{natbib}
\usepackage[final]{neurips_2021}

\usepackage[utf8]{inputenc} 
\usepackage[T1]{fontenc}    
\usepackage{hyperref}       
\usepackage{url}            
\usepackage{booktabs}       
\usepackage{amsfonts}       
\usepackage{nicefrac}       
\usepackage{microtype}      
\usepackage{xcolor}         

\usepackage{bigints}
\usepackage{amsfonts,amssymb,amsthm,mathtools}
\usepackage{xfrac}
\usepackage{stmaryrd}

\newcommand{\bx}{\boldsymbol{x}}
\newcommand{\bm}{\boldsymbol{m}}
\newcommand{\bn}{\boldsymbol{n}}
\newcommand{\bv}{\boldsymbol{v}}

\newcommand{\bQ}{\mathbf{Q}}
\newcommand{\bc}{\boldsymbol{c}}
\newcommand{\bp}{\boldsymbol{p}}
\newcommand{\bq}{\boldsymbol{q}}

\newcommand{\bfQ}{\mathbf{Q}}

\newcommand{\Prob}{\mathbb{P}}
\newcommand{\R}{\mathbb{R}}
\newcommand{\E}{\mathbb{E}}

\newcommand{\Exp}{\mathbb{E}}

\newcommand{\bbR}{\mathbb{R}}
\newcommand{\bbN}{\mathbb{N}}

\newcommand{\class}[1]{\mathcal{#1}}
\newcommand{\cA}{\class{A}}
\newcommand{\cB}{\class{B}}
\newcommand{\cC}{\class{C}}
\newcommand{\cS}{\class{S}}
\newcommand{\cX}{\class{X}}

\newcommand{\eqdef}{\vcentcolon=}

\newcommand{\parent}[1]{\left( #1 \right)}
\newcommand{\ens}[1]{\left\{ #1\right\}}
\newcommand{\norm}[1]{\left\lVert#1\right\rVert}

\newcommand{\abs}[1]{\left\lvert #1\right\rvert}

\newcommand{\enscond}[2]{\left\{ #1 \, : \, #2\right\}}
\newcommand{\enscondb}[2]{\bigl\{ #1 \, : \, #2\bigr\}}
\newcommand{\ind}[1]{\mathbb{I}{\left\{#1\right\}}}
\newcommand{\scalar}[2]{\left\langle #1, #2\right\rangle}

\newcommand{\pa}[1]{\left(#1\right)}
\newcommand{\ac}[1]{\left\{#1\right\}}
\newcommand{\cro}[1]{\left[#1\right]}

\DeclareMathOperator*{\supp}{supp}

\DeclareMathOperator{\Proj}{Proj}

\DeclareMathOperator*{\argmin}{arg\,min}
\DeclareMathOperator*{\argmax}{arg\,max}

\DeclareMathOperator{\TV}{TV}

\newtheorem{theorem}{Theorem}
\newtheorem{definition}{Definition}
\newtheorem{proposition}{Proposition}
\newtheorem{lemma}{Lemma}

\newtheorem{assumption}{Assumption}
\newtheorem{remark}{Remark}

\newtheorem*{theorem*}{Theorem}
\newtheorem*{proposition*}{Proposition}
\newtheorem*{lemma*}{Lemma}

\renewcommand{\leq}{\leqslant}
\renewcommand{\geq}{\geqslant}
\renewcommand{\enspace}{\,}
\renewcommand{\d}{\,\mathrm{d}}
\renewcommand{\P}{\mathbb{P}}
\renewcommand{\bar}{\overline}
\renewcommand{\epsilon}{\varepsilon}

\newcommand{\bmcal}{\bm_{\mbox{\tiny\textrm{cal}}}}
\newcommand{\bmcalT}{\bar{\bm}_{\mbox{\tiny\textrm{cal}},T}}
\newcommand{\tbm}{\widetilde{\bm}}
\newcommand{\bmgcal}{\bm_{\mbox{\tiny\textrm{gr-cal}}}}
\newcommand{\bmgcalT}{\bar{\bm}_{\mbox{\tiny\textrm{gr-cal}},T}}

\newcommand{\bmgrpay}{\bm_{\mbox{\tiny\textrm{eq-pay}}}}
\newcommand{\tbmgcal}{\widetilde{\bm}_{\mbox{\tiny\textrm{gr-cal}}}}
\newcommand{\cCcal}{\cC_{\mbox{\tiny\textrm{cal}}}}
\newcommand{\cCgcal}{\cC_{\mbox{\tiny\textrm{gr-cal}}}}
\newcommand{\bmreg}{\bm_{\mbox{\tiny\textrm{reg}}}}
\newcommand{\bmgreg}{\bm_{\mbox{\tiny\textrm{gr-reg}}}}
\newcommand{\cCgreg}{\cC_{\mbox{\tiny\textrm{gr-reg}}}}
\newcommand{\cCreg}{\cC_{\mbox{\tiny\textrm{reg}}}}
\newcommand{\bmdp}{\bm_{\mbox{\tiny\textrm{DP}}}}
\newcommand{\tbmdp}{\widetilde{\bm}_{\mbox{\tiny\textrm{DP}}}}
\newcommand{\cCdp}{\cC_{\mbox{\tiny\textrm{DP}}}}
\newcommand{\cCgrpay}{\cC_{\mbox{\tiny\textrm{eq-pay}}}}
\newcommand{\dirac}{\text{dirac}}

\newcommand{\gr}{{\mbox{\tiny\textrm{gr}}}}
\newcommand{\pv}[2]{\langle {#1}, \, {#2} \rangle}
\newcommand{\bpv}[2]{\bigl\langle {#1}, \, {#2} \bigr\rangle}
\renewcommand{\epsilon}{\varepsilon}
\newcommand{\cCgcaleps}{\cC_{\mbox{\tiny\textrm{gr-cal}}}^{\varepsilon}}
\newcommand{\cCdpdelta}{\cC_{\mbox{\tiny\textrm{DP}}}^{\delta}}
\newcommand{\tcCgcaleps}{\widetilde{\cC}_{\mbox{\tiny\textrm{gr-cal}}}^{\varepsilon}}
\newcommand{\tcCdpdelta}{\widetilde{\cC}_{\mbox{\tiny\textrm{DP}}}^{\delta}}
\newcommand{\hcCgcaleps}{\widehat{\cC}_{\mbox{\tiny\textrm{gr-cal}}}^{\varepsilon}}
\newcommand{\hcCdpdelta}{\widehat{\cC}_{\mbox{\tiny\textrm{DP}}}^{\delta}}
\newcommand{\hgamma}{\hat{\gamma}}
\newcommand{\hM}{\widehat{M}}
\newcommand{\bfP}{\mathbf{P}}

\usepackage{thmtools}
\usepackage{thm-restate}
\usepackage{enumitem}
\usepackage[most]{tcolorbox}
\newtcolorbox[auto counter,number within=section]{protocol}[1][]{
  enhanced,
  breakable,
  fonttitle=\scshape,
  title={Protocol \thetcbcounter},
  #1
}

\newcommand{\eparagraph}[1]{\textit{#1}~~}

    \newlength{\leftstackrelawd}
    \newlength{\leftstackrelbwd}
    \def\leftstackrel#1#2{\settowidth{\leftstackrelawd}%
    {${{}^{#1}}$}\settowidth{\leftstackrelbwd}{$#2$}%
    \addtolength{\leftstackrelawd}{-\leftstackrelbwd}%
    \leavevmode\ifthenelse{\lengthtest{\leftstackrelawd>0pt}}%
    {\kern-.5\leftstackrelawd}{}\mathrel{\mathop{#2}\limits^{#1}}}

\newcommand{\titre}{A Unified Approach to Fair Online Learning\\ via Blackwell Approachability}
\title{\titre}

\author{Evgenii Chzhen \qquad Christophe Giraud \qquad Gilles Stoltz \\
  \ \\
  Universit{\'e} Paris-Saclay, CNRS, Laboratoire de mathématiques d'Orsay, 91405, Orsay, France \\
  \texttt{\{evgenii.chzhen,\,\,christophe.giraud,\,\,gilles.stoltz\}\,\,@universite-paris-saclay.fr}
}

\begin{document}

\maketitle

\begin{abstract}
We provide a setting and a general approach to fair online learning with stochastic sensitive and non-sensitive contexts. The setting is a repeated game between the Player and Nature, where at each stage both pick actions based on the contexts. Inspired by the notion of unawareness, we assume that the Player can only access the non-sensitive context before making a decision, while we discuss both cases of Nature accessing the sensitive contexts and Nature unaware of the sensitive contexts. Adapting Blackwell's approachability theory to handle the case of an unknown contexts' distribution, we provide a general necessary and sufficient condition for learning objectives to be compatible with some fairness constraints. This condition is instantiated on (group-wise) no-regret and (group-wise) calibration objectives, and on demographic parity as an additional constraint. When the objective is not compatible with the constraint, the provided framework permits to characterise the optimal trade-off between the two.
\end{abstract}

\section{Introduction}
\label{sec:introduction}

Classically, the goal of the decision maker in sequential environment is purely performance driven --- she wants to obtain as high reward as if she has had a complete information about the environment. In contrast, algorithmic fairness shifts the attention from the performance-driven behavior by taking into account additional ethical considerations. The latter is often formalized via the notion of fairness constraint~\citep{dwork2012fairness,Pleiss_Raghavan_Wu_Kleinberg_Weinberger17,Chouldechova17} on the decision maker's strategies. The goal of this work is to bring to light Blackwell's approachability theory as a suitable theoretical formalism for fair online learning under \emph{group fairness} constraints. The appealing feature of this theory is two-fold: first, it gives explicit criteria when learning is possible; second, if this criteria is met, it comes with an explicit strategy.

It is well known that Blackwell's approachability theory may be used
to characterize online learning problems that are tractable and to design strategies
to solve them---for instance, for no-regret learning or calibration.
Extensive references to such uses may be found in \citet{CBL06}, \citet{PerchetPhD} and \citet{ABH11}.
Actually, as noted by the latter two references,
no-regret learning, calibration, and approachability imply each other in some sense.
The main first achievement of this article is to extend this use to online
learning under fairness constraints.
This idea, though natural and intuitive, requires some extensions to Blackwell's approachability theory,
like ignoring the target set and having to estimate it.

\emph{Related works in fair online learning.}
Several frameworks have been proposed to tackle various problems of fairness arising in online learning. \citet{Blum_Gunasekar_Lykouris_Srebro18} consider the problem of online prediction with experts and define fairness via (approximate) equality of average payoffs. \citet{Hebert-Johnson_Kim_Reingold_Rothblum18,Gupta_Jung_Noarov_Pai_Roth21} consider the problem of group-wise calibration. (In passing, we may note that \citet{Gupta_Jung_Noarov_Pai_Roth21} consider some techniques with a flavor of approachability.)
\citet{Bechavod_Ligett_Roth_Waggoner_Wu19} consider the problem of online binary classification with partial feedback and equal opportunity constraint~\cite{hardt2016equality}.
We treat the above works as sources of inspiration; they all differ in the specific (sensitive and non-sensitive)
information that the Player may or may not access before or after taking an action.
We apply the general formalism of approachability theory to give new insights into
online learning under fairness constraints, and approach this goal in a unified (and geometric) way.
In particular, the generality of this formalism allows to derive (im)possibility results nearly effortlessly.
But we also go beyond such a mere compatibility/incompatibility check between the learning objectives and fairness constraints,
and note that approachability theory also gives a clear strategy for the study of trade-offs between incompatible learning objectives
and fairness constraints, which often arise in batch setup~\cite{Chouldechova17}.

\emph{Outline.}
We describe our approachability setting in Section~\ref{sec:setting_approachability_goal_objectives_constraints}
and provide some learning objectives (no-regret and calibration) and fairness constraints (group-wise controls,
demographic parity, equalized average payoffs) that fit our framework. A slight extension of the classical result of~\citet{Blackwell56}
is required and discussed in Section~\ref{sec:blackwell_s_approachability}.
We then support the generality of our framework by deriving (im)possibility results for some objective--constraint pairs in Section~\ref{sec:first_work_outs}.
We also illustrate in Section~\ref{sec:trade-off} how this formalism can be used to derive optimal trade-offs (Pareto frontiers)
between performance and fairness for incompatible objective--constraint pairs;
as an example, we deal with group-wise calibration (studied by~\cite{Hebert-Johnson_Kim_Reingold_Rothblum18,Gupta_Jung_Noarov_Pai_Roth21})
under demographic parity constraint.
For the sake of exposition, we deal in Sections~\ref{sec:setting_approachability_goal_objectives_constraints}--\ref{sec:trade-off}
with stochastic sensitive contexts whose distribution is known;
Section~\ref{sub:approachability_with_unknown_target} explains how to overcome this and develops
a theory of approachability relying on ignoring but estimating the target set.

\emph{Notation.} The Euclidean norm is denoted by $\|\cdot\|$, while the $\ell_1$ norm is denoted by $\|\cdot\|_1$. Given a convex closed set $\cC \subset \bbR^d$, we denote by $\Proj_{\cC}(\cdot)$ the projection operator onto $\cC$ in Euclidean norm.

\section{Fair online learning cast as an approachability problem}
\label{sec:setting_approachability_goal_objectives_constraints}

In this section, we propose a setting for fair online learning
based on approachability---a theory
introduced by \citet{Blackwell56} (see also the more modern expositions by \citet{Perchet13} or \citet{MSZ}).
More precisely, we consider the following repeated game between a Player and Nature,
with stochastic contexts. The existence of these contexts is a (minor) variation on the
classical statement of the approachability problem.

The Player and Nature have respective finite action sets $\cA$ and $\cB$.
The sets of sensitive and non-sensitive contexts are respectively denoted by $\cS$ and $\cX$. The set $\cX$ is a general Borel set,
while $\cS$ is a finite set with cardinality denoted by $|\cS|$. Typical choices are $\cS = \{0,1\}$ and $\cX = \R^m$ for some $m \in \bbN$.
A joint distribution $\bfQ$ on $\cX \times \cS$ is fixed and is unknown to the Player.
Finally, a (bounded) Borel-measurable vector-valued payoff function $\bm : \cA \times \cB \times \cX \times \cS \to \R^d$,
as well as a closed target set $\cC \subseteq \R^d$, are given and known by the Player.

At each round $t \geq 1$ the pair of non-sensitive and sensitive contexts $(x_t, s_t) \sim \bfQ$ is generated independently
from the past. The Player observes only the non-sensitive context $x_t$; while Nature also observes $x_t$, it
may or may not observe the sensitive context $s_t$.
Then, Nature and the Player simultaneously pick (possibly in a randomized fashion) $b_t \in \class{B}$ and
$a_t \in \class{A}$, respectively. The Player finally accesses the obtained reward $\bm(a_t, b_t, x_t, s_t)$ and the sensitive context $s_t$,
while Nature has a more complete monitoring and may observe $a_t$ and $s_t$.
We introduce an observation operation $G$ to indicate whether Nature observes $x_t$ only---i.e.,
$G(x_t,s_t) = x_t$, the case of Nature's unawareness---or whether Nature observes
both contexts---i.e., $G(x_t,s_t) = (x_t,s_t)$, the case of Nature's awareness.

We consider the short-hand notation $\bm_t \eqdef \bm(a_t, b_t, x_t, s_t)$, \vspace{-.1cm}
\begin{align*}
\bar{\bm}_T \eqdef \frac{1}{T}\sum_{t = 1}^T \bm(a_t, b_t, x_t, s_t), \qquad \textrm{and} \qquad
\bar{\bc}_T = \Proj_{\cC}\bigl( \bar{\bm}_T \bigr) =
\argmin_{\bv \in \class{C}} \|\bar{\bm}_T - \bv\|\enspace \vspace{-.25cm}
\end{align*}
for the instantaneous and average payoffs of the player, as well as the Euclidean projection
of the latter onto the closed set $\cC$, respectively. The distance of $\bar{\bm}_T$
to $\cC$ thus equals $d_T \eqdef  \| \bar{\bm}_T - \bar{\bc}_T \|$.
The game protocol is summarized on the next page.

We recall that the Player does not know the context distribution $\bfQ$.

\begin{definition}
A target set $\class{C}$ is called $\bm$--approachable by the Player under the distribution $\bfQ$
if there exists a strategy of the Player such that, for all strategies of the Nature, $\bar{\bm}_T \to \cC$ a.s.
\end{definition}

\begin{protocol}[label=prot:general]
\textbf{Parameters:} Observation operator $G$ for Nature;
distribution $\bfQ$ on $\cX \times \cS$
\medskip \\
\textbf{For} $t = 1, 2, \ldots$
\begin{enumerate}[topsep=0pt,itemsep=-1ex,partopsep=1ex,parsep=1ex]
    \item Contexts $(x_t, s_t)$ are sampled according to $\bfQ$, independently from the past;
    \item Simultaneously,
    \begin{itemize}[topsep=-1ex,itemsep=-1ex,partopsep=1ex,parsep=1ex]
    \item Nature observes $G(x_t, s_t)$ and picks $b_t \in \class{B}$;
    \item the Player observes $x_t$ and picks an action $a_t \in \class{A}$;
    \end{itemize}
    \item The Player observes the reward $\bm(a_t, b_t, x_t, s_t)$ and
    the sensitive context $s_t$, \\ while Nature observes $(a_t, b_t, x_t, s_t)$. \medskip
\end{enumerate}
\textbf{Aim:} The Player wants to ensure that $\bar{\bm}_T \to \cC$ a.s.,
i.e., $d_T = \| \bar{\bm}_T - \bar{\bc}_T \| \to 0$ a.s.
\end{protocol}
\ \\

\begin{remark}[Awareness for the Player]
\label{rk:aware}
We are mostly
interested in a Player unaware of the sensitive contexts $s_t$
(\citet{Gajane_Pechenizkiy17}).
However, the setting above also covers the case of a Player aware of these
contexts: simply consider the lifted non-sensitive contexts $x'_t = (x_t,s_t)$.
\end{remark}

We now describe payoff functions and target sets corresponding to online learning objectives or online fairness constraints.
They may be combined together. For instance, vanilla calibration corresponds below to the
$\bmcal$--approachability of a set $\cCcal$, demographic parity, to the $\bmdp$--approachability of a set $\cCdp$,
so that vanilla calibration under a demographic parity constraint translates into the $(\bmcal, \bmdp)$--approachability
of the product set $\cCcal \times \cCdp$.
We therefore consider each objective and each constraint as some elementary brick, to be combined with one or several other bricks.
We recall that $\cS$ is a finite set and will indicate the cases where we only consider $\cS = \{0,1\}$.

We discuss two objectives: no-regret and approximate calibration, as well as three fairness constraints:
group-wise (per-group) control, demographic parity, and equal average payoffs.

\subsection{Statement of the objectives}
\label{sec:objectives}

For the sake of a more compact exposition, we define the objectives in two
forms: global objectives (the vanilla form of objectives) and group-wise objectives.
We denote $\gamma_s = \P(s_t = s)$, so that $(\gamma_s)_{s \in \cS}$
corresponds to the marginal of $\bfQ$ on $\cS$.

\paragraph{Objective 1: (Vanilla and group-wise) no-regret.}
The definition is based on some payoff function~$r$, possibly taking contexts into account:
at each round $t$, the Player obtains the payoff $r(a_t,b_t,x_t,s_t)$. The aim is to
get, on average, almost as much payoff as the best constant action, all things equal.
The vanilla (average) regret equals
\[
R_{T} = \min_{a \in \cA} \ \frac{1}{T} \sum_{t=1}^T \bigl( r(a_t,b_t,x_t,s_t) - r(a,b_t,x_t,s_t) \bigr)\,,
\]
while the group-wise (average) regret equals
\[
R_{\gr,T} = \min_{s \in \cS} \ \min_{a_s' \in \cA} \ \frac{1}{T} \sum_{t=1}^T \bigl( r(a_t,b_t,x_t,s_t) - r(a_s',b_t,x_t,s_t) \bigr) \ind{s_t = s}\,.
\]
The aim is that $\liminf R_T \geq 0$ a.s. (no-regret) and $\liminf R_{\gr,T} \geq 0$ a.s. (group-wise no-regret), respectively.
We could replace the $1/T$ factor by a $1/(\gamma_s T)$ factor in the definition of $R_{\gr,T}$,
as we will do for the $C_T$ calibration criterion, but given the wish of a non-negative limit,
this is irrelevant.

Denote by $N = |\cA|$ the cardinality of $\cA$.
No-regret corresponds to the $\bmreg$--approachability of $\bigl( [0,+\infty) \bigr)^{N}$,
with the global payoff function $\bmreg(a,b,x,s) = \bigl( r(a,b,x,s) - r(a',b,x,s) \bigr)_{a' \in \cA}$.
We also duplicate $\bmreg$ into the group-wise payoff function
\[
\smash{\bmgreg(a,b,x,s) = \bigl( \bmreg(a,b,x,s) \,\ind{s' = s} \bigr)_{s' \in \cS}\,.}
\]
Group-wise no-regret then corresponds to the $\bmgreg$--approachability of
$\cCgreg = \bigl( [0,+\infty) \bigr)^{N |\cS|}$.

\paragraph{Objective 2: Approximate (vanilla or group-wise) calibration.}
Online calibration was first solved by \citet{FV98} and \citet{F99};
see the monograph by \citet[Section 4.8]{CBL06} for references
to other solutions and extensions.
For simplicity, we focus on binary outcomes $b_t \in \{0,1\}$
and ask the Player to provide at each round forecasts $a_t$ in $[0,1]$,
and even in a discretization of $[0,1]$ based on a fixed number
$N \geq 2$ of points:
\[
\cA = \bigl\{ a^{(k)} \eqdef (k-1/2)/N, \ \ k \in \{1,\ldots,N\} \bigr\}\,.
\]
Each $x \in [0,1]$ can be approximated by some $a^{(k)} \in \cA$ with $|x-a^{(k)}| \leq 1/(2N)$.
At each round, the Player picks $k_t \in \{1,\ldots,N\}$ and forecasts $a_t = a^{(k_t)}$.
The action set $\cA$ can thus be identified with $\{1,\ldots,N\}$.

This problem is actually called $1/N$--calibration or approximate calibration.
The global (vanilla) form of the criterion reads
\[
C_T = \sum_{k=1}^N
\abs{\frac{1}{T} \sum_{t = 1}^T \bigl( a^{(k)} - b_t \bigr) \,\ind{k_t = k}}\enspace,
\]
while the approximate group-wise calibration criterion is defined as
\[
C_{\gr,T} =
\sum_{s \in \cS} \, \sum_{k=1}^N
\abs{\frac{1}{\gamma_s T} \sum_{t = 1}^T \bigl( a^{(k)} - b_t \bigr) \,\ind{k_t = k} \,\ind{s_t = s}}\,.
\]
The aim is that $\limsup C_T \leq 1/N$ a.s.\ or $\limsup C_{\gr,T} \leq 1/N$, respectively.
Note that unlike vanilla calibration, its group-wise version requires to be calibrated on each sensitive attribute $s \in \cS$. In particular, the classical $1/T$ factor is replaced by $1 / (\gamma_s T)$, the expected number of appearances of $s_t = s$ for $t = 1, \ldots, T$.

\citet{Mannor_Stoltz10} and \citet{ABH11}
rewrote the problem of approximate calibration as an approachability problem as follows:
introduce the global payoff function
\[
\bmcal(k,b) = \bigl( (a^{(1)} - b) \,\ind{k=1}, \,\, \ldots, \,\, (a^{(N)} - b) \,\ind{k=N} \bigr)\enspace,
\]
and duplicate it into the group-wise payoff function as follows:
\[
\bmgcal(k,b,s) = \bigl( \bmcal(k, b) \,\ind{s = s'} / \gamma_{s'} \bigr)_{s' \in \cS}\enspace.
\]
The calibration criteria $C_{T}$ and $C_{\gr,T}$ can now be rewritten as the $\ell^1$--norms
of the average payoff vectors $\bmcalT$ and $\bmgcalT$.
Approximate vanilla calibration thus
corresponds to the $\bmcal$--approachability of
$\cCcal = \bigl\{ \bv \in \R^{N} : \ \| \bv \|_1 \leq 1/N \bigr\}$,
while approximate group-wise calibration
corresponds to the $\bmgcal$--approachability of
$\smash{\cCgcal = \bigl\{ \bv \in \R^{N |\cS|} : \ \| \bv \|_1 \leq 1/N \bigr\}}$.

Note that non-sensitive contexts play no role in the calibration objectives,
but the Player can (and \emph{must})
leverage these non-sensitive contexts to possibly infer sensitive contexts
when handling group-wise calibration.

\subsection{Statement of the fairness constraints}
\label{sec:constraints}

\paragraph{Fairness constraint 1: Group-wise objectives.}
We already considered possibly group-wise objectives above and
Section~\ref{sec:first_work_outs} will show that
handling them is already a challenge in our setting
where the Player is unaware of the sensitive contexts.

\paragraph{Fairness constraint 2: Demographic parity.}
We will consider it only in the setting of approximate calibration
and further restrict our attention to the case of two groups: $\cS = \{0,1\}$.
The demographic parity criterion measures the difference between the average forecasts issued
for the two groups:
\[
D_T = \abs{\frac{1}{\gamma_0 T} \sum_{t = 1}^T a_t \, \ind{s_t = 0} - \frac{1}{\gamma_1 T}\sum_{t = 1}^T a_t \, \ind{s_t = 1}}\,.
\]
Given the discretization used, the wish is that $\limsup D_T \leq 1/N$.
Abiding by a demographic parity constraint is equivalent to $\bmdp$--approaching
$\cCdp = \bigl\{ (u,v) \in \R^2 : \ |u-v| \leq 1/N \bigr\}$, where
\[
\bmdp(k,s) = \bigl( {a^{(k)} \, \ind{s = 0}}/{\gamma_0}, \,\, {a^{(k)} \, \ind{s = 1}}/{\gamma_1} \bigr)\enspace.
\]

\paragraph{Fairness constraint 3: Equalized average payoffs.}
This criterion is to be combined with a no-regret criterion; in particular,
a base payoff function $r$ is considered. We restrict our attention
to the case of two groups, $\cS = \{0,1\}$, and measure
the difference of average payoffs:
\begin{align*}
    P_T = \abs{\frac{1}{\gamma_0 T}\sum_{t = 1}^T r(a_t,b_t,x_t,s_t) \, \ind{s_t = 0} - \frac{1}{\gamma_1 T}\sum_{t = 1}^T r(a_t,b_t,x_t,s_t) \, \ind{s_t = 1}}\enspace.
\end{align*}
Ensuring $\limsup P_T \leq \epsilon$ corresponds to
$\bmgrpay$--approaching $\cCgrpay = \bigl\{ (u,v) \in \R^2 : \ |u - v| \leq \epsilon \bigr\}$,
where
\[
\bmgrpay(a, b, x, s) =
\bigl( r(a,b,x, 0) \, \ind{s = 0} /\gamma_0, \,\, r(a,b,x, 1) \, \ind{s = 1} / \gamma_1 \bigr)\,.
\]

\begin{remark}
Note that in this general form, the equality of average payoffs encompasses the demographic parity constraint. Indeed, the latter is obtained by setting $r(a, b, x, s) = a$ and $\epsilon = 1/N$.
\end{remark}

\subsection{Summary table}
\label{sec:summarytable}
The table below gives a summary of different criteria and associated pairs of payoff function and target set. We remark that some of the payoff functions depend on the marginals $(\gamma_s)_{s \in \cS}$. Meanwhile, Protocol~\ref{prot:general} assumes the perfect knowledge of the former. In Section~\ref{sub:approachability_with_unknown_target} we will show how to bypass this issue, transferring all the unknown quantities into the target set and estimating it.
\begin{table}[h]
\resizebox{\textwidth}{!}{%
\begin{tabular}{@{}lll@{}}
\toprule
Criterion & Vector payoff function & Closed convex target set \\ \midrule
Calibration        & $\bmcal(k,b) = \bigl( (a^{(k')} - b) \, \ind{k=k'} \bigr)_{k' \in \cA}$                                                             & $\cCcal = \bigl\{ \bv \in \R^{N} : \ \| \bv \|_1 \leq 1/N \bigr\}$        \\
Group-calibration     & $\bmgcal(k,b,s) = \bigl( \bmcal(k, b) \,\ind{s = s'} / \gamma_{s'} \bigr)_{s' \in \cS}$                    & $\cCgcal = \bigl\{ \bv \in \R^{N |\cS|} : \ \| \bv \|_1 \leq 1/N \bigr\}$ \\
No-regret          & $\bmreg(a,b,x,s) = \bigl( r(a,b,x,s) - r(a',b,x,s) \bigr)_{a' \in \cA}$                                             & $\cCreg = \bigl( [0,+\infty) \bigr)^{N}$                                           \\
Group-no-regret       & $\bmgreg(a,b,x,s) = \bigl( \bmreg(a,b,x,s) \,\ind{s' = s} \bigr)_{s' \in \cS}$                                      & $\cCgreg = \bigl( [0,+\infty) \bigr)^{N |\cS|}$                           \\
Demographic parity & $\bmdp(k,s) = \bigl( {a^{(k)} \, \ind{s = 0}} / {\gamma_0}, \,\, {a^{(k)} \, \ind{s = 1}} / {\gamma_1} \bigr)$ & $\cCdp = \bigl\{ (u,v) \in \R^2 : \ |u-v| \leq 1/N \bigr\}$               \\
Equalized payoffs   & $\bmgrpay(a, b, x, s) = \bigl( r(a,b,x,s') \, \ind{s = s'} / \gamma_{s'} \bigr)_{s' \in \{0,1\}}$    & $\cCgrpay = \bigl\{ (u,v) \in \R^2 : \ |u - v| \leq \epsilon \bigr\}$ \\
\hline
\end{tabular}%
}
\end{table}

\section{Approachability theory adapted}
\label{sec:blackwell_s_approachability}

We provide a rather straightforward extension of the approachability theory to deal with Protocol~\ref{prot:general},
namely, with the existence of stochastic contexts, drawn according to an unknown distribution $\bfQ$.
We want to characterize closed convex sets that are approachable.

\paragraph{Pure vs.\ mixed actions.}
To conclude the description of the setting, we provide more details
on the randomized draws of the (pure) actions $a_{t+1}$ and $b_{t+1}$ of the Player
and Nature at round $t+1$.
We denote by $h_t$ the information available to Player at the end of round $t$,
and by $H_t$ the full history of the first $t$ rounds: $h_t = (\bm_{t'}, x_{t'}, s_{t'})_{t' \leq t}$
and $H_t = (a_{t'}, b_{t'}, x_{t'}, s_{t'})_{t' \leq t}$.
At the beginning of round $t+1$, the Player thus picks in a $h_t$--measurable way
a measurable family $\bigl( \bp_{t+1}^{x} \bigr)_{x \in \cX}$
of probability distributions over $\cA$
(i.e., a collection of distributions such that $x \in \cX \mapsto \bp^x_{t+1}$ is Borel-measurable),
and then draws $a_{t+1}$ independently
at random according to the mixed action $\bp_{t+1}^{x_{t+1}}$.
Similarly, Nature picks in a $H_t$--measurable way
a measurable family $\bigl( \bq_{t+1}^{G(x,s)} \bigr)_{(x,s) \in \cX \times \cS}$
of probability distributions over $\cB$, and uses
$\smash{\bq_{t+1}^{G(x_{t+1},s_{t+1})}}$ to draw $b_{t+1}$.

\paragraph{Approachability strategy.}
We adapt the original strategy by \citet{Blackwell56} by asuming the existence of and substituting
a sequence of estimates $\hat{\bfQ}_{t}$ that are $h_t$--adapted
in place of the unknown distribution $\bfQ$. We will assume that this sequence is convergent
in the total variation distance in the sense of Assumption~\ref{ass:consistent_estimator} below.
To state the strategy, we extend linearly $\bm$: for all probability distributions
$\bp$ over $\cA$ and $\bq$ over $\cB$, for all $(x,s) \in \cX \times \cS$,
\[
\bm(\bp, \bq, x, s \big) = \sum_{a \in \cA} \sum_{b \in \cB} \bp(a) \,  \bq(b) \, \bm(a, b, x, s)\,.
\]
Now, the Player uses an arbitrary measurable family of distributions $(\bp^x_1)_{x \in \cX}$ for the first round,
gets the estimate $\hat{\bfQ}_{1}$, and then uses, for rounds $t+1$, where $t \geq 1$:
\begin{equation}
\label{eq:defBlkstrat}
(\bp^x_{t+1})_{x \in \class{X}} \in \argmin_{(\bp^x)_{x \in \class{X}}}\max_{(\bq^{G(x, s)})_{(x, s) \in \class{X} \times \cS}}
\scalar{\bar{\bm}_{t} - \bar{\bc}_t}{\int_{\cX \times \cS} \bm\big(\bp^x, \bq^{G(x, s)}, x, s\big) \d\hat{\bfQ}_{t}(x, s)}\,,
\end{equation}
where the minimum and maximum are over all measurable families of probability distributions over $\cA$ and $\cB$, respectively.
The Player then gets access to $h_{t+1}$ and may compute the estimate $\hat{\bfQ}_{t+1}$ to be used at the next round.

\paragraph{Necessary and sufficient condition for approachability.}
We were able to work out such a condition under the assumption that $\bfQ$ can be estimated well enough, e.g., faster than
at a $1/\ln^3(T)$ rate in total variation distance. We recall that the total variation distance between
two probability distributions $\bfQ_{1}$ and $\bfQ_{2}$ on $\class{X}\times\class{S}$ equals
(see, e.g., \citet{Dev87}):
\[
\TV(\bfQ_{1},\bfQ_{2}) = \sup_{E \subseteq \cX \times \cS} \, \bigl|\bfQ_{1}(E) - \bfQ_{2}(E)\bigr|
= \frac{1}{2} \int_{\cX \times \cS} \bigl| g_1(x,s) - g_2(x,s) \bigr| \d\mu(x,s)\,,
\]
where the supremum is over all Borel sets $E$ of $\cX \times \cS$, and
where $g_1$ and $g_2$ denote densities of $\bfQ_1$ and $\bfQ_2$ with respect to a common dominating probability
distribution $\mu$.

\begin{assumption}[fast enough sequential estimation of $\bfQ$]
\label{ass:consistent_estimator}
The sequence of $(h_t)$--adapted estimators $(\hat{\bfQ}_{t})$ used is such that $\sum_{t = 1}^{+\infty} \frac{1}{t}\sqrt{\Exp\bigl[\TV^2(\hat{\bfQ}_{t}, \bfQ)\bigr]} < +\infty$.
\end{assumption}

The above assumption implies both $\tfrac{1}{T}\sum_{t = 1}^{T-1} \sqrt{\E \bigl[ \TV^2(\hat{\bfQ}_{t}, \bfQ) \bigr]}$ and $\sum_{t \geq T + 1} \tfrac{1}{t} \sqrt{\Exp\bigl[ \TV^2(\hat{\bfQ}_{t}, \bfQ) \bigr]}$ converge to zero (with $T$; see Appendix~\ref{app:A} for details).
Assumption~\ref{ass:consistent_estimator} is trivially satisfied in the case when $\bfQ$ is known, as it is sufficient to take $\hat{\bfQ}_t = \bfQ$.
When both $\cX$ and $\cS$ are finite sets, we may use the empirical frequencies as estimators $\hat{\bfQ}_{t}$; they satisfy
$\Exp[\TV^2(\hat{\bfQ}_{t}, \bfQ)] = O(1/t)$; see, e.g., \citep[Lemma 3]{Devroye83}.
The general case of an uncountable $\class{X}$, e.g., $\cX = \R^m$ requires results for density estimation in the $L^1$ or $L^2$ norms;
such results rely typically on moving averages or kernel estimates and may be found,
for instance, in the monographs by \citet{DG85} and \citet{Dev87} (see also \citet{Tsy08}). Under mild conditions,
the estimation takes place at a polynomial rate in total variation distance (e.g., a $T^{-1/5}$ rate
in dimension $m=1$). Note that the needed rate of decrease for $\Exp\bigl[\TV^2(\hat{\bfQ}_{t}, \bfQ)\bigr]$
in Assumption~\ref{ass:consistent_estimator} is extremely slow: a $1/\ln^3(T)$ rate would suffice.

\begin{assumption}[boundedness]
    \label{ass:bound_pay-off}
We assume that
$
\|\bm\|_{\infty{,} 2} {\eqdef} \underset{{{a,b} \in \cA \times \cB}}{\max} \underset{(x ,s) \in \cX \times \cS}{\sup}
\|\bm(a, b, x, s)\| {<} {+}\infty$.
\end{assumption}

We may now state our main result; the distance of $\bar{\bm}_T$ to $\cC$
was denoted by $d_T$ in Protocol~\ref{prot:general}.

\begin{restatable}[]{theorem}{main}
\label{thm:approachability_main}
Assume that $\class{C}$ is a closed convex set and that Assumptions~\ref{ass:consistent_estimator} (fast enough sequential estimation of $\bfQ$) and~\ref{ass:bound_pay-off} (bounded reward function) are satisfied,
then $\class{C}$ is approachable if and only if
\begin{align}
\label{eq:Blackwell_condition}
\forall (\bq^{G(x, s)})_{(x, s) \in \class{X} \times \{0, 1\}}\,\, \exists (\bp^{x})_{x \in \class{X}} \quad\text{s.t.}
\quad \int_{\class{X} \times \cS} \bm\big(\bp^x, \bq^{G(x, s)}, x, s\big) \d\bfQ(x, s) \in \class{C}\,.
\end{align}
In this case, the strategy of Eq.~\eqref{eq:defBlkstrat} achieves the following rates for
$L^2$ and almost-sure convergences:
\begin{align*}
&\E \bigl[ d_T^2 \bigr]  \leq
\sqrt{\frac{K}{T}} + 4\|\bm\|_{\infty, 2} \,\, \overbrace{\frac{1}{T}\sum_{t = 1}^{T-1} \sqrt{\E \bigl[ \TV^2(\hat{\bfQ}_{t}, \bfQ) \bigr]}}^{\eqdef \bar{\Delta}_T} \qquad \mbox{and} \\
&\P \! \left( \sup_{t \geq T} d_t \geq \varepsilon \! \right)
 \leq
\frac{3K}{T\varepsilon^2} + \frac{16 \|\bm\|_{\infty, 2}}{\varepsilon^2} \Biggl( \! \sqrt{\frac{K}{T-1}}
 + 2\!\left( \sup_{t \geq T} \bar{\Delta}_t \right) \!\!
    \biggl( \bar{\Delta}_T {+} \sum_{t \geq T } \frac{1}{t} \sqrt{\Exp\bigl[ \TV^2(\hat{\bfQ}_{t}, \bfQ) \bigr]} \biggr) \!
\Biggr)
\end{align*}
where $K < +\infty$ denotes the maximal distance to $\cC$ of an element of the compact set $\bm(\cA,\cB,\cX,\cS)$.
\end{restatable}

The proof lies in Appendix~\ref{app:A}. The necessity part of the theorem
actually relies on no assumption other than $\cC$ being closed; it consists
of showing that Nature has a stationary strategy such that there exists $\alpha > 0$
with $d_T \geq \alpha$ in the limit, i.e., the average payoff vectors $\bar{\bm}_T$
remain $\alpha$--away from $\cC$ in the limit. This exactly indicates that the underlying
fair online learning problem is not tractable: the underlying objectives and underlying
fairness constraints cannot be simultaneously satisfied.

\section{Working out some objective--constraint pairs: (im)possibility results}
\label{sec:first_work_outs}

In this section we apply Theorem~\ref{thm:approachability_main} to deal with some examples of
objective--constraint pairs described in
Sections~\ref{sec:objectives} and~\ref{sec:constraints}. Some of them have been considered before in the literature (sometimes in the batch setup) using various tools~\citep{Blum_Gunasekar_Lykouris_Srebro18,Hebert-Johnson_Kim_Reingold_Rothblum18,martinez20a,Gupta_Jung_Noarov_Pai_Roth21},
as discussed in Section~\ref{sec:introduction}.

We keep the original criteria and obtain possiblity or impossibility results. This is a first step, meanwhile, Section~\ref{sec:trade-off} will explain how to go further and obtain a trade-off, if needed, between
the objective and the fairness constraint.

\emph{Additional notation.} We recall that $\gamma_s = \P(s_t = s)$ and denote by $\bfQ^{s}$ the conditional distribution of $x_t$ given $s_t {=} s$, so that $\d\bfQ(x,s) = \gamma_{s} \d\bfQ^s(x)$. We denote by $\supp(\bfQ^{s}) \subseteq \cX$ the support of $\bfQ^{s}$.

\paragraph{Example 1: Vanilla calibration under a demographic parity constraint---achievable.}
Consider the following payoff function and target set,
obtained by simultaneously considering the objective of vanilla
calibration and the constraint of demographic parity: $\bm = (\bmcal, \bmdp)$ and $\cC = \cCcal \times \cCdp$.

Defining $\psi(u_1, u_2) \eqdef |u_1 - u_2|$, the approachability condition~\eqref{eq:Blackwell_condition} then reads as follows (where we introduce short-hand notation $\texttt{C}$ and
$\texttt{DP}$):
\begin{equation}
\label{eq:CNS-cal-DP}
\forall (\bq^{G(x, s)})_{(x, s) \in \class{X} \times \{0, 1\}} \ \ \exists (\bp^{x})_{x \in \class{X}} \quad \text{s.t.}
\begin{cases}
\displaystyle{\texttt{C} \eqdef \norm{\int_{\cX \times \{0, 1\}} \bmcal\bigl(\bp^{x}, \bq^{G(x, s)}\bigr) \d\bfQ(x, s)}_1 \leq \frac{1}{N}}\,; \\
\ \\
\displaystyle{\texttt{DP} \eqdef \psi\!\pa{\int_{\cX \times \{0, 1\}} \bmdp(\bp^{x}, s) \d\bfQ(x, s)} \leq \frac{1}{N}}\,.
\end{cases}
\end{equation}
Recalling the notation $\bfQ^0$ and $\bfQ^1$ for the conditional distributions, we observe that
\begin{align*}
\texttt{DP} & = \Biggl| \bigintsss_\cX \sum_{k=1}^N \bp^x(k) \, a^{(k)} \d\bfQ^0(x) - \bigintsss_\cX \sum_{k=1}^N \bp^x(k) \, a^{(k)} \d\bfQ^1(x)
\Biggr|\,.
\end{align*}
We now show that the condition in Eq.~\eqref{eq:CNS-cal-DP} is satisfied. For any $(\bq^{G(x, s)})$, we define the family $(\bp^x)$ as the constant family $\bigl( \dirac(Q_{\cA}) \bigr)$,
where $\dirac(Q_{\cA})$ denotes
the Dirac mass supported on $Q_{\cA}$, the closest point of $\cA$ to $Q \eqdef \int_{\cX \times \{0, 1\}} \bq^{G(x, s)}(1) \d \bfQ(x, s)$.
We have $\texttt{DP} = 0$ as $\bp^x$ does not depend on $x$.
Substituting the expression for $\bmcal$ into the definition of $\texttt{C}$,
we observe that for such a choice of $(\bp^x)_{x \in \cX}$, we have
\[
\texttt{C} =
\abs{\int_{\cX \times \{0, 1\}} \bigl( Q_{\cA} - \bq^{G(x, s)}(1) \bigr) \d \bfQ(x, s)}
\leq \frac{1}{2N} + \Biggl|
\underbrace{\int_{\cX \times \{0, 1\}} \bigl( Q - \bq^{G(x, s)}(1) \bigr) \d \bfQ(x, s)}_{=0}
\Biggr|
\,,
\]
where the inequality holds by taking the effect of discretization in $\cA$
into account and by the very definition of $Q$.
The condition of Eq.~\eqref{eq:CNS-cal-DP} is thus satisfied.
Therefore, under Assumption~\ref{ass:consistent_estimator} (the existence of fast enough sequential estimators of $\bfQ$) and thanks to Theorem~\ref{thm:approachability_main},
the vanilla calibration and the demographic parity can be achieved simultaneously no matter the monitoring of the Nature.

\paragraph{Example 2: Group-wise no-regret---mixed picture.}
Let the target set be $\cCgreg = \bigl( [0,+\infty) \bigr)^{N |\cS|}$ and the payoff function be $\bmgreg$,
i.e., we consider the case of group-wise no-regret under no additional constraint.
The approachability condition in Eq.~\eqref{eq:Blackwell_condition} demands that
\begin{align}
\label{eq:no_regret1}
& \quad \forall (\bq^{G(x, s)}) \ \exists (\bp^x) \quad \mbox{s.t.} \quad
\int_{\cX \times \cS}\bmgreg\bigl(\bp^x, \bq^{G(x, s)}\bigr)\d \bfQ(x, s) \in \bigl( [0, +\infty) \bigr)^{N |\cS|}\enspace, \qquad
\mbox{i.e.}, \\
& \forall (a', s), \quad
\nonumber
\bigintsss_{\! \supp(\bfQ^{s})} \,\,
\sum_{a \in \cA}\bp^x(a) \pa{\sum_{b \in \cB} \bq^{G(x, s)}(b) \, \bigl( r(a,b,x,s) - r(a',b,x, s) \bigr)}\!
\d\bfQ^{s}(x) \geq 0\enspace.
\end{align}

No-regret seems a harmless challenge, and it is so when the sensitive context is directly
observed by the Player, which we do not assume.
(In this case, the Player may simply run several no-regret algorithms in parallel, one per sensitive group $s$.)
In our context, the direct observation is emulated in some sense when the non-sensitive context $x$ reveals the sensitive
context $s$; this is the case, for instance, when the supports of the distributions $\bfQ^s$
are pairwise disjoint. Note, however, that these distributions $\bfQ^s$ are unknown to the Player and
need to be learned. The second part of Proposition~\ref{prop:noregrposs} shows that in this case,
the group-wise no-regret may be controlled.
We get a similar control in the case of irrelevant sensitive contexts,
i.e., not affecting the payoffs and not used by Nature; see
the first part of Proposition~\ref{prop:noregrposs}, which corresponds
to the case of vanilla no-regret minimization.
In both cases, the group-wise no-regret can be controlled
under Assumption~\ref{ass:consistent_estimator}, thanks to Theorem~\ref{thm:approachability_main}.
However, as we show by means of counter-examples, these are the only cases
that may be favorably dealt with.

\begin{proposition}
\label{prop:noregrposs}
The condition of Eq.~\eqref{eq:no_regret1} holds when
\begin{itemize}[topsep=-1ex,itemsep=-1ex,partopsep=1ex,parsep=1ex]
\item the sensitive context is irrelevant, i.e.,
the payoff function is such that $r(a, b, x, s) = r(a, b, x)$ and Nature's monitoring is $G(x, s) = x$;
\item for all $s \neq s'$, it holds that $\supp(\bfQ^s) \cap \supp(\bfQ^{s'}) = \emptyset$,
no matter Nature's monitoring~$G$.\smallskip
\end{itemize}
Otherwise, the condition of Eq.~\eqref{eq:no_regret1} may not hold.
\end{proposition}

\begin{proof}
We mimic the classical proof of no-regret by approachability
for the positive results.
For the \emph{first} positive result:
for any $(\bq^x)$, we define
$a^x \in \arg\max_{a \in \cA} \sum_{b \in \cB} \bq^x(b) \, r(a,b,x)$ and let
$(\bp^x) = \bigl( \dirac(a^x) \bigr)$.
For the \emph{second} positive result: fix any $(\bq^{G(x, s)})$;
we define $(\bp^{x})_{x \in \cX}$ point-wise as follows.
For all $s \in \cS$, all $x \in \supp(\bfQ^s)$, we set
$\bp^x = \dirac(a^x)$, where
we validly define $a^x \in \arg\max_{a \in \cA} \sum_{b \in \cB} \bq^{G(x, s)}(b) \, r(a,b,x,s)$
on the union of the supports of $(\bfQ^s)_{s \in \cS}$, since they are pair-wise disjoint;
we define the $a^x$ arbitrarily elsewhere.

Two counter-examples detailed in Appendix~\ref{app:B} back up the final part of the
proposition: we show that Eq.~\eqref{eq:no_regret1} does not hold.
In the first counter-example, the monitoring is $G(x, s) = x$, the payoff function depends on $s$, and the supports of $(\bfQ^s)_{s \in \cS}$ have non negligible intersection.
In the second example, the monitoring is $G(x, s) = (x, s)$, the payoff function does not depend on $s$, and the supports of $(\bfQ^s)_{s \in \cS}$ have non negligible intersection.
\end{proof}

\paragraph{Example 3: (Vanilla) no-regret under the equalized average payoffs constraint.}
For the sake of space we deal with this example in Appendix~\ref{app:B},
obtaining similar conclusions as that of~\citet{Blum_Gunasekar_Lykouris_Srebro18}.

\section{Group-wise calibration under a demographic parity constraint: trade-off}
\label{sec:trade-off}


In this section, we consider the problem of group-wise calibration under the demographic parity constraint;
in particular, $\cS = \{0,1\}$.
As we will see, except for special cases, the corresponding two error criteria cannot be simultaneously
smaller than the desired $1/N$ in the limit. However, a (possibly optimal) trade-off may be set between the calibration error $\varepsilon$
and the violation level $\delta$ of demographic parity.
To that end, we introduce neighborhoods of the original target sets $\cCgcal$
and $\cCdp$:
\[
\cCgcaleps = \bigl\{ \bv \in \R^{2N} : \ \| \bv \|_1 \leq \varepsilon \bigr\}
\qquad \mbox{and} \qquad
\cCdpdelta = \bigl\{ (u,v) \in \R^2 : \ |u-v| \leq \delta \bigr\}\,.
\]
A pair $(\epsilon, \delta) \in \bbR_+ \times \bbR_+$ is said \emph{achievable} when
$\cCgcaleps \times \cCdpdelta$ is approachable with $\bm = (\bmgcal, \bmdp)$.
Theorem~\ref{thm:approachability_main} provides a characterization of this approachability
as well as an associated strategy; in particular, when $(\epsilon, \delta)$ is achievable, this strategy
ensures that the calibration error $C_T$ and the violation $D_T$ of demographic parity satisfy: ${\limsup}\, C_T \leq \epsilon$ a.s. and ${\limsup}\, D_T \leq \delta$ a.s.

The goal of this section is to identify all achievable pairs $(\epsilon, \delta)$.
We will do so by determining, for $\delta \geq 0$ of interest, the \emph{smallest} $\epsilon \geq 0$
such that $(\epsilon,\delta)$ is achievable\footnote{Note that if $(\epsilon, \delta)$ is achievable,
then $(\epsilon', \delta')$ with $\epsilon' \geq \epsilon$ and $\delta' \geq \delta$ is also achievable.}; we denote it by $\epsilon^\star(\delta)$.
The line $\bigl( \delta, \, \epsilon^\star(\delta) \bigr)$ is a Pareto frontier.

\paragraph{Re-parametrization of the problem.}
Under Assumption~\ref{ass:consistent_estimator} (the existence of fast enough sequential estimators of $\bfQ$)
and thanks to Theorem~\ref{thm:approachability_main},
the $(\bmgcal, \bmdp)$--approachability of $\cCgcaleps \times \cCdpdelta$
holds if and only if the condition of Eq.~\eqref{eq:Blackwell_condition} is satisfied. The latter can be stated as follows:
\begin{equation}
\label{eq:CNS-Gcal-DP}
\forall (\bq^{G(x, s)})_{(x, s) \in \class{X} \times \{0, 1\}} \ \ \exists (\bp^{x})_{x \in \class{X}} \quad \text{s.t.} \quad
\begin{cases}
\displaystyle{\norm{\int_{\cX \times \{0,1\}} \bmgcal\bigl(\bp^{x}, \bq^{G(x, s)} \bigr) \d\bfQ(x, s)}_1 \leq \epsilon} \,;\\
\ \\
\displaystyle{\psi\!\pa{\int_{\cX \times \{0,1\}} \bmdp(\bp^{x}, s) \d\bfQ(x, s)} \leq \delta}\,,
\end{cases}
\end{equation}
where we recall that $\psi(u_1, u_2) = |u_1 - u_2|$.
Now, one can show (see comments after Lemma~\ref{lm:csqTV**} of Appendix~\ref{app:C}) that
the $\psi(\,\ldots\,)$ term above is always smaller than $\TV(\bfQ^0, \bfQ^1)$.
Thus, we can re-parameterize the problem and focus only on $\delta_\tau = \tau \cdot \TV(\bfQ^0, \bfQ^1)$, where $\tau \in [0,1]$.

\paragraph{Computation of the Pareto frontier.}
The
condition of Eq.~\eqref{eq:CNS-Gcal-DP}
indicates that
\begin{equation}
\begin{aligned}
\label{eq:minmax_global}
\epsilon^\star(\delta_\tau) =
\max_{(\bq^{G(x, s)})} \ \min_{(\bp^{x})} \ & \norm{\int_{\cX \times \{0,1\}} \bmgcal\bigl(\bp^{x}, \bq^{G(x, s)}, s\bigr) \d\bfQ(x, s)}_1 \\
        &\text{s.t.}\quad \psi\!\pa{\int_{\cX \times \{0,1\}} \bmdp(\bp^{x}, s) \d\bfQ(x, s)} \leq \tau \cdot \TV(\bfQ^0, \bfQ^1)\,.
\end{aligned}
\end{equation}
Propositions~\ref{prop:trade_off_awareness} and~\ref{prop:trade_off_unawareness}
below compute the values (up to the $1/N$ discretization error) of $\epsilon^\star(\delta_\tau)$ in two scenarios,
depending on whether Nature observes the sensitive contexts $s_t$.

\begin{restatable}[Nature awareness: $G(x, s) = (x, s)$]{proposition}{awareness}
    \label{prop:trade_off_awareness}
    Under Assumption~\ref{ass:consistent_estimator} and with the monitoring $G(x, s) = (x, s)$ for Nature,
    the Pareto frontier $\bigl( \epsilon^\star(\delta_\tau), \, \delta_\tau \bigr)_{\tau \in [0,1]}$ of achievable pairs
    satisfies
    \begin{align*}
    \delta_\tau = \tau \cdot \TV(\bfQ^0, \bfQ^1) \quad \mbox{and} \quad
     1 - \tau \cdot \TV(\bfQ^0, \bfQ^1) \leq \epsilon^\star(\delta_\tau) \leq 1 - \tau \cdot \TV(\bfQ^0, \bfQ^1) + \frac{1}{N}\enspace.
     \end{align*}
\end{restatable}

\begin{restatable}[Nature unawareness: $G(x, s) = x$]{proposition}{unawareness}
    \label{prop:trade_off_unawareness}
    Under Assumption~\ref{ass:consistent_estimator} and with the monitoring $G(x, s) = x$ for Nature,
    the Pareto frontier $\bigl( \epsilon^\star(\delta_\tau), \, \delta_\tau \bigr)_{\tau \in [0,1]}$ of achievable pairs
    satisfies:
     \begin{align*}
     \delta_\tau = \tau \cdot \TV(\bfQ^0, \bfQ^1)
     \quad \mbox{and} \quad
     (1 - \tau) \cdot \TV(\bfQ^0, \bfQ^1) \leq \epsilon^\star(\delta_\tau) \leq (1 - \tau) \cdot \TV(\bfQ^0, \bfQ^1) + \frac{1}{N}\,.
     \end{align*}
\end{restatable}

The parameter $\tau \in [0,1]$ is set by the user.

We observe that in the case when the true label $b_t$ provided by the Nature can be directly influenced by the sensitive attribute $s_t$, Proposition~\ref{prop:trade_off_awareness} shows that approximate
group-wise calibration with $\epsilon = 1/N$ is never possible, unless $\TV(\bfQ^{0}, \bfQ^{1}) = 1$ (and $\tau = 1$
is picked).
The latter case corresponds to the situation when the supports of $\bfQ^{0}$ and $\bfQ^{1}$ are disjoint, hence allowing the Player
to infer the sensitive context $s$ from the non-sensitive one $x$, essentially reducing (up to unknown $\bfQ$) the problem to the previously studied setup of Player's awareness~\citep{Hebert-Johnson_Kim_Reingold_Rothblum18}.

When the true label $b_t$ provided by the Nature is not \emph{directly} influenced by the sensitive attribute $s_t$ (it is influenced by $s_t$ only via $x_t$), Proposition~\ref{prop:trade_off_unawareness}
indicates that calibration is always possible by setting $\tau = 1$, no matter the value of $\TV(\bfQ^{0}, \bfQ^{1})$.
Interestingly, this proposition also shows that if $\TV(\bfQ^{0}, \bfQ^{1}) = 0$, i.e., the $x_t$ and the $s_t$ are independent, then the Player is able to achieve calibration and satisfy the demographic parity constraint simultaneously.

\section{Approachability of an unknown target set}
\label{sub:approachability_with_unknown_target}

A limitation of the calibration problems under demographic parity constraint discussed in Section~\ref{sec:first_work_outs} (Example~1)
and Section~\ref{sec:trade-off} is that the unknown probabilities $\gamma_0$ and $\gamma_1$ enter
the payoff functions $\bmgcal$ and $\bmdp$. We already pointed out this issue in Section~\ref{sec:summarytable}.
Even worse, the trade-off claimed in Propositions~\ref{prop:trade_off_awareness} and~\ref{prop:trade_off_unawareness}
relies on the knowledge of the unknown $\TV(\bfQ^0, \bfQ^1)$, to set the values of the achievable pair $(\delta,\epsilon)$
targeted; that is, the target set is unknown.
To bypass the first limitation we transfer the unknown $(\gamma_0, \gamma_1)$ to the target set,
which makes the payoff function fully known to the Player. We will then be left with the problem of approaching
an unknown target set only.
For instance, in the context of Section~\ref{sec:trade-off}, we can define
\begin{align*}
    &\tbmgcal(k,y,s) = \left(\bmcal(k, y) \,\ind{s = s'} \right)_{s' = 0,1}\,\,\text{ and }\,\,
    {\tbmdp}(k,s) = \bigl( {a^{(k)} \, \ind{s = 0}}, \,\, {a^{(k)} \, \ind{s = 1}} \bigr)\,,
\end{align*}
and set $\tbm \eqdef (\tbmgcal, \tbmdp)$. Taking into account the definition of $\bmcal$, we note that $\tbm$ does not depend on $(\gamma_0, \gamma_1)$.
Furthermore, by considering the closed convex target sets
\[
{\tcCgcaleps} = \enscond{ (\bv_0, \bv_1) \in \R^{2N} }{ \tfrac{\| \bv_0 \|_1}{\gamma_0} {+} \tfrac{\| \bv_1 \|_1}{\gamma_1} \leq \varepsilon }, \quad
{\tcCdpdelta} = \enscond{ (u,v) \in \R^2}{ \ \big|{\tfrac{u}{\gamma_0} - \tfrac{v}{\gamma_1}}\big| \leq \delta }\enspace,
\]
we remark that the $(\tbmgcal, \tbmdp)$--approachability of $\tcCgcaleps \times \tcCdpdelta$ is equivalent to
the $(\bmgcal, \bmdp)$--approachability of $\cCgcaleps \times \cCdpdelta$.
The unknown quantities appear only in the target set $\tcCgcaleps \times \tcCdpdelta$ (and $\delta$ and $\epsilon$ count as unknown
quantities given the trade-off exhibited), while the payoff $\tbm$ is known beforehand.
Thus, it is sufficient to consider the setup of Protocol~\ref{prot:general} with an \emph{unknown} target set $\cC$.

\paragraph{Approachability strategy for an unknown target set $\cC$.}
We still assume that the Player is able to build an $h_t$--adapted sequence of estimates $\hat{\bfQ}_{t}$.
Additionally, we assume that for $T_r \eqdef 2^r$, with $r \geq 0$, the Player can construct an $h_{T_r}$--adapted estimate $\hat{\cC}_r$ of $\cC$.
We discuss this assumption at the end of this section.
We define $d(\hat{\class{C}}_r, \class{C})=\sup_{x\in \hat{\class{C}}_r}d(x,\cC)$.

\begin{assumption}
\label{ass:set_estimation}
There exist $B < +\infty$ and a summable non-increasing sequence $(\beta_{r})_{r\geq 0}$
such that for all $r \geq 0$, the sets ${\hat{\cC}_{r}}$ are convex closed,
with $\|\bv - \Proj_{\hat{\class{C}}_r} (\bv)\| \leq B$ for all $\bv \in \bm(\cA, \cB, \cX, \{0,1\})$,
\begin{align*}
\Prob{\big(\class{C}\subset\hat{\class{C}}_r \big)}\geq 1- {{1}/ {(2T_{r})}},
\quad\text{and}\quad
\max\Bigl\{ \Exp\bigl[d(\hat{\class{C}}_r, \class{C})^2\bigr], \,\, \Exp\bigl[d(\cC,\hat{\class{C}}_r)^2\bigr] \Bigr\}
\leq \beta^2_{r}\enspace\,.
\end{align*}
\end{assumption}

For all $r \geq 0$ and all $t \in \{T_{r}, \, \ldots, \, T_{r+1} - 1 \}$, define
$\hat{\bc}_t \eqdef \Proj_{\hat{\class{C}}_r}(\bar{\bm}_t)$.
The idea of the approachability strategy is to use $\hat{\bc}_t$ in place of $\bar{\bc}_t$ in
Eq.~\eqref{eq:defBlkstrat} and update the estimate $\hat{\cC}_r$ of the target $\cC$ only at the end of rounds
$t = T_r$. More precisely, the strategy of the Player is:
\begin{align}
  \label{eq:strategy_unknown}
    (\bp^x_{t+1})_{x \in \class{X}} \in \argmin_{(\bp^x)}\max_{(\bq^{G(x, s)})}
\scalar{\bar{\bm}_{t} - \hat{\bc}_t}{\int \bm\big(\bp^x, \bq^{G(x, s)}, x, s\big) \d\hat{\bfQ}_{t}(x, s)}\enspace.
\end{align}

\begin{theorem}
\label{th:unknown}
Under Assumption~\ref{ass:set_estimation} and the assumptions of Theorem~\ref{thm:approachability_main}, a convex closed set $\cC$,
unknown to the Player, is $\bm$--approachable if and only if Blackwell's condition in Eq.~\eqref{eq:Blackwell_condition} is satisfied.
In this case, the strategy of Eq.~\eqref{eq:strategy_unknown} is an approachability strategy.
\end{theorem}

Appendix~\ref{sec:proofthunknown} provides a proof of Theorem~\ref{th:unknown}.
But before we do so, we discuss in Appendix~\ref{sec:detailsreduc}
why and how the target set $\cC$
may be estimated by sets $\hat{\cC}_r$ satisfying Assumption~\ref{ass:set_estimation}.
The construction is idiosyncratic and strongly depends on the problem and exact setting
considered (in particular, whether the set of non-sensitive contexts $\cX$
is finite or not).
We provide an illustration for the target set $\cC = \tcCgcaleps \times \tcCdpdelta$ of Section~\ref{sec:trade-off}, in the case of a finite set $\cX$.

\section{Limitations of the current work and topics for future work}
\label{sec:limitations}

The anonymous reviewers of this article pointed out some limitations
to or possible extensions of the current work, which we list now.

We only considered, for the sake of readability, the case of demographic
parity with two groups. While the criterion of demographic parity easily
extends to more groups, the extension of the trade-offs stated in
Section~\ref{sec:trade-off} is less clear. It would probably
involve the total variation distances $\TV(\bfQ^s, \bfQ^{s'})$
between each pair $\bfQ^s$ and $\bfQ^{s'}$ of marginal distributions,
where $s,s' \in \cS$, or the distances $\TV(\bfQ^s, \bfQ)$.

The complexity of the generic approachability strategy of Section~\ref{sec:blackwell_s_approachability}
is at least linear in the number of groups (if an approximate solution is used,
and even polynomial in this number for an exact solution), see \citet[Sections~3.3 and~3.4]{Mannor_Stoltz10}.
The convergence rates achieved in Theorem~\ref{thm:approachability_main}
involve total variation distances that would also probably depend in at least a linear fashion
on the number of groups, unless some special structure is assumed.
Both facts may be an issue for large numbers of groups.
More generally, we only provide in this article a generic strategy, that is, a first
approach to tackle a given fair online learning problem,
but specific strategies may be more efficient and get better regret bounds,
in particular for large numbers of groups.
(For instance, for group-wise calibration, \citet{Gupta_Jung_Noarov_Pai_Roth21}
base a specific and computationally more efficient strategy on an exponential surrogate loss:
this strategy enjoys a sample complexity only logarithmic in the number of groups.)
The design of such specific strategies remains largely open.

As mentioned in the introduction,
\citet{Bechavod_Ligett_Roth_Waggoner_Wu19} consider the objective of online binary classification
and an equal-opportunity fairness constraint, under some partial monitoring known as ``apple
tasting''. In this article, we considered a bandit monitoring for the Player: she observes
the reward obtained at each round. Partial monitoring, which was introduced by \citet{Rust99},
consists of only receiving feedback being a (possibly random) function of the actions played
by the Player and Nature. A theory of approachability under partial monitoring was initiated by \citet{Perchet11},
who stated a necessary and sufficient condition (see also \citet{MPS13});
\citet{MPS14} then exhibited a computationally more efficient strategy, with improved convergence rates,
and \citet{KP17} finally obtained the optimal convergence rates.
Therefore, the question to investigate is the extension of the results of this article
from a bandit monitoring to a partial monitoring.

\clearpage
\acksection

Evgenii Chzhen was fully funded by grant ANR-11-LABX-0056-LMH (Labex LMH, part of
Programme d'investissements d'avenir). Christophe Giraud received partial support
by grant ANR-19-CHIA-0021-01 (``BiSCottE'', Agence Nationale de la Recherche).
Gilles Stoltz has no direct funding to acknowledge.

Additional revenues for authors are: Evgenii Chzhen---none;
Christophe Giraud---none; Gilles Stoltz---part time employment as an affiliate professor with HEC Paris.

\bibliography{biblio}
\bibliographystyle{abbrvnat}

\clearpage
\appendix

\begin{center}
  {\Large\bf Supplementary Material for \medskip \\
  	``{\titre}'' \bigskip \\
  by Evgenii Chzhen, Christophe Giraud, Gilles Stoltz}
\end{center}

\ \\

This supplementary material contains all the proofs omitted from the main body.
Each section provides the proofs of claims, theorems, or propositions of a section of the main body;
more precisely, Appendix~\ref{app:A} provides proofs for Section~\ref{sec:blackwell_s_approachability},
Appendix~\ref{app:B} does so for Section~\ref{sec:first_work_outs},
Appendix~\ref{app:C}, for Section~\ref{sec:trade-off},
and finally, Appendix~\ref{app:D} deals with Section~\ref{sub:approachability_with_unknown_target}.

\clearpage
\section{Proofs for Section~\ref{sec:blackwell_s_approachability}}
\label{app:A}

We start by proving a claim stated right after Assumption~\ref{ass:consistent_estimator}:
that
\[
\sum_{t = 1}^{+\infty} \frac{1}{t}\sqrt{\Exp\bigl[\TV^2(\hat{\bfQ}_{t}, \bfQ)\bigr]} \eqdef C < +\infty
\qquad \mbox{entails} \qquad
\bar{\Delta}_T \eqdef \frac{1}{T}\sum_{t = 1}^{T-1} \sqrt{\E \bigl[ \TV^2(\hat{\bfQ}_{t}, \bfQ) \bigr]} \longrightarrow 0\,.
\]
Indeed,
\begin{align*}
\frac{1}{T}\sum_{t = 1}^{T-1} \sqrt{\E \bigl[ \TV^2(\hat{\bfQ}_{t}, \bfQ) \bigr]}
& = \frac{1}{T} \sum_{t = 1}^{\lfloor \sqrt{T} \rfloor} \sqrt{\E \bigl[ \TV^2(\hat{\bfQ}_{t}, \bfQ) \bigr]}
+ \frac{1}{T}\sum_{t = \lfloor \sqrt{T} \rfloor + 1}^{T-1} \sqrt{\E \bigl[ \TV^2(\hat{\bfQ}_{t}, \bfQ) \bigr]} \\
& \leq \frac{1}{\sqrt{T}} \underbrace{\sum_{t = 1}^{\lfloor \sqrt{T} \rfloor} \frac{1}{t} \sqrt{\E \bigl[ \TV^2(\hat{\bfQ}_{t}, \bfQ) \bigr]}}_{\leq C}
+ \sum_{t = \lfloor \sqrt{T} \rfloor + 1}^{T} \frac{1}{t} \sqrt{\E \bigl[ \TV^2(\hat{\bfQ}_{t}, \bfQ) \bigr]}\,,
\end{align*}
which converges to~$0$, as it is the sum of $C/\sqrt{T}$ with a quantity smaller than the remainder of a convergent series.

We recall that for all Borel-measurable
functions $f : \class{X}\times\class{S} \to \R^d$ with $\smash{\displaystyle{\sup_{(x,s) \in \cX \times \cS} \| f(x,s) \|} \leq M}$,
\begin{multline}
\label{eq:majo-int-TV}
\left\| \int_{\cX \times \cS} f(x,s) \d\bfQ_{1}(x,s) - \int_{\cX \times \cS} f(x,s) \d\bfQ_{2}(x,s) \right\| \\
\leq \int_{\cX \times \cS} \| f(x,s) \| \, \bigl| g_1(x,s) - g_2(x,s) \bigr| \d\mu(x,s) \leq 2M \,\cdot\, \TV(\bfQ_{1},\bfQ_{2})\,,
\end{multline}
where $g_1$ and $g_2$ denote densities of the distributions $\bfQ_1$ and $\bfQ_2$
with respect to a common dominating measure $\mu$.

We now move to the proof of Theorem~\ref{thm:approachability_main}, which we restate below.
It relies on two lemmas stated below in Section~\ref{sec:lm-proof-main-th-appr}.
Unless stated otherwise (namely, for matters related to the estimation of $\bfQ$), all
material is standard and was introduced by \citet{Blackwell56} (see also the more modern expositions by
\citet{Perchet13} or \citet{MSZ}).

\main*

\begin{proof}[Proof of Theorem~\ref{thm:approachability_main}]
  {\bf Part I: Necessity.} Assume that the condition in Eq.~\eqref{eq:Blackwell_condition} is not satisfied, then
  \begin{align*}
    \exists (\bq^{G(x, s)}_0)_{(x, s) \in \class{X} \times \cS} \ \ \forall (\bp^{x})_{x \in \class{X}} \quad\text{s.t.}\quad \int_{\class{X} \times \cS}
    \, \bm(\bp^x, \bq^{G(x, s)}_0, x, s) \d\bfQ(x,s) \notin \class{C}\enspace.
  \end{align*}
Since $\cC$ is closed and by continuity of the norm, there exists $\alpha > 0$ such that
\begin{align}
\label{eq:gammaLB}
\forall (\bp^{x})_{x \in \class{X}} \qquad
\min_{\bv \in \class{C}}\norm{\bv - \int_{\class{X} \times \cS} \bm(\bp^x, \bq^{G(x, s)}_0, x, s) \d\bfQ(x, s)} \geq \alpha \enspace.
  \end{align}
Let Nature play using this distribution $(\bq^{G(x, s)}_0)_{(x, s)}$ at each stage $t \geq 1$ to draw $b_t$.
Given that the sensitive attributes and contexts $(x_t,s_t)$ are drawn i.i.d., the conditional expectation of the reward
of the player at round $t \geq 1$ based on the history $H_{t-1} = (a_{t'}, b_{t'}, x_{t'}, s_{t'})_{t' \leq t-1}$ equals
\[
\E \bigl[ \bm(a_t, b_t, x_t, s_t) \,\big|\, H_{t-1} \bigr] = \int_{\cX \times \cS} \bm \bigl( \bp^x_t, \bq^{G(x, s)}_0, x, s \bigr) \d\bfQ(x, s) \,.
\]
Then, for any strategy of the player, it holds by martingale convergence (e.g.,
by the Hoeffding-Azuma inequality and the Borel–Cantelli lemma, used for each component of $\bm$) that
\begin{align*}
\norm{\frac{1}{T}\sum_{t = 1}^T \bm(a_t, b_t, x_t, s_t) -
\frac{1}{T}\sum_{t = 1}^T \int_{\cX \times \cS} \bm\bigl(\bp^x_t, \bq^{G(x, s)}_0, x, s \bigr) \d\bfQ(x, s)}
\longrightarrow 0 \quad \text{a.s.}
\end{align*}
Set $\displaystyle{\bar{\bp}^x_T \eqdef \frac{1}{T}\sum_{t = 1}^T \bp_t^x}$, then the above implies that
  \begin{align}
    \label{eq:as_necessary}
    \norm{\bar{\bm}_T - \int_{\cX \times \cS} \bm(\bar{\bp}^x_T, \bq^{G(x, s)}_0, x, s) \d\bfQ(x, s)} \longrightarrow 0 \quad \text{a.s.}
\end{align}
By the triangle inequality for the Euclidean norm, Eqs.~\eqref{eq:gammaLB} and~\eqref{eq:as_necessary} entail that
\[
\liminf_{T \to +\infty} \
d\bigl(\bar{\bm}_T, \cC\bigr) =
\liminf_{T \to +\infty} \
\bigl\| \bar{\bm}_T - \bar{\bc}_T \bigr\|
=
\liminf_{T \to +\infty} \
\min_{\bv \in \class{C}} \bigl\| \bv - \bar{\bm}_T \bigr\| \geq \alpha \quad \text{a.s.}
\]
That is, Nature prevents the player from approaching $\cC$ (and even: Nature approaches the complement of the $\alpha$-neighborhood of $\cC$).

Note that in this part we did not use that the target set $\cC$ was convex, only that it was a closed set.

{\bf Part II: Sufficiency.}
    Recall that we denoted by $d_t \eqdef \|\bar{\bm}_t - \bar{\bc}_t\|_2$ the Euclidean distance of $\bar{\bm}_t$
    to~$\cC$.
    Observe that by definition of the projections $\bar{\bc}_{t+1}$ and $\bar{\bc}_t$
    and by expanding the square norm,
    \begin{equation}
      \label{eq:recursion_1}
    \begin{aligned}
        d_{t+1}^2
        \leq
        \|\bar{\bm}_{t+1} - \bar{\bc}_t\|^2
        & =
        \norm{\frac{t}{t+1}\parent{\bar{\bm}_t - \bar{\bc}_t} + \frac{1}{t+1}\parent{\bm_{t+1} - \bar{\bc}_t}}^2\\
        &=
        \parent{\frac{t}{t+1}}^{\!\! 2} d_t^2 + \frac{\|\bm_{t+1} - \bar{\bc}_t\|^2}{(t+1)^2} + \frac{2t}{(t+1)^2}\scalar{\bar{\bm}_t - \bar{\bc}_t}{\bm_{t+1} - \bar{\bc}_t}\,.
    \end{aligned}
    \end{equation}
    Moreover, we have, by definition of $(\bp^x_{t+1})_{x \in \cX}$ in~Eq.~\eqref{eq:defBlkstrat} as the argmin of a maximum,
    \begin{equation}
      \label{eq:cross_1}
    \begin{aligned}
      \lefteqn{\scalar{\bar{\bm}_t - \bar{\bc}_t}{\bm_{t+1} - \bar{\bc}_t}} \\
      =
      &\scalar{\bar{\bm}_t - \bar{\bc}_t}{\bm_{t+1} - \int_{\cX \times \cS} \bm\big(\bp^x_{t+1}, \bq^{G(x, s)}_{t+1}, x, s\big) \d\hat{\bfQ}_{t}(x, s)}\\
      &+
      \scalar{\bar{\bm}_t - \bar{\bc}_t}{\int_{\cX \times \cS} \bm\big(\bp^x_{t+1}, \bq^{G(x, s)}_{t+1}, x, s\big) \d\hat{\bfQ}_{t}(x, s) - \bar{\bc}_t}\\
      \leq
      &\scalar{\bar{\bm}_t - \bar{\bc}_t}{\bm_{t+1} - \int_{\cX \times \cS} \bm\big(\bp^x_{t+1}, \bq^{G(x, s)}_{t+1}, x, s\big) \d\hat{\bfQ}_{t}(x, s)}\\
      &+
      \min_{(\bp^x)_x}\max_{(\bq^{G(x, s)})_{G(x, s)}} \scalar{\bar{\bm}_t - \bar{\bc}_t}{\int_{\cX \times \cS} \bm\big(\bp^x, \bq^{G(x, s)}, x, s\big) \d\hat{\bfQ}_{t}(x, s) - \bar{\bc}_t}\,.
    \end{aligned}
    \end{equation}
    Furthermore, the Cauchy-Schwarz inequality, followed by an application of the bound of Eq.~\eqref{eq:majo-int-TV},
    indicates that for all $(\bp^x)_{x}$ and all $(\bq^{G(x, s)})_{(x, s)}$,
    \begin{align*}
      & \left| \scalar{\bar{\bm}_t - \bar{\bc}_t}{\int_{\cX \times \cS} \bm\big(\bp^x, \bq^{G(x, s)}, x, s\big) \d\hat{\bfQ}_{t}(x, s) -
      \int_{\cX \times \cS} \bm\big(\bp^x, \bq^{G(x, s)}, x, s\big) \d\bfQ(x, s)} \right| \\
      \leq & \ d_t \cdot \left\| \int_{\cX \times \cS} \bm\big(\bp^x, \bq^{G(x, s)}, x, s\big) \d\hat{\bfQ}_{t}(x, s) -
      \int_{\cX \times \cS} \bm\big(\bp^x, \bq^{G(x, s)}, x, s\big) \d\bfQ(x, s) \right\| \\
      \leq & \ 2 d_t \cdot \TV(\hat{\bfQ}_{t}, \bfQ) \cdot \|\bm\|_{\infty, 2}\enspace.
    \end{align*}
    Hence, using twice this bound in Eq.~\eqref{eq:cross_1} and introducing
    \begin{align}
      \label{eq:martingale_diff}
      Z_{t+1} \eqdef \scalar{\bar{\bm}_t - \bar{\bc}_t}{\bm_{t+1} -
      \int_{\cX \times \cS} \bm\big(\bp^x_{t+1}, \bq^{G(x, s)}_{t+1}, x, s\big) \d\bfQ(x, s)}\enspace,
    \end{align}
    we obtain
    \begin{equation}
      \label{eq:cross_2}
    \begin{aligned}
        \scalar{\bar{\bm}_t - \bar{\bc}_t}{\bm_{t+1} - \bar{\bc}_t}
        \leq
        &\,\, Z_{t+1} + 4d_t \cdot \TV(\hat{\bfQ}_{t}, \bfQ) \cdot \|\bm\|_{\infty, 2} \\
        & \ \ +\min_{(\bp^x)_x}\max_{(\bq^{G(x, s)})_{G(x, s)}}\scalar{\bar{\bm}_t - \bar{\bc}_t}{\int_{\cX \times \cS} \bm\big(\bp^x, \bq^{G(x, s)}, x, s\big)
        \d\bfQ(x, s) - \bar{\bc}_t}\enspace.
    \end{aligned}
    \end{equation}
    We recall that the Euclidean projection $\bc$ of a vector $\bn$ onto a closed convex set $\cC \subset \R^d$ satisfies:
    \[
    \forall \bc' \in \cC, \qquad \scalar{\bn-\bc}{\bc' - \bc} \leq 0\,.
    \]
    Thus, thanks to von Neumann's minmax theorem (for the equality) and
    the Blackwell's condition in Eq.~\eqref{eq:Blackwell_condition} together with the above-recalled property
    of the projection,
    \begin{multline}
      \label{eq:cross_3}
      \min_{(\bp^x)_x}\max_{(\bq^{G(x, s)})_{G(x, s)}} \scalar{\bar{\bm}_t - \bar{\bc}_t}{\int_{\cX \times \cS} \bm\big(\bp^x, \bq^{G(x, s)}, x, s\big)\d\bfQ(x, s) - \bar{\bc}_t} \\
      = \max_{(\bq^{G(x, s)})_{G(x, s)}} \min_{(\bp^x)_x} \scalar{\bar{\bm}_t - \bar{\bc}_t}{\int_{\cX \times \cS} \bm\big(\bp^x, \bq^{G(x, s)}, x, s\big) \d\bfQ(x, s) - \bar{\bc}_t} \leq 0\enspace.
    \end{multline}
    Hence, combining Eqs.~\eqref{eq:cross_2} and~\eqref{eq:cross_3} with Eq.~\eqref{eq:recursion_1},
    and bounding $\|\bm_{t+1} - \bar{\bc}_t\|^2$ by $K$ (by definition of $K$),
    we have obtained so far
    \begin{align}
        \label{eq:recusrion_2}
        d_{t+1}^2 \leq \parent{\frac{t}{t+1}}^{\!\! 2} d_t^2 + \frac{K}{(t+1)^2} + \frac{2t}{(t+1)^2}\parent{Z_{t+1} + 4d_t \cdot \TV(\hat{\bfQ}_{t}, \bfQ) \cdot \|\bm\|_{\infty, 2}}\enspace.
    \end{align}
    The $4d_t \cdot \TV(\hat{\bfQ}_{t}, \bfQ) \cdot \|\bm\|_{\infty, 2}$ is the sole difference to the standard proof of approachability.
    We deal with it by adapting the conclusions of the original proof.

    Before we do so, we note that the $Z_{t+1}$ introduced in Eq.~\eqref{eq:martingale_diff} form a martingale difference sequence
    with respect to the history $H_t$: indeed, $\bar{\bm}_t$ and $\bar{\bc}_t$ are $H_t$--measurable and
    so are the $(\bp^x_{t+1})_x$ and the $(\bq^{G(x, s)}_{t+1})_{x,s}$; since in addition $(x_{t+1},s_{t+1})$ is drawn independently from everything
    according to $\bfQ$ and $a_{t+1}$ and $b_{t+1}$ are drawn independently at random according
    to $\bp^{x_t}_{t+1}$ and $\bq^{G(x_t, s_t)}$, we have
    \[
    \E[ \bm_{t+1} \,|\, H_{t} ] =
    \E \bigl[ \bm(a_{t+1}, b_{t+1}, x_{t+1}, s_{t+1}) \,\big|\, H_{t} \bigr] =
    \int_{\cX \times \cS} \bm \bigl( \bp^x_{t+1}, \bq^{G(x, s)}_{t+1}, x, s \bigr) \d\bfQ(x, s) \,,
    \]
    so that $\E[Z_{t+1}\,|\,H_t] = 0$. \medskip

    {\bf Part II: Sufficiency---convergence in $L^2$.}
    In particular, taking expectations in Eq.~\eqref{eq:recusrion_2} and applying the tower rule (for the first inequality)
    and applying the Cauchy-Schwarz inequality (for the second inequality), we have
    \begin{align}
    \nonumber
        \E \bigl[ d_{t+1}^2 \bigr]
        &\leq
        \parent{\frac{t}{t+1}}^{\!\! 2} \E \bigl[ d_t^2 \bigr] + \frac{K}{(t+1)^2} + \frac{8t \|\bm\|_{\infty, 2}}{(t+1)^2} \, \E \bigl[d_t \cdot \TV(\hat{\bfQ}_{t}, \bfQ) \bigr] \\
    \label{eq:extratermBlk}
        &\leq
        \parent{\frac{t}{t+1}}^{\!\! 2} \E \bigl[ d_t^2 \bigr] + \frac{K}{(t+1)^2} + \frac{8t \|\bm\|_{\infty, 2}}{(t+1)^2}
        \sqrt{\E \bigl[ d_t^2 \bigr]} \, \sqrt{\E \bigl[ \TV^2(\hat{\bfQ}_{t}, \bfQ) \bigr]} \enspace.
    \end{align}
    Applying Lemma~\ref{lem:recursion} below, we get
    \begin{align}
      \label{eq:bound_recursion_exp}
      \sqrt{\E\bigl[ d_{T}^2 \bigr]} \leq B_T \eqdef \sqrt{\frac{K}{T}} +
      4\|\bm\|_{\infty, 2} \,\, \underbrace{\frac{1}{T}\sum_{t = 1}^{T-1} \sqrt{\E \bigl[ \TV^2(\hat{\bfQ}_{t}, \bfQ) \bigr]}}_{\eqdef \bar{\Delta}_T} \enspace.
    \end{align}
    By (a consequence of)
    Assumption~\ref{ass:consistent_estimator}, the second term in the right-hand side converges to zero, and we obtain convergence
    in $L^2$.

    {\bf Part II: Sufficiency---almost-sure convergence.}
    We define
    \begin{align*}
      S_T \eqdef d_T^2 + \Exp\!\left[\sum_{t \geq T }\parent{\frac{K}{(t+1)^2} + \frac{8t \|\bm\|_{\infty, 2}}{(t+1)^2} \, d_t \cdot \TV(\hat{\bfQ}_{t}, \bfQ)} \,\bigg|\, H_{T}\right]\enspace,
    \end{align*}
    and note that $(S_T)_{T \geq 1}$ is a non-negative super-martingale with respect to the filtration induced by $(H_T)_{T \geq 1}$;
    indeed, the recursion of Eq.~\eqref{eq:recusrion_2} entails, together with $\smash{\bigl( t/(t+1) \bigr)^2} \leq 1$ and $\E[Z_{T+1}\,|\,H_T] =0$:
    \begin{align*}
      \Exp[S_{T+1} \mid H_{T}]
      &=
      \Exp\bigl[d_{T+1}^2 \mid H_{T} \bigr]
      + \Exp\!\left[\sum_{t \geq T + 1}\parent{\frac{K}{(t+1)^2} + \frac{8t \|\bm\|_{\infty, 2}}{(t+1)^2} \, d_t \cdot \TV(\hat{\bfQ}_{t}, \bfQ)} \,\bigg|\, H_{T}\right]\\
      & \leq
      d_{T}^2
      + \Exp\left[\sum_{t \geq T }\parent{\frac{K}{(t+1)^2} + \frac{8t \|\bm\|_{\infty, 2}}{(t+1)^2} \, d_t \cdot \TV(\hat{\bfQ}_{t}, \bfQ)} \,\bigg|\, H_{T}\right] = S_{T}\enspace.
    \end{align*}
    We may thus use $d^2_T \leq S_T$ and apply Doob's maximal inequality for non-negative super-martingales (Lemma~\ref{lem:super-martingale}):
    \[
      \Prob\parent{\sup_{T' \geq T} d_{T'} \geq \varepsilon}
    = \Prob\parent{\sup_{T' \geq T} d^2_{T'} \geq \varepsilon^2}
    \leq \Prob\parent{\sup_{T' \geq T} S_{T'} \geq \varepsilon^2}
    \leq \frac{\E[S_T]}{\varepsilon^2}\,.
    \]
    The proof is concluded by upper bounding $\E[S_T]$.
    The tower rule, the Cauchy-Schwarz inequality, and the bound $t/(t+1)^2 \leq 1/(t+1) \leq 1/t$ yield
    \begin{align*}
      \Exp[S_T]
    \leq \E\bigl[ d_{T}^2 \bigr] + \sum_{t \geq T} \frac{K}{(t+1)^2}
       + 8 \|\bm\|_{\infty, 2} \sum_{t \geq T } \sqrt{\E\bigl[ d_{t}^2 \bigr]} \, \, \frac{\sqrt{\Exp\bigl[ \TV^2(\hat{\bfQ}_{t}, \bfQ) \bigr]}}{t}\,.
    \end{align*}
    We substitute the bound from Eq.~\eqref{eq:bound_recursion_exp}, keeping in mind that the total variation distance is always smaller
    than 1:
    \begin{align*}
     \Exp[S_T]
    \leq \E\bigl[ d_{T}^2 \bigr] + \frac{K}{T} + 8 \|\bm\|_{\infty, 2} \sum_{t \geq T} \frac{1}{t} \sqrt{\frac{K}{t}}
    +  32 \|\bm\|_{\infty, 2}^2 \sum_{t \geq T } \bar{\Delta}_t \, \frac{\sqrt{\Exp\bigl[ \TV^2(\hat{\bfQ}_{t}, \bfQ) \bigr]}}{t}\,.
    \end{align*}
    Eq.~\eqref{eq:bound_recursion_exp} also implies, together with $(a+b)^2 \leq 2a^2 + 2b^2$,
    that $\E\bigl[ d_{T}^2 \bigr] \leq 2K/T + 32 \|\bm\|_{\infty, 2}^2 (\bar{\Delta}_T)^2$. All in all, we get the final bound
    \[
    \Exp[S_T] \leq \frac{3K}{T} + \frac{16 \|\bm\|_{\infty, 2} \sqrt{K}}{\sqrt{T-1}} + 32 \|\bm\|_{\infty, 2}^2 \left( \sup_{t \geq T+1} \bar{\Delta}_t \right) \!
    \left( \bar{\Delta}_T + \sum_{t \geq T} \frac{1}{t} \sqrt{\Exp\bigl[ \TV^2(\hat{\bfQ}_{t}, \bfQ) \bigr]} \right).
    \]
\end{proof}

\subsection{Two lemmas used in the proof of Theorem~\ref{thm:approachability_main}}
\label{sec:lm-proof-main-th-appr}

The following lemma is an ad-hoc and new, but elementary, tool to
deal with the additional term appearing in Eq.~(\ref{eq:extratermBlk})
compared to the original proof of approachability.

\begin{lemma}
  \label{lem:recursion}
  Let $t^* \geq 0$, and consider two non-negative sequences  $(d_t)_{t \geq t^*}$ and $(\delta_t)_{t \geq t^*}$   fulfilling, for $t\geq t^*$, the recursive inequality
    \begin{align}\label{eq:recursion-lem}
      d_{t+1}^2 \leq \parent{\frac{t}{t+1}}^{\!\! 2} d_t^2 + \frac{K}{(t+1)^2} + \frac{2t}{(t+1)^2} \, \delta_t \, d_t \enspace.
    \end{align}
   Then, for all $t \geq t^*+1$,
    \begin{align*}
      d_{t} \leq \frac{\sqrt{K(t - t^*)}}{t} + \frac{1}{t}\sum_{t' = t^*}^{t-1} \delta_{t'} + \frac{t^* d_{t^*}}{t}\enspace.
    \end{align*}
    In particular, if $(d_t)_{t \geq 1}$ and $(\delta_t)_{t \geq 1}$ are two non-negative sequences fulfilling the recursive inequality (\ref{eq:recursion-lem}) for $t\geq 1$,
    and if $d_{1}\leq \sqrt{K}$, then, for all $t \geq 1$,
    \begin{align*}
      d_{t} \leq \sqrt{\frac{K}{t}} + \frac{1}{t}\sum_{t' = 1}^{t-1} \delta_{t'}\enspace.
    \end{align*}

\end{lemma}

\begin{proof}
\underline{Second part of the lemma.}
Let us first check that the second part of the lemma follows from the first part.
Setting $d_{0}=\delta_{0}=0$, the sequences  $(d_t)_{t \geq 0}$ and $(\delta_t)_{t \geq 0}$
fulfill Eq.~\eqref{eq:recursion-lem} for $t\geq t^*=0$, hence
$$ d_{t} \leq \sqrt{\frac{K}{t}} + \frac{1}{t}\sum_{t' = 0}^{t-1} \delta_{t'}= \sqrt{\frac{K}{t}} + \frac{1}{t}\sum_{t' = 1}^{t-1} \delta_{t'}\,,$$
where the equality in the right-hand side comes from $\delta_{0}=0$.

\underline{First part of the lemma.}
    Set $U_t = t \, d_t$ and $\Delta^*_t = \delta_{t^*} + \ldots + \delta_t$ with the convention that $\Delta^*_t = 0$ for all $t < t^*$.
    It is equivalent to prove that for all $t\geq t^*+1$, we have
      \begin{align}\label{eq:U-recursion_v2}
      U_{t} \leq \sqrt{K(t - t^*)} + \Delta^*_{t-1} + U_{t^*}\enspace.
    \end{align}
    We observe that Eq.~\eqref{eq:U-recursion_v2} trivially holds for $t=t^*$. Assume that Eq.~\eqref{eq:U-recursion_v2} holds for $t\geq t^*$.
    By assumption, we have $U_{t+1}^2 \leq U_t^2 + K + 2 \delta_t U_t \leq (U_t + \delta_t)^2 + K$.
    Substituting Eq.~\eqref{eq:U-recursion_v2} together with the fact that $U_t \geq 0$ and $\delta_t \geq 0$, we get
    \begin{align*}
      U_{t+1}^2
      &\leq
      (U_t + \delta_t)^2 + K
      \leq
      \bigl( \sqrt{K(t - t^*)} + \Delta^*_{t} + U_{t^*} \bigr)^2+K\\
      &= K(t+1 - t^*)+(\Delta^*_{t} + U_{t^*})^2+2\sqrt{K(t - t^*)}\, (\Delta_{t} + U_{t^*})\\
      &\leq
      \bigl( \sqrt{K(t+1 - t^*)}+\Delta^*_{t} + U_{t^*}\bigr)^2.
    \end{align*}
    We have proved that $ U_{t+1} \leq \sqrt{K(t+1 - t^*)} + \Delta^*_{t} + U_{t^*}$, and we conclude by induction.
\end{proof}

Two maximal inequalities for martingales are called Doob's inequality.
We use the less famous one, for non-negative super-martingales.

\begin{lemma}[One of Doob's maximal inequalities]
  \label{lem:super-martingale}
  Let $(S_n)_{n \geq 1}$ be a non-negative super-martingale, then
  \begin{align*}
    \Prob\!\left(\sup_{m \geq n} S_m \geq \eta\right) \leq \frac{\Exp[S_n]}{\eta}\enspace.
  \end{align*}
\end{lemma}

\clearpage
\section{Proofs for Section~\ref{sec:first_work_outs}}
\label{app:B}

We first detail the two counter-examples alluded at in the proof of
Proposition~\ref{prop:noregrposs}, relative to Example~2 on group-wise no-regret.
We then discuss Example~3 on vanilla no-regret under the equalized average payoffs constraint.
\vfill

\subsection{Counter-examples for group-wise no-regret}
\eparagraph{First counter-example.}
We take $\cS = \cA = \cB = \{0, 1\}$ and let $\cX$ be an arbitrary finite set.
The monitoring is assumed to be $G(x, s) = x$. Finally, we consider the
specific payoff function
\begin{align*}
\forall (a, b, x) \in \cA \times \cB \times \cX, \qquad
r(a, b, x, 0) = a^2 \quad\text{and}\quad r(a, b, x, 1) = (a - 1)^2\enspace.
\end{align*}
The integral conditions in Eq.~\eqref{eq:no_regret1} read: for all $a' \in \{0, 1\}$,
\[
\int_{\cX}\sum_{a \in \{0, 1\}}\bp^x(a) \, a^2 \d\bfQ^0(x) \geq (a')^2 \quad\text{and}\quad
\int_{\cX}\sum_{a \in \{0, 1\}}\bp^x(a) \, (a - 1)^2 \d\bfQ^1(x) \geq (a' - 1)^2\enspace\,,
\]
or equivalently, simply
\begin{align*}
    \int_{\!\supp(\bfQ^0)} \, \bp^x(1)\d\bfQ^0(x) \geq 1\quad\text{and}\quad
    \int_{\!\supp(\bfQ^1)} \, \bp^x(0)\d\bfQ^1(x) \geq 1\enspace.
\end{align*}
As $\bp^x(1) \in [0,1]$, the fact that the first integral above is larger than~$1$
entails that $\bp^x(1) = 1$ on $\supp(\bfQ^0)$. Similarly,
$\bp^x(0) = 1$ on $\supp(\bfQ^1)$. As we also have $\bp^x(0) + \bp^x(1) = 1$ for
all $x \in \cX$, we see that the condition in Eq.~\eqref{eq:no_regret1}
cannot hold as soon as $\supp(\bfQ^0) \cap \supp(\bfQ^1) \neq \emptyset$.

\eparagraph{Second counter-example.}
Again, we take $\cS = \cA = \cB = \{0, 1\}$ and let $\cX$ be an arbitrary finite set
but assume this time that Nature's monitoring is $G(x,s) = (x,s)$.
Another difference is that we consider a payoff function not depending on $s$:
\begin{align*}
\forall (a, b, x, s) \in \cA \times \cB \times \cX \times \cS, \qquad
r(a, b, x, s) = \ind{a = b} = 1 -(a-b)^2 \enspace.
\end{align*}
Nature picks the following difficult family of distributions: $\bq^{(x, 0)} = (1,0)^\top$ and $\bq^{(x, 1)} = (0,1)^\top$ for all $x \in \cX$,
so that $\bq^{(x, s)}(b) = 1$ if and only if $b = s$.
The integral conditions in Eq.~\eqref{eq:no_regret1} therefore read: for all $s \in \{0,1\}$,
\[
\min_{a' \in \{0,1\}} \bigintsss_{\!\cX} \, \sum_{a \in \{0, 1\}}\bp^x(a) \, \bigl( r(a,s,x,s) - r(a',s,x,s) \bigr) \d\bfQ^s(x)
= \int_{\cX} \, \bp^x(s) \d\bfQ^s(x) - 1 \geq 0 \,.
\]
From here we conclude similarly to the previous counter-example.
\vfill

\subsection{Vanilla no-regret under the equalized average payoffs constraint}
\label{sec:Blum2018}

\citet[Section~4]{Blum_Gunasekar_Lykouris_Srebro18} study online regret minimization under a constraint of equal average payoffs,
that is, they discuss the $(\bmreg, \bmgrpay)$--approachability of $\cCreg \times \cCgrpay$, with the
notation of Section~\ref{sec:setting_approachability_goal_objectives_constraints}.

Their setting is different from the setting considered in this article, as the latter relies on the no-regret
based on a fixed base payoff function $r$, while the former considers prediction with expert advice,
that may be assimilated to an adversarially chosen sequence $(r_t)$ of payoff functions.

Yet, we mimic the spirit of their results, which is two-fold.

First, we show an impossibility
result for the simultaneous satisfaction of the vanilla no-regret objective and the
constraint of equal average payoffs, i.e., for the $(\bmreg, \bmgrpay)$--approachability of $\cCreg \times \cCgrpay$.
We do so for an example of binary online classification.
This corresponds to Theorem~4 of \citet[Section~4]{Blum_Gunasekar_Lykouris_Srebro18}.

Second, we provide a positive result for the mentioned approachability problem, in the case of a Player
aware of the sensitive contexts, i.e., following Remark~\ref{rk:aware}, when the Player accesses the
contexts $x'_t = (x_t,s_t)$. This corresponds to Theorem~3 of \citet[Section~4]{Blum_Gunasekar_Lykouris_Srebro18}.

\clearpage
Before we do so, we first instantiate
the approachability condition of Eq.~\eqref{eq:Blackwell_condition}
with the vector payoff function $\bm = (\bmreg, \bmgrpay)$ and the target
set $\cC = \cCreg \times \cCgrpay$; it reads:
$\forall (\bq^{G(x, s)})_{(x, s)} \ \exists (\bp^x)_{x} \ $ such that
\begin{align}
\label{eq:regr-B}
& \int_{\cX \times \{0,1\}} r \bigl( \bp^x, \bq^{G(x,s)},x,s \bigr) \d\bfQ(x,s)
\geq \max_{a' \in \cA} \, \int_{\cX \times \{0,1\}} r\bigl(a', \bq^{G(x,s)},x,s\bigr) \d\bfQ(x,s)  \\
\label{eq:EAP-B}
& \mbox{and} \qquad
\abs{\int_{\cX} r\bigl(\bp^x, \bq^{G(x, 0)},x,0\bigr) \d\bfQ^0(x) -
\int_{\cX} r\bigl((\bp^x, \bq^{G(x, 1)},x,1\bigr) \d\bfQ^1(x)} \leq \epsilon\enspace.
\end{align}
Second, we also introduce some additional notation.

\paragraph{Additional notation and reminder on total variation distance.}
Recall that we denoted by $\bfQ^{0}$ and $\bfQ^{1}$ the two marginals of $\bfQ$ on $\cX$.
We fix some measure $\mu$ which dominates both $\bfQ^{0}$ and $\bfQ^{1}$, e.g.,
$\mu = \bfQ^{0} + \bfQ^{1}$, and denote by $g_0$ and $g_1$ densities of $\bfQ^{0}$ and $\bfQ^{1}$ with respect to $\mu$.
We introduce the following three sets (defined up to $\mu$--neglectable events):
\begin{align*}
    &\class{X}_0 = \enscondb{x \in \class{X}}{g_0(x) > g_1(x)}\,, \\
    &\class{X}_1 = \enscondb{x \in \class{X}}{g_1(x) > g_0(x)}\,, \\
    &\class{X}_{=} = \enscondb{x \in \class{X}}{g_1(x) = g_0(x)}\,.
\end{align*}
Using the above defined sets and densities, we remind that the total variation distance  between $\bfQ^{0}$ and $\bfQ^{1}$ can be expressed in the following equivalent ways (see, e.g., \citet{Dev87} or \citet[Lemma 2.1]{Tsy08}):
\begin{align*}
\TV(\bfQ^{0}, \bfQ^{1})
&={1\over 2} \int_{\class{X}} \bigl| g_0(x)-g_1(x) \bigr|\d\mu(x)\\
&= \int_{\class{X}_{1}} \bigl( g_1(x)-g_0(x) \bigr) \d\mu(x)
= \int_{\class{X}_{0}} \bigl( g_0(x)-g_1(x) \bigr) \d\mu(x)\\
&=1-\int_{\class{X}} \min\bigl\{g_0(x),\,g_1(x)\bigr\} \d\mu(x)\enspace.
\end{align*}

We may now describe the impossibility example.

\paragraph{Impossibility example for online classification.}
Binary classification corresponds to the sets of actions $\cA = \cB = \{0, 1\}$ and
to the payoff function $r(a,b,x,s) = \ind{a = b}$. In particular,
for all distributions $\bq$ and $\bq$, for all contexts $(x,s)$,
\[
r(\bp,\bq,x,s) = \bp(0) \, \bq(0) + \bp(1) \, \bq(1) \eqdef \pv{\bp}{\bq}\,.
\]
We focus our attention on the monitoring $G(x,s) = x$, which gives less freedom to Nature. Our impossibility
result holds in particular in the case of the more complete monitoring $G(x,s) = (x,s)$.

We will have Nature pick distributions $(\bq^x)_{x \in \cX}$ such that
$\bq^x(0) > 1/2$ for all $x \in \cX$; the maximum in the
right-hand side of Eq.~\eqref{eq:regr-B} is then achieved for $a' = 0$.
Because of this and with the notion introduced,
the regret criterion of Eq.~\eqref{eq:regr-B} may be rewritten as
\[
\int_{\cX \times \{0,1\}} \Bigl( \underbrace{\bpv{\bp^x}{\bq^x} - \bq^x(0)}_{\leq 0} \Bigr) \d\bfQ(x,s) \geq 0\,.
\]
The inequality $\bpv{\bp^x}{\bq^x} - \bq^x(0) \leq 0$ holds because $\bq^x(0) > \bq^x(1)$
by the constraint $\bq^x(0) > 1/2$; this inequality is strict unless $\bp^x(1) = 0$.
Therefore, the regret constraint imposes $\bpv{\bp^x}{\bq^x} = \bq^x(0)$ and $\bp^x(1) = 0$ on the support of $\bQ$
(which is the union of the supports of $\bQ^0$ and $\bQ^1$).

The constraint of equal average payoffs relies on
the following difference, which we rewrite based on the equality just proved:
\[
\int_{\cX} \bpv{\bp^x}{\bq^x} \d\bfQ^0(x) -
\int_{\cX} \bpv{\bp^x}{\bq^x} \d\bfQ^1(x)
= \int_{\cX} \bq^x(0) \d\bfQ^0(x) - \int_{\cX} \bq^x(0) \d\bfQ^1(x)\,.
\]
We let Nature pick the distributions $(\bq^x)$ defined by
\begin{align*}
    \bq^{x}(0) = \begin{cases}
        1 & \mbox{for} \ x \in \supp(\bfQ^0), \\
        1/2 + \epsilon & \mbox{for} \ x \in \supp(\bfQ^1).
    \end{cases}
\end{align*}
We also replace $\d\bfQ^0$ and $\d\bfQ^1$ by $g_0 \d\mu$ and $g_1 \d\mu$, respectively.
The difference in average payoffs thus rewrites, given
the various expressions of the total variation distance recalled above:
\begin{align*}
\int_{\cX} \bq^x(0) \d\bfQ^0(x) & - \int_{\cX} \bq^x(0) \d\bfQ^1(x)
= \int_{\cX} \bq^x(0) \, \bigl( g_0(x) - g_1(x) \bigr) \d\mu(x) \\
& = \underbrace{\int_{\cX_0} \bigl( g_0(x) - g_1(x) \bigr) \d\mu(x)}_{= \TV(\bfQ^0,\bfQ^1)}
+ \left( \frac{1}{2} + \epsilon \right) \underbrace{\int_{\cX_1 \cup \cX_{=}} \bigl( g_0(x) - g_1(x) \bigr) \d\mu(x)}_{= - \TV(\bfQ^0,\bfQ^1)} \\
& = \left( \frac{1}{2} - \epsilon \right) \TV(\bfQ^0,\bfQ^1)\,.
\end{align*}
All in all, the equal average payoffs constraint of Eq.~\eqref{eq:EAP-B}, and thus,
the approachability condition of Eq.~\eqref{eq:Blackwell_condition},
hold if and only if
\[
\TV(\bfQ^0,\bfQ^1) \leq \frac{\epsilon}{1/2 - \epsilon}\,,
\]
i.e., if the distributions $\bfQ^0$ and $\bfQ^1$ are close enough.

This is typically not the case, and having such a small distance between
$\bfQ^0$ and $\bfQ^1$ should be considered a degenerate case.
The limit case $\TV(\bfQ^0,\bfQ^1) = 0$ indeed corresponds to the case
when the sensitive attributes $s_t$ are independent of the non-sensitive contexts $x_t$.

\paragraph{Positive result for a Player aware of the $s_t$ and a fair-in-isolation payoff function $r$.}
The positive result will be exhibited in the same spirit as the one of
Theorem~3 of \citet[Section~4]{Blum_Gunasekar_Lykouris_Srebro18}. This spirit is interesting but somewhat limited,
as it relies on a (heavy) fair-in-isolation assumption.
The latter indeed indicates that for all sequences of contexts and observations, the average loss achieved by
a given expert is the same among sensitive groups.
This ``for all sequences'' requirement is particularly demanding.
(A question not answered in \citet{Blum_Gunasekar_Lykouris_Srebro18} is the existence of
experts that are fair in isolation, for general reward functions $r_t$ or general loss functions $\ell_t$,
and metrics $\mathcal{M}$, using their notation.)
See comments after Eq.~\eqref{eq:FII} below for the adaptation of this assumption in our context.

A second ingredient for the positive result is that the Player accesses the sensitive contexts $s_t$.
Following Remark~\ref{rk:aware}, this translates into our setting by considering that the Player accesses the
contexts $x'_t = (x_t,s_t)$; hence, the distributions picked by the Player will be indexed by $(x,s)$ in this example.
Nature's monitoring is $G(x,s) = (x,s)$ as well.

Given $(\bq^{(x,s)})$, to fulfill the no-regret condition
\[
\int_{\cX \times \{0,1\}} r \bigl( \bp^{(x,s)}, \bq^{(x,s)},x,s \bigr) \d\bfQ(x,s)
\geq \max_{a' \in \cA} \, \int_{\cX \times \{0,1\}} r\bigl(a', \bq^{(x,s)},x,s\bigr) \d\bfQ(x,s)\,,
\]
the Player may pick $\bp^{(x,s)}$ based only on $s$:
\[
\bp^{(x,s)} = \dirac(a^s)\,, \qquad \mbox{where} \qquad a^s \in \argmax_{a \in \cA}
\int_{\cX} r \bigl( a, \bq^{(x,s)},x,s \bigr) \d\bfQ^s(x)\,.
\]
This corresponds to using separate no-regret algorithms in the construction of
Theorem~3 of \citet{Blum_Gunasekar_Lykouris_Srebro18}, one algorithm per sensitive context.
The no-regret algorithms based on approachability used here actually have a regret converging to zero in the limit (they do not just approach
the set of non-negative numbers) and thus share the same ``not worse but not better'' property with respect to the best action $a$
as the one used in Theorem~3 of \citet{Blum_Gunasekar_Lykouris_Srebro18}.

The constraint of equal average payoffs requires that the following difference is smaller than $\epsilon$ in
absolute values:
\begin{multline}
\label{eq:FII}
\int_{\cX} r\bigl(\bp^{(x,0)}, \bq^{(x, 0)},x,0\bigr) \d\bfQ^0(x) -
\int_{\cX} r\bigl(\bp^{(x,1)}, \bq^{(x, 1)},x,1\bigr) \d\bfQ^1(x) \\
=
\max_{a \in \cA} \int_{\cX} r \bigl( a, \bq^{(x,0)},x,0 \bigr) \d\bfQ^0(x) -
\max_{a \in \cA} \int_{\cX} r \bigl( a, \bq^{(x,1)},x,1 \bigr) \d\bfQ^1(x)\,.
\end{multline}
This constraint is automatically taken care of by the fair-in-isolation assumption:
its analogue in our context (keeping in mind that the actions $a \in \cA$
play here the role of the experts in \citet{Blum_Gunasekar_Lykouris_Srebro18})
is to require that for all distributions $(\bq^{(x, s)})$ picked by the opponent,
\[
\forall a \in \cA, \qquad
\left| \int_{\cX} r \bigl( a, \bq^{(x,0)},x,0 \bigr) \d\bfQ^0(x) -
\int_{\cX} r \bigl( a, \bq^{(x,1)},x,1 \bigr) \d\bfQ^1(x) \right| \leq \epsilon \,.
\]
This basically corresponds to an assumption on the effective range of the payoff function $r$
and is therefore a heavy assumption, of limited interest.

\clearpage
\clearpage
\section{Proofs for Section~\ref{sec:trade-off}}
\label{app:C}

We recalled various expressions of the total variation distance
in Section~\ref{sec:Blum2018}. We keep the notation
defined therein. The following inequality
will be used repeatedly in our proofs.

\begin{lemma}
\label{lm:csqTV**}
For all Borel-measurable functions $f : \cX \to [0,1]$,
\[
\left| \int_{\cX} f \d\bfQ^0 - \int_{\cX} f \d\bfQ^1 \right|
= \left| \int_{\cX} f \, (g_0-g_1) \d\mu \right| \leq \TV(\bfQ^{0}, \bfQ^{1})\,.
\]
\end{lemma}

\begin{proof}
The proof heavily relies on the fact that $f$ takes values in $[0,1]$.
Since, by definitions of $\cX_0$, $\cX_1$, and $\cX_{=}$,
\[
\int_{\cX} f \, (g_0-g_1) \d\mu
= \int_{\cX_0} f \, (\underbrace{g_0-g_1}_{> 0}) \d\mu
+ \int_{\cX_1} f \, (\underbrace{g_0-g_1}_{< 0}) \d\mu\,,
\]
we have
\[
- \TV(\bfQ^{0}, \bfQ^{1}) =
- \int_{\cX_1} (g_1-g_0)\d\mu
\leq
\int_{\cX} f \, (g_0-g_1) \d\mu
\leq \int_{\cX_0} (g_0-g_1)\d\mu = \TV(\bfQ^{0}, \bfQ^{1})\,,
\]
which concludes the proof.
\end{proof}

An application of Lemma~\ref{lm:csqTV**} is
the bound $\TV(\bfQ^0, \bfQ^1)$ on the $\psi(\,\ldots\,)$ quantity of Eq.~\eqref{eq:CNS-Gcal-DP}.
Indeed,
\begin{align}
\nonumber
\texttt{DP} & \eqdef \psi\!\pa{\int_{\cX \times \{0,1\}} \bmdp(\bp^{x}, s) \d\bfQ(x, s)} \\
\nonumber
& = \Biggl| \bigintsss_\cX \sum_{k=1}^N \bp^x(k) \, a^{(k)} \d\bfQ^0(x) -
\bigintsss_\cX \sum_{k=1}^N \bp^x(k) \, a^{(k)} \d\bfQ^1(x) \Biggr| \\
\nonumber
& = \Biggl| \int_\cX A(\bp^x) \d\bfQ^0(x) - \int_\cX A(\bp^x) \d\bfQ^1(x)
\Biggr| \leq \TV(\bfQ^0, \bfQ^1)\,,
\end{align}
where we introduced
\[
A(\bp^x) \eqdef \sum_{k=1}^N \bp^x(k) \, a^{(k)} \in [0,1]\,.
\]

Before providing the proofs of Propositions~\ref{prop:trade_off_awareness}
and~\ref{prop:trade_off_unawareness} we introduce some additional short-hand notation.
We set
\[
t^* \eqdef \TV(\bfQ^0, \bfQ^1)\,.
\]

The objective function of the maxmin problem in~\eqref{eq:minmax_global}, relative to group-wise calibration, is denoted by and equals
\begin{align*}
\texttt{GC} & \eqdef \norm{\int_{\cX \times \cS} \bmgcal\bigl(\bp^{x}, \bq^{G(x, s)}\bigr) \d\bfQ(x, s)}_1 \\
& = \sum_{k = 1}^N \abs{\int_{\class{X}} \bp^{x}(k) \, \bigl(a^{(k)} - \bq^{G(x, 0)}(1) \bigr) \, g_0(x) \d\mu(x)} \\
& \qquad + \sum_{k = 1}^N \abs{\int_{\class{X}} \bp^{x}(k) \, \bigl(a^{(k)} - \bq^{G(x, 1)}(1) \bigr) \, g_1(x) \d\mu(x)}\,.
\end{align*}
The problem of Eq.~\eqref{eq:minmax_global} can now be written as
\begin{align}
\label{eq:OPT_simplified}
\epsilon^\star(\delta_\tau) =
\max_{(\bq^{G(x, s)})} \ \min_{(\bp^{x})} \ \enscond{\texttt{GC}}{\texttt{DP} \leq \tau t^*} \enspace.
\end{align}

The proof technique for each of Propositions~\ref{prop:trade_off_awareness} and~\ref{prop:trade_off_unawareness} consists of two steps.
First, by setting some convenient family $(\bq^{G(x, s)})_{(x,s)}$, we obtain a lower bound on $\epsilon^\star(\delta_\tau)$.
Second, by exhibiting some convenient family $(\bp^{x})_x$, possibly based on the
knowledge of $(\bq^{G(x, s)})_{(x,s)}$, an upper bound on $\epsilon^\star(\delta_\tau)$ is derived.

The definitions of these families will be based on a rounding operator
$p \in [0, 1] \mapsto \Pi_{\class{A}} \in \class{A}$, that maps a number $p \in [0, 1]$ to the closest element in the grid $\class{A}$.
Note that by definition of $\class{A}$ and $\Pi_{\class{A}}$, it holds that $|p - \Pi_{\class{A}}(p)| \leq 1/(2N)$
for all $p \in [0, 1]$.
We are finally in the position of proving
Propositions~\ref{prop:trade_off_awareness} and~\ref{prop:trade_off_unawareness}, and start with the former.

\begin{proof}[Proof of Proposition~\ref{prop:trade_off_awareness}]
Fix some $\tau \in [0, 1]$. Recall that $G(x,s) = (x,s)$, so that families of distributions
picked by Nature are ``truly'' indexed by $(x,s)$.

For the lower bound on $\epsilon^\star(\delta_\tau)$, we consider, for all $x \in \cX$,
\[
\bq^{(x,0)} = \dirac(1) \qquad \mbox{and} \qquad \bq^{(x,0)} = \dirac(0) \,.
\]
Then, for all choices $(\bp^{x})_x$,
using that $\displaystyle{\sum_{k = 1}^N \bp^{x}(k) = 1}$ for each $x \in \cX$:
\begin{align*}
\texttt{GC} &
= \sum_{k = 1}^N \abs{\int_{\class{X}} \bp^{x}(k) \, \bigl(a^{(k)} - 1 \bigr) \, g_0(x) \d\mu(x)}
+ \sum_{k = 1}^N \abs{\int_{\class{X}} \bp^{x}(k) \, a^{(k)} \, g_1(x) \d\mu(x)} \\
& = \sum_{k = 1}^N \int_{\class{X}} \bp^{x}(k) \, \bigl(1-a^{(k)}\bigr) \, g_0(x) \d\mu(x)
+ \sum_{k = 1}^N \int_{\class{X}} \bp^{x}(k) \, a^{(k)} \, g_1(x) \d\mu(x) \\
& = \underbrace{\int_{\class{X}} g_0(x) \d\mu(x)}_{= 1}
+ \underbrace{\int_{\class{X}} A(\bp^x) \, \bigl( g_1(x)-g_0(x) \bigr) \d\mu(x)}_{\mbox{\tiny absolute value equals $\texttt{DP}$}}
\geq 1 - \texttt{DP}\,,
\end{align*}
so that the rewriting of Eq.~\eqref{eq:OPT_simplified} entails $\epsilon^\star(\delta_\tau) \geq 1 - \tau t^*$,
as claimed.

To derive an upper bound on $\epsilon^\star(\delta_\tau)$,
we consider, for each $(\bq^{(x, s)})_{(x, s) \in \cX \times \cS}$ and each $x \in \class{X}$,
\begin{multline}
\label{eq:player_trade_off1}
\bp^{\tau, x} = (1 - \tau) \cdot \dirac\bigl( \Pi_{\class{A}}(1/2) \bigr) + \tau \cdot \dirac \bigl( f(x) \bigr) \,, \\
\mbox{where} \qquad f(x) = \begin{cases}
\Pi_{\class{A}}\bigl( \bq^{(x, 1)})(1) \bigr) &\text{if } x \in \class{X}_1\cup \class{X}_{=}\,; \\
\Pi_{\class{A}}\bigl( \bq^{(x, 0)})(1) \bigr) &\text{if } x \in \class{X}_0\,.
\end{cases}
\end{multline}
Note that for this strategy of the Player, $\texttt{DP} \leq \tau t^*$;
indeed, $A(\bp^{\tau,x}) = (1-\tau) \cdot \Pi_{\cA}(1/2) + \tau \cdot f(x)$, so that
\begin{align*}
\texttt{DP}
= \Biggl| \int_\cX A(\bp^{\tau,x}) \d\bfQ^0(x) - \int_\cX A(\bp^{\tau,x}) \d\bfQ^1(x) \Biggr|
& = \tau \cdot \Biggl| \int_\cX f(x) \d\bfQ^0(x) - \int_\cX f(x) \d\bfQ^1(x) \Biggr| \\
& \leq \tau \cdot \TV(\bfQ^{0}, \bfQ^{1})\,,
\end{align*}
where we applied Lemma~\ref{lm:csqTV**} for the final inequality.
Moreover, the choice of~Eq.~\eqref{eq:player_trade_off1} ensures that $\texttt{GC} \leq 1 - \tau t^* + 1/N$,
as we will prove below. This will lead to $\epsilon^\star(\delta_\tau) \leq 1 - \tau t^* + 1/N$ and will conclude the proof.
Indeed,
\begin{align*}
\texttt{GC}
& = \sum_{k = 1}^N \abs{\int_{\class{X}} \bp^{\tau,x}(k) \, \bigl(a^{(k)} - \bq^{(x, 0)}(1) \bigr) \, g_0(x) \d\mu(x)} \\
& \qquad + \sum_{k = 1}^N \abs{\int_{\class{X}} \bp^{\tau,x}(k) \, \bigl(a^{(k)} - \bq^{(x, 1)}(1) \bigr) \, g_1(x) \d\mu(x)} \\
& = (1-\tau) \abs{\int_{\class{X}} \bigl(\Pi_{\class{A}}(1/2) - \bq^{(x, 0)}(1) \bigr) \, g_0(x) \d\mu(x)} \\
& \qquad + (1-\tau) \abs{\int_{\class{X}} \bigl(\Pi_{\class{A}}(1/2) - \bq^{(x, 1)}(1) \bigr) \, g_1(x) \d\mu(x)} \\
& \qquad + \tau \abs{\int_{\class{X}} \bigl( f(x) - \bq^{(x, 0)}(1) \bigr) \, g_0(x) \d\mu(x)}
+ \tau \abs{\int_{\class{X}} \bigl( f(x) - \bq^{(x, 1)}(1) \bigr) \, g_1(x) \d\mu(x)}\,.
\end{align*}
We replace $f(x)$ by its specific values and take care of all rounding operators $\Pi_{\class{A}}$ by adding a $2 \times 1/(2N) = 1/N$ term
after application of triangle inequalities:
\begin{align*}
\texttt{GC}
& \leq \frac{1}{N} + (1-\tau) \biggl| \int_{\class{X}} \bigl( \overbrace{1/2 - \bq^{(x, 0)}(1)}^{\in [-1/2, \, 1/2]} \bigr) \, g_0(x) \d\mu(x) \biggr|
+ (1-\tau) \biggl| \int_{\class{X}} \bigl( \overbrace{1/2 - \bq^{(x, 1)}(1)}^{\in [-1/2, \, 1/2]} \bigr) \, g_1(x) \d\mu(x) \biggr| \\
& \qquad + \tau \biggl| \int_{\class{X}_1 \cup \cX_{=}} \bigl( \overbrace{\bq^{(x, 1)}(1) - \bq^{(x, 0)}(1)}^{\in [-1,1]} \bigr) \, g_0(x) \d\mu(x) \biggr|
+ \tau \biggl| \int_{\class{X}_0} \bigl( \overbrace{\bq^{(x, 0)}(1) - \bq^{(x, 1)}(1)}^{\in [-1,1]} \bigr) \, g_1(x) \d\mu(x) \biggr| \\
& \leq \frac{1}{N} + \frac{1-\tau}{2} + \frac{1-\tau}{2} + \tau \int_{\class{X}_1 \cup \cX_{=}} g_0(x) \d\mu(x) + \tau \int_{\class{X}_0} g_1(x) \d\mu(x) \\
& = 1-\tau + \tau \underbrace{\int_{\cX} \min \bigl\{ g_0(x), \, g_1(x) \bigr\} \d\mu(x)}_{= 1 - t^*} + \frac{1}{N}
=  1 - \tau t^* + \frac{1}{N}\,,
\end{align*}
where we used one of the expressions of $t^* = \TV(\bfQ^0, \bfQ^1)$ in the last equality.
\end{proof}

\begin{proof}[Proof of Proposition~\ref{prop:trade_off_unawareness}]
Fix some $\tau \in [0, 1]$.
Recall that $G(x,s) = x$, so that families of distributions
picked by Nature are only indexed by $x$ and may not depend on $s$.

For the lower bound on $\epsilon^\star(\delta_\tau)$, we consider $(\bq^{x})_{x \in \cX}$ defined as
\[
\bq^{x} = \begin{cases}
\dirac(1) & \text{if } x \in \cX_1 \cup \cX_{=}\,; \\
\dirac(0) & \text{if } x \in \cX_0\,.
\end{cases}
\]
Then, for all choices $(\bp^{x})_x$,
using the notation $A(\bp^x)$ and the fact that $\displaystyle{\sum_{k = 1}^N \bp^{x}(k) = 1}$ for each $x \in \cX$:
\begin{allowdisplaybreaks}
\begin{align*}
\texttt{GC}
& = \sum_{k = 1}^N \abs{\int_{\class{X}} \bp^{x}(k) \, \bigl(a^{(k)} - \bq^{x}(1) \bigr) \, g_0(x) \d\mu(x)}
+ \sum_{k = 1}^N \abs{\int_{\class{X}} \bp^{x}(k) \, \bigl(a^{(k)} - \bq^{x}(1) \bigr) \, g_1(x) \d\mu(x)} \\
& = \sum_{k = 1}^N \abs{
\int_{\cX_1 \cup \cX_{=}} \bp^{x}(k) \, \bigl(a^{(k)} - 1\bigr) \, g_0(x) \d\mu(x)
+ \int_{\cX_0} \bp^{x}(k) \, a^{(k)} \, g_0(x) \d\mu(x) } \\*
& \qquad + \sum_{k = 1}^N \abs{
\int_{\cX_1 \cup \cX_{=}} \bp^{x}(k) \, \bigl(a^{(k)} - 1\bigr) \, g_1(x) \d\mu(x)
+ \int_{\cX_0} \bp^{x}(k) \, a^{(k)} \, g_1(x) \d\mu(x)} \\
&\geq \abs{\int_{\cX_1 \cup \cX_{=}} \bigl(A(\bp^x) - 1\bigr) \, g_0(x) \d\mu(x)
+ \int_{\cX_0} A(\bp^x) \, g_0(x) \d\mu(x)} \\*
& \qquad + \abs{\int_{\cX_1 \cup \cX_{=}} \bigl(A(\bp^x) - 1\bigr) \, g_1(x) \d\mu(x)
+ \int_{\cX_0} A(\bp^x) \, g_1(x) \d\mu(x)}\,\\
&= \abs{\int_{\cX} A(\bp^x) \, g_0(x) \d\mu(x) - \int_{\cX_1 \cup \cX_{=}} \, g_0(x) \d\mu(x)} \\*
& \qquad  + \abs{\int_{\cX} A(\bp^x) \, g_1(x) \d\mu(x) - \int_{\cX_1 \cup \cX_{=}}\, g_1(x) \d\mu(x)}\,,\\
& \geq \Biggl| \underbrace{\int_{\cX_1 \cup \cX_{=}} \bigl( g_1(x) - g_0(x) \bigr) \d\mu(x)}_{= t^*}
- \underbrace{\int_{\cX_1 \cup \cX_{=}} A(\bp^x) \, \bigl( g_1(x) - g_0(x) \bigr) \d\mu(x)}_{\mbox{\tiny absolute value equals $\texttt{DP}$}} \Biggr|
\geq t^* - \texttt{DP}\enspace.
\end{align*}
\end{allowdisplaybreaks}
where all the inequalities follows from the triangle inequality.
Eq.~\eqref{eq:OPT_simplified} entails $\epsilon^\star(\delta_\tau) \geq (1 - \tau) t^*$, as claimed.

To derive an upper bound on $\epsilon^\star(\delta_\tau)$,
we consider, for each $(\bq^{x})_{x \in \cX}$ and each $x \in \class{X}$,
\begin{multline}
\label{eq:player_trade_off2}
\bp^{\tau, x} = (1-\tau) \cdot \dirac \bigl( \Pi_{\cA}(Q) \bigr) + \tau \cdot \dirac \Bigl( \Pi_{\class{A}} \bigl( \bq^{x}(1) \bigr) \Bigr) \\
\mbox{where} \qquad Q = \int_{\class{X}} \bq^{u}(1) \, g_0(u) \d\mu(u)\,.
\end{multline}
Note that for this strategy of the Player, $\texttt{DP} \leq \tau t^*$;
indeed, $A(\bp^{\tau,x}) = (1-\tau) \cdot \Pi_{\cA}(Q) + \tau \cdot \Pi_{\class{A}} \bigl( \bq^{x}(1) \bigr)$, so that
\begin{align*}
\texttt{DP}
& = \Biggl| \int_\cX A(\bp^{\tau,x}) \d\bfQ^0(x) - \int_\cX A(\bp^{\tau,x}) \d\bfQ^1(x) \Biggr| \\
& = \tau \cdot \Biggl| \int_\cX \Pi_{\class{A}} \bigl( \bq^{x}(1) \bigr) \d\bfQ^0(x) -
\int_\cX \Pi_{\class{A}} \bigl( \bq^{x}(1) \bigr) \d\bfQ^1(x) \Biggr|
\leq \tau \cdot \TV(\bfQ^{0}, \bfQ^{1})\,,
\end{align*}
where we applied Lemma~\ref{lm:csqTV**} for the final inequality.
Moreover, the choice of~Eq.~\eqref{eq:player_trade_off2} ensures that $\texttt{GC} \leq (1 - \tau)t^* + 1/N$,
as we will prove below. This will lead to $\epsilon^\star(\delta_\tau) \leq (1 - \tau)t^* + 1/N$ and will conclude the proof.
Indeed,
\begin{align*}
\texttt{GC}
& = \sum_{k = 1}^N \abs{\int_{\class{X}} \bp^{\tau,x}(k) \, \bigl(a^{(k)} - \bq^{x}(1) \bigr) \, g_0(x) \d\mu(x)}
+ \sum_{k = 1}^N \abs{\int_{\class{X}} \bp^{\tau,x}(k) \, \bigl(a^{(k)} - \bq^{x}(1) \bigr) \, g_1(x) \d\mu(x)} \\
& = (1-\tau) \abs{\int_{\class{X}} \bigl( \Pi_{\cA}(Q) - \bq^{x}(1) \bigr) \, g_0(x) \d\mu(x)}
+ (1-\tau) \abs{\int_{\class{X}} \bigl( \Pi_{\cA}(Q) - \bq^{x}(1) \bigr) \, g_1(x) \d\mu(x)} \\
& \qquad + \tau \,\biggl| \int_{\class{X}} \Bigl( \underbrace{\Pi_{\class{A}} \bigl( \bq^{x}(1) \bigr) - \bq^{x}(1)}_{\leq 1/(2N)} \Bigr) \, g_0(x) \d\mu(x) \biggr|
+ \tau \, \biggl| \int_{\class{X}} \Bigl( \underbrace{\Pi_{\class{A}} \bigl( \bq^{x}(1) \bigr) - \bq^{x}(1)}_{\leq 1/(2N)} \Bigr) \, g_1(x) \d\mu(x) \biggr|\,.
\end{align*}
Taking into account the rounding errors, i.e., replacing the two occurrences of $\Pi_{\cA}(Q)$ by $Q$ by adding twice a $(1-\tau)/(2N)$ term, we get
\begin{align*}
\texttt{GC}
& \leq (1-\tau) \abs{\int_{\class{X}} \bigl( Q - \bq^{x}(1) \bigr) \, g_0(x) \d\mu(x)}
+ (1-\tau) \abs{\int_{\class{X}} \bigl( Q - \bq^{x}(1) \bigr) \, g_1(x) \d\mu(x)}
+ \frac{1}{N} \\
& = (1-\tau) \, \biggl| \underbrace{Q - \int_{\class{X}} \bq^{x}(1) \, g_0(x) \d\mu(x)}_{=0} \biggr|
+ (1-\tau) \abs{Q - \int_{\class{X}} \bq^{x}(1) \, g_1(x) \d\mu(x)}
+ \frac{1}{N} \\
& = (1-\tau) \abs{\int_{\class{X}} \bq^{x}(1) \, g_0(x) \d\mu(x) - \int_{\class{X}} \bq^{x}(1) \, g_1(x) \d\mu(x)}
+ \frac{1}{N} \leq (1 - \tau)t^* + 1/N\,,
\end{align*}
where the last equality holds by definition of $Q$ as an integral and
we applied Lemma~\ref{lm:csqTV**} for the final inequality.
\end{proof}

\clearpage
\section{Proofs for Section~\ref{sub:approachability_with_unknown_target}}
\label{app:D}

In this section, we go over the results alluded at in Section~\ref{sub:approachability_with_unknown_target}.
We first illustrate that Assumption~\ref{ass:set_estimation} (which indicates
that the target set $\cC$ should be estimated in some way) is realistic.
We do so in Section~\ref{sec:detailsreduc} by dealing with the most involved situation discussed in this article,
namely, the example discussed at the beginning of Section~\ref{sub:approachability_with_unknown_target}.
We then provide in Section~\ref{sec:proofthunknown}
a more complete statement of Theorem~\ref{th:unknown}, with convergence rates,
and prove it.

\subsection{Assumption~\ref{ass:set_estimation} is realistic}
\label{sec:detailsreduc}

The beginning of Section~\ref{sub:approachability_with_unknown_target} explained
why and how performing an optimal trade-off between accuracy in group-calibration
and unfairness in terms of demographic parity amounts to studying
the $(\tbmgcal, \tbmdp)$--approachability of $\cC = \tcCgcaleps \times \tcCdpdelta$,
where $(\tbmgcal, \tbmdp)$ is a known vector payoff function and
where $\cC:=\tcCgcaleps \times \tcCdpdelta$ is unknown:
\[
{\tcCgcaleps} = \enscond{ (\bv_0, \bv_1) \in \R^{2N} }{ \tfrac{\| \bv_0 \|_1}{\gamma_0} {+} \tfrac{\| \bv_1 \|_1}{\gamma_1} \leq \varepsilon }, \qquad
{\tcCdpdelta} = \enscond{ (u,v) \in \R^2}{ \ \big|{\tfrac{u}{\gamma_0} - \tfrac{v}{\gamma_1}}\big| \leq \delta }\enspace,
\]
with $\epsilon = (1 - \tau) \cdot \TV(\bfQ^0, \bfQ^1)$ and $\delta = \tau \cdot \TV(\bfQ^0, \bfQ^1)$ for some \emph{known} $\tau \in [0, 1]$ but an unknown
$\TV(\bfQ^0, \bfQ^1)$, and with unknown probabilities $\gamma_0,\gamma_1$.
The parameter $\tau$ controls the desired trade-off between the calibration error and the discrepancy in demographic parity and thus is left as a parameter of user's choice.

We recall that the strategy of the Player proceeds in phases: at each time $T_r \eqdef 2^r$ for $r \geq 1$, the Player updates the estimate $\hat{\cC}_{r}$ of $\cC$.
The focus of this section is to provide a sequence of estimates $\hat{\cC}_{r}$ of $\cC$ fulfilling Assumption \ref{ass:set_estimation}.
The latter is a key requirement for the existence of an approachability strategy stated in Theorem~\ref{th:unknown}.
We must therefore prove that it is a realistic assumption.

\paragraph{The four requirements of Assumption \ref{ass:set_estimation}.}
For the convenience of the reader, we restate the various requirements of Assumption \ref{ass:set_estimation},
giving them nicknames, to be able to refer to them easily in the sequel:
for all $r \geq 0$, the sets ${\hat{\cC}_{r}}$

\begin{itemize}[topsep=0pt,itemsep=-1ex,partopsep=1ex,parsep=1ex]
\item[] {\makebox[3cm]{(CC)\hfill}} are convex closed; \smallskip
\item[] {\makebox[3cm]{(Proj-dist)\hfill}} satisfy $\|\bv - \Proj_{\hat{\class{C}}_r} (\bv)\| \leq B$, for all $\bv \in \bm(\cA, \cB, \cX, \{0,1\})$;
\item[] {\makebox[3cm]{(Super-set)\hfill}} satisfy $\Prob{\big(\class{C}\subset\hat{\class{C}}_r \big)}\geq 1- {{1}/ {(2T_{r})}}$;
\item[] {\makebox[3cm]{(L2-Hausdorff)\hfill}} satisfy $\max\Bigl\{ \Exp\bigl[d(\hat{\class{C}}_r, \class{C})^2\bigr], \,\, \Exp\bigl[d(\cC,\hat{\class{C}}_r)^2\bigr] \Bigr\}
\leq \beta^2_{r}$.
\end{itemize}
The constant $B < +\infty$ is independent of $r$
and the sequence $(\beta_{r})_{r\geq 0}$ is summable and non-increasing.
The vector payoff function $\bm$ above refers to $(\tbmgcal, \tbmdp)$.

(Proj-dist) requires that the distance of a possible vector payoff to sets ${\hat{\cC}_{r}}$ are
uniformly controlled. (Super-set) requires that the ${\hat{\cC}_{r}}$ are, with high probability,
super-sets of $\cC$. Finally, (L2-Hausdorff) requires that some $L^2$ criterion of Hausdorff distance between
sets is controlled. We will go over each of these requirements but first deepen our reduction scheme.

In the sequel, and as in the main body of the paper, we focus on the case where $\gamma_0 > 0$ and $\gamma_1>0$,
i.e., there are two effective values for the sensitive contexts.

\paragraph{But first, a further reduction.}
Recall that the average vector payoff, described in Section~\ref{sub:approachability_with_unknown_target}, is equal to
\begin{align*}
  \frac{1}{T}\sum_{t = 1}^T \bigl( \tbmgcal(a_t, b_t, x_t, s_t), \tbmdp(a_t, b_t, x_t, s_t) \bigr)\enspace,
\end{align*}
where the first $2N$ components always lie in the interval $[-1, 1]$, while the last two ones lie in the interval $[0, 1]$.
Therefore, in the definition of ${\tcCdpdelta}$, we may restrict our attention to $(u,v) \in [0, 1]^2$ and
use rather the alternative definition
 \[
 {\tcCdpdelta} \eqdef \enscond{ (u,v) \in [0, 1]^2}{ \ \big|{\tfrac{u}{\gamma_0} - \tfrac{v}{\gamma_1}}\big| \leq \delta }
 = \Bigl\{ (u,v) \in [0, 1]^2 : \ |\gamma_1 u - \gamma_0 v| \leq \gamma_0 \gamma_1 \delta \Bigr\}
 \,.
 \]
As for $\tcCgcaleps$, given we are studying a calibration problem,
we note that the $\varepsilon$ of interest lie in $[0,1]$ (with $0$ included).
Vectors $(\bv_0,\bv_1)$ of $\tcCgcaleps$ satisfy in particular that
$\| \bv_0 \|_1 + \| \bv_1 \|_1 \leq \epsilon$, which shows that $\tcCgcaleps \subseteq B_{\R^{2N}}^{\ell_{1}}$, where
$B_{\R^{2N}}^{\ell_{1}}=\ac{\bv\in \R^{2N}:\|\bv\|_{1}\leq 1}$ is the unit $\ell_{1}$ ball in $\R^{2N}$. Therefore,
\begin{align*}
{\tcCgcaleps} & = \enscond{ (\bv_0, \bv_1) \in B_{\R^{2N}}^{\ell_{1}} }{ \tfrac{\| \bv_0 \|_1}{\gamma_0} {+} \tfrac{\| \bv_1 \|_1}{\gamma_1} \leq \varepsilon } \\
& = \Bigl\{ (\bv_0, \bv_1) \in B_{\R^{2N}}^{\ell_{1}} : \ \gamma_1 \| \bv_0 \|_1 + \gamma_0 \| \bv_1 \|_1 \leq \gamma_0 \gamma_1 \epsilon \Bigr\}\,.
\end{align*}

\paragraph{Plug-in estimation of $\cC$.}
We consider estimators $\hgamma_{0, t}, \hgamma_{1, t}\in[0,1]$ of $\gamma_0, \gamma_1$ and an estimator $\hM_{t}\in[0,1]$ of $\TV(\bfQ^0, \bfQ^1)$,
 based on the first $t$ i.i.d.\ samples from $\bfQ$, see the end of the section for examples.
We  substitute them in the definitions of the target sets.
We actually perform a careful such substitution by considering possibly data-dependent parameters
$\alpha_1(t) \in (0,1]$ and $\alpha_2(t) \in (0,1]$, to be specified by the analysis, that will provide the needed upper confidence bounds (i.e., super-set condition).
More precisely, we define estimators of $\tcCgcaleps$ and $\tcCdpdelta$ by 
\begin{align*}
    \hcCgcaleps (t) & = \enscond{ (\bv_0, \bv_1) \in  B_{\R^{2N}}^{\ell_{1}}}{ \hgamma_{1, t}\| \bv_0 \|_1 {+} \hgamma_{0, t}\| \bv_1 \|_1 \leq \hgamma_{0, t}\hgamma_{1, t}\hat\varepsilon_t + \alpha_1(t) + 4\alpha_2(t)}\enspace,\\
    {\hcCdpdelta}(t) & = \enscond{ (u,v) \in [0, 1]}{ \ \big|\hgamma_{1, t}u - \hgamma_{0, t}v\big| \leq \hgamma_{0, t}\hgamma_{1, t}\hat\delta_t + \alpha_1(t) + 4\alpha_2(t)}\enspace,
\end{align*}
where $\hat\epsilon_t = (1 - \tau)\hM_t $ and $\hat\delta_t = \tau\hM_t $. We then set $\hat{\cC}_r \eqdef \hcCgcaleps(T_{r}) \times \hcCdpdelta(T_{r})$.

\paragraph{Requirements (CC) and (Proj-dist) hold.}
We observe that both $\hcCgcaleps(t)$ and $\hcCdpdelta(t)$ are convex, closed, and bounded. The boundedness of these sets and
the fact that $\bm = (\tbmgcal, \tbmdp)$ is bounded as well ensure the (Proj-dist) property.

\paragraph{Choice of $\alpha_1(T_r)$ and $\alpha_2(T_r)$, part 1.}
We introduce the following sets, indicating that some confidence bounds around the introduced estimators
hold, of widths smaller than the introduced parameters $\alpha_1(T_r)$ and $\alpha_2(T_r)$.
These sets need only to be considered at times $T_r$, where $r \geq 1$:
\begin{align*}
  \Omega_{T_r}^{\alpha_1, \alpha_2} \eqdef \ens{\abs{\widehat M_{T_r} - \TV(\bfQ^0, \bfQ^1)} \leq \alpha_1(T_r)
   \ \ \ \mbox{and} \ \ \ \forall s \in \{0, 1\}, \ \abs{{\hat \gamma_{s, T_r}} - {\gamma_s}} \leq \alpha_2(T_r)} \enspace.
\end{align*}
We assume in the sequel that we could pick all $\alpha_1(T_r)$ and $\alpha_2(T_r)$ such that
for all $r \geq 1$,
\begin{align}
\label{eq:probOmega}
\Prob \bigl( \Omega_{T_r}^{\alpha_1, \alpha_2} \bigr) \geq 1 - \frac{1}{2T_r}\enspace\,,
\end{align}
and explain, in the final part of this section, how this can be ensured.

\paragraph{Requirement (Super-set) holds.}
It follows from the assumption above on the probability of $\Omega_{T_r}^{\alpha_1, \alpha_2}$
and from the following lemma.
\begin{lemma}
  \label{D1:superset_lemma}
  On the event $\Omega_{T_{r}}^{\alpha_1, \alpha_2}$ defined above, it holds that
  $\tcCgcaleps \subseteq \hcCgcaleps(T_{r})$ and $\tcCdpdelta \subseteq \hcCdpdelta(T_{r})$, thus
\emph{
  \begin{align*}
 \cC= \pa{\tcCgcaleps \times \tcCdpdelta} \subseteq \hat{\cC}_r= \pa{\hcCgcaleps(T_{r}) \times \hcCdpdelta(T_{r})}\enspace.
\end{align*}
}
\end{lemma}

\begin{proof}
For brevity, we drop the dependencies in~$T_r$ in the notation.

\underline{Part I: $\tcCgcaleps \subset \hcCgcaleps$.}
We fix some $(\bv_0, \bv_1) \in \tcCgcaleps$. By assumption,
   \begin{align}
    \label{D1:part1}
     \gamma_1{\|\bv_0\|_1} + \gamma_0{\|\bv_1\|_1} \leq (1 - \tau) \, \gamma_0
\gamma_1 \cdot \TV(\bfQ^0, \bfQ^1)\enspace.
   \end{align}
Furthermore, since $\|\bv_0\|_1+\|\bv_1\|_1 \leq 1$,  it holds  on $\Omega^{\alpha_1,\alpha_2}$ that
\begin{equation}
    \label{D1:TVbound}
\begin{aligned}
  &\gamma_1{\|\bv_0\|_1} + \gamma_0{\|\bv_1\|_1} \geq \hat\gamma_1{\|\bv_0\|_1} + \hat\gamma_0{\|\bv_1\|_1} - 2\alpha_2\enspace,\\
  &\gamma_0\gamma_1 \cdot \TV(\bfQ^0, \bfQ^1) \leq \hat\gamma_0\gamma_1 \cdot \TV(\bfQ^0, \bfQ^1) + \alpha_2 \leq \ldots \leq \hat\gamma_0\hat\gamma_1 \cdot \widehat M + \alpha_1 + 2\alpha_2\enspace.
\end{aligned}
\end{equation}
Thus, in view of Eq.~\eqref{D1:part1} and the definition of $\hcCgcaleps$, it holds that $(\bv_0, \bv_1) \in \hcCgcaleps$.

\underline{Part II: $\tcCdpdelta \subset \hcCdpdelta$.}
  We fix some $(u, v) \in \tcCdpdelta$. By assumption, $u,v \in [0,1]$ and
  \begin{equation}
  \label{D1:part2}
  \abs{\gamma_1u - \gamma_0v} \leq \tau \, \gamma_0 \gamma_1\cdot \TV(\bfQ^0, \bfQ^1)\enspace.
  \end{equation}
  Furthermore, on $\Omega^{\alpha_1, \alpha_2}$,
  \begin{align*}
    \abs{\gamma_1{u} - \gamma_0{v}} &= \abs{\hgamma_{1}{u} - \hgamma_0{v}+(\gamma_{1}- \hat\gamma_{1})u - (\gamma_0-\hat\gamma_{0}){v}}
     \geq \abs{\hat\gamma_{1}u - \hat\gamma_{0}{v}} - 2\alpha_2\enspace.
  \end{align*}
  In view of Eq.~\eqref{D1:part2} and the second bound of Eq.~\eqref{D1:TVbound}, we conclude that $(u, v) \in \hcCdpdelta$ on
  $\Omega^{\alpha_1, \alpha_2}$.
\end{proof}

\paragraph{Requirement (L2-Hausdorff) holds.}
We bound separately the two expectations appearing in (L2-Hausdorff).
As in the proof above, we omit the dependencies in $T_{r}$ in the notation.

\underline{Part I: bound on $\Exp\bigl[d(\cC, \hat\cC_r)^2\bigr]$.}~~By definition of $d$,
given that we are dealing with Euclidean projections onto a product set and
are bounding square Euclidean distances, we have the decomposition:
\begin{align}
  \nonumber
  \Exp\Big[d(\class{C}, \hat\cC_r)^2\Big]
  =
  \Exp\bigg[\sup_{\bx \in \cC}d(\bx, \hat{\cC}_r)^2\bigg]
  &=
  \Exp\Bigg[\sup_{\bx \in \tcCgcaleps \times \tcCdpdelta}d(\bx, \hcCgcaleps \times \hcCdpdelta)^2\Bigg]\\
  \label{eq:decompd}
  &=
  \Exp\Bigg[\sup_{\bv \in \tcCgcaleps} d(\bv, \hcCgcaleps)^2\Bigg] + \Exp\Bigg[\sup_{(u,v) \in \tcCdpdelta} d\bigl( (u,v), \hcCdpdelta\bigr)^2\Bigg]\enspace.
\end{align}
We start with the first term in the right-hand side of (\ref{eq:decompd}).
As $\hcCgcaleps$ always contains the null vector and $\tcCgcaleps \subseteq B_{\R^{2N}}^{\ell_{1}}$,
\begin{align}
\label{D1:bound_uniform_cal}
\sup_{\bv \in \tcCgcaleps} d(\bv, \hcCgcaleps)^2 \leq
\sup_{\bv \in \tcCgcaleps} \| \bv \|^2 \leq
\sup_{\bv \in \tcCgcaleps} \| \bv \|_1 \leq 1\,.
\end{align}
In addition,  Lemma~\ref{D1:superset_lemma} ensures that  on $\Omega^{\alpha_1, \alpha_2}$  we have $\tcCgcaleps \subset \hcCgcaleps$, and hence
$d(\tcCgcaleps, \hcCgcaleps)=0$,
on $\Omega^{\alpha_1, \alpha_2}$.
Thus, we can write
\begin{align*}
    \Exp\!\left[\sup_{\bv \in \tcCgcaleps}d(\bv, \hcCgcaleps)^2\right]
    & = \Exp\!\left[ \bigl( \ind{\Omega^{\alpha_1, \alpha_2}} + (1 - \ind{\Omega^{\alpha_1, \alpha_2}}) \bigr)
    \sup_{\bv \in \tcCgcaleps}d(\bv, \hcCgcaleps)^2 \right] \\
    & \leq  0+1 - \Prob(\Omega^{\alpha_1, \alpha_2}) \leq \frac{1}{2T_{r}}\enspace,
\end{align*}
where the inequality comes from  the assumption made in Eq.~\eqref{eq:probOmega} combined with Eq.~\eqref{D1:bound_uniform_cal}.

A bound $1/T_r$ on the second term of Eq.~\eqref{eq:decompd} follows similarly, using that
\[
\sup_{(u,v) \in \tcCdpdelta} d\bigl((u,v), \hcCdpdelta \bigr)^2 \leq \sup_{(u,v) \in \tcCdpdelta} u^2 + v^2 \leq 2\,.
\]
Hence, $\Exp\bigl[d(\cC, \hat\cC_r)^2\bigr]\leq 3/(2T_{r})$.

\underline{Part II: bound on $\Exp\bigl[d(\hat{\class{C}}_r, \class{C})^2\bigr]$.}~~
We start in a similar manner:
\begin{align}
  \Exp\Big[d(\hat{\class{C}}_r, \class{C})^2\Big]
  =
  \Exp\bigg[\sup_{\bx \in \hat{\cC}_r}d(\bx, \cC)^2\bigg]
  &=
  \Exp\Bigg[\sup_{\bx \in \hcCgcaleps \times \hcCdpdelta}d(\bx, \tcCgcaleps \times \tcCdpdelta)^2\Bigg]\nonumber\\
  &=
  \Exp\Bigg[\underbrace{\sup_{\bv \in \hcCgcaleps}d(\bv, \tcCgcaleps)^2}_{\leq 1 \ \mbox{\tiny a.s.}} \Bigg]
  + \Exp\Bigg[\underbrace{\sup_{(u,v) \in \hcCdpdelta} d\bigl((u,v), \tcCdpdelta\bigr)^2}_{\leq 2 \ \mbox{\tiny a.s.}}\Bigg]\enspace.
    \label{eq:decompd2}
\end{align}
As in Part I, we start with the first term in the right-side of (\ref{eq:decompd2})  and we split the expectation into two parts
\begin{align}\label{eq:2partEd2}
 \Exp\Bigg[\sup_{\bv \in \hcCgcaleps}d(\bv, \tcCgcaleps)^2 \Bigg]  \leq   \Exp\bigg[ \ind{\Omega^{\alpha_{1},\alpha_{2}}} \sup_{\bv \in \hcCgcaleps}d(\bv, \tcCgcaleps)^2 \bigg]+{1\over 2 T_{r}}\enspace.
\end{align}
Let us upper-bound the right-hand side expectation.
We introduce some local short-hand notation. Given two real numbers $a, b$, we denote by $a \vee b$ and $a\wedge b$
the maximum and minimum between $a$ and $b$, respectively.
We set $\alpha \eqdef \alpha_{1}\vee \alpha_{2}$ and now show that
\begin{equation}
\label{eq:part2dgcal}
\mbox{on } \ \Omega^{\alpha_{1},\alpha_{2}}, \qquad
\sup_{\bv \in \hcCgcaleps} d(\bv, \tcCgcaleps)^2 \leq \alpha^{2/3} \left( \frac{81}{(\gamma_0 \gamma_1)^2} \,\,\vee\,\,
\frac{10}{\gamma_0 \wedge \gamma_1} \right).
\end{equation}
We fix some $\bv = (\bv_0, \bv_1) \in \hcCgcaleps$ and set
$\bv' = (\bv'_0, \bv'_1) := \lambda\bv$ with
\begin{align*}
\lambda&:=1\wedge \Biggl( \biggl( {\gamma_{0}\gamma_{1}\varepsilon\over \gamma_{0}\gamma_{1}\varepsilon+8\alpha} \biggr)
\biggl( {\hgamma_{1}\over \gamma_{1}}\wedge {\hgamma_{0}\over \gamma_{0}} \biggr) \Biggr)\,.
\end{align*}
The fact that $\bv \in \hcCgcaleps$ entails that on $\Omega^{\alpha_{1},\alpha_{2}}$,
\begin{equation}
\label{eq:v1}
  \hgamma_{1}\| \bv_0 \|_1 {+} \hgamma_{0}\| \bv_1 \|_1 \leq \hat \gamma_{0}\hat\gamma_{1}\hat\varepsilon+5\alpha \leq
  \bigl((\gamma_{0}+\alpha)\wedge 1\bigr) \bigl((\gamma_{1}+\alpha)\wedge 1) \bigl((\varepsilon+\alpha)\wedge 1\bigr)+5\alpha
  \leq \gamma_{0}\gamma_{1}\varepsilon+8\alpha\enspace.
\end{equation}
Here, and in what follows, we repeatedly use that $\gamma_{0},\gamma_{1},\varepsilon$
and their estimates all lie in $[0,1]$.
Furthermore, for the above-defined $\bv'$, we can write on $\Omega^{\alpha_{1},\alpha_{2}}$, by definition of $\lambda$,
\begin{align*}
\gamma_{1}\|\bv'_{0}\|_{1}+\gamma_{0} \|\bv'_{1}\|_{1}
&\leq
\lambda \pa{{\gamma_{1}\over \hgamma_{1}}\vee {\gamma_{0}\over \hgamma_{0}}} \underbrace{(\hgamma_{1}\|\bv_{0}\|_{1}+\hgamma_{0} \|\bv_{1}\|_{1})}_{\leq \gamma_{0}\gamma_{1}\varepsilon+8\alpha}
\leq \gamma_{0}\gamma_{1}\varepsilon\,,
\end{align*}
implying that $\bv'\in \tcCgcaleps$.
Thus, $d(\bv, \tcCgcaleps) \leq \|\bv-\bv'\|=(1-\lambda)\|\bv\|$ on $\Omega^{\alpha_{1},\alpha_{2}}$.
Since $\|\bv\|_1\leq 1$, we have $\|\bv\| \leq \sqrt{\|\bv\|_1} \leq 1$.
All in all, we obtained the following upper bound on $\Omega^{\alpha_{1},\alpha_{2}}$:
\begin{align*}
d(\bv, \tcCgcaleps)^2 \leq \|\bv-\bv'\|^2 \leq (1 - \lambda)^2 \wedge \|\bv\|_1\,.
\end{align*}
We now bound separately each term to obtain the bound~\eqref{eq:part2dgcal}.
First, on $\Omega^{\alpha_{1},\alpha_{2}}$, we have
\begin{align*}
1\geq \lambda
& \geq  \biggl( {\gamma_{0}\gamma_{1}\varepsilon\over \gamma_{0}\gamma_{1}\varepsilon+8\alpha} \biggr)
\biggl( {\gamma_{1} - \alpha \over \gamma_{1}}\wedge {\gamma_{0} - \alpha \over \gamma_{0}} \biggr)  \\
& \geq {{\gamma_{0}\gamma_{1}\varepsilon -\alpha\varepsilon (\gamma_{0}\vee \gamma_{1})\over \gamma_{0}\gamma_{1}\varepsilon+8\alpha}}
\geq 1- {9\alpha \over \gamma_{0}\gamma_{1}\varepsilon+8\alpha}
\geq 1-{9\alpha \over \gamma_{0}\gamma_{1}\varepsilon}
\,,
\end{align*}
and thus,
\[
(1-\lambda)^2 \leq \frac{81 \alpha^2}{(\gamma_{0}\gamma_{1}\varepsilon)^2}\,.
\]
Second, for $\|\bv\|_1$, we start from~\eqref{eq:v1} and write
\[
\bigl( \gamma_1 \wedge \gamma_0 - \alpha \bigr) \| \bv \|_1 \leq
\big(\hgamma_{1} \wedge \hgamma_{0}\big) \| \bv \|_1 \leq
\hgamma_{1}\| \bv_0 \|_1 {+} \hgamma_{0}\| \bv_1 \|_1
\leq \gamma_{0}\gamma_{1}\varepsilon+8\alpha\enspace,
\]
from which we get $(\gamma_1 \wedge \gamma_0) \| \bv \|_1 \leq \gamma_{0}\gamma_{1}\varepsilon+9\alpha$,
which in turn yields
\[
\| \bv \|_1 \leq \varepsilon + \frac{9\alpha}{\gamma_0 \wedge \gamma_1}\,.
\]
The bound $(1-\lambda)^2$ is convenient to use when $\epsilon \geq \alpha^{2/3}$,
while the bound on $\| \bv \|_1$ will be used when $\epsilon \leq \alpha^{2/3}$.
When combining them by distinguishing these two cases, the $\wedge$ symbol
needs to be replaced by a $\vee$ symbol, so, on $\Omega^{\alpha_{1},\alpha_{2}}$
\begin{align*}
d(\bv, \tcCgcaleps)^2 \leq (1 - \lambda)^2 \wedge \|\bv\|_1
& \leq \frac{81 \alpha^2}{(\gamma_{0}\gamma_{1}\varepsilon)^2} \,\wedge\,
\biggr( \varepsilon + \frac{9\alpha}{\gamma_0 \wedge \gamma_1} \biggl) \\
& \leq \frac{81 \alpha^{2/3}}{(\gamma_{0}\gamma_{1})^2} \,\vee\,
\biggr( \alpha^{2/3} + \frac{9\alpha}{\gamma_0 \wedge \gamma_1} \biggl) ,
\end{align*}
which entails the claimed bound~\eqref{eq:part2dgcal}, via $\alpha \leq \alpha^{2/3} \leq 1$.
From \eqref{eq:2partEd2} and \eqref{eq:part2dgcal}, we get
\[ \Exp\Bigg[\sup_{\bv \in \hcCgcaleps}d(\bv, \tcCgcaleps)^2 \Bigg]  \leq   {1\over 2 T_{r}}+ C^{\textrm{gr-cal}}_{\gamma_0,\gamma_1} \Exp[\alpha^{2/3}]\enspace,\]
for some constant $C^{\textrm{gr-cal}}_{\gamma_0,\gamma_1}$ only depending on $\gamma_0$ and $\gamma_1$.

The second term of the decomposition \eqref{eq:decompd2} of
$\Exp\Big[d(\hat{\class{C}}_r, \class{C})^2\Big]$ can be handled similarly,
leading to the existence of
a constant $C_{\gamma_0,\gamma_1}$, only depending on $\gamma_0$ and $\gamma_1$, such that
\[
\Exp\Big[d(\hat{\class{C}}_r, \class{C})^2\Big] \leq \frac{3}{2T_r} + C_{\gamma_0,\gamma_1} \Exp[\alpha^{2/3}]\,,
\]
{where the expectation in the right-hand side is due to the fact that $\alpha_1, \alpha_2$ might be data-dependent.}

\underline{Combining Part I and Part II.}~~The bound of Part~II contains an additional term
compared to the one of Part~I. We have thus have proved so far (writing again the dependencies on $T_r$):
\begin{align}
\label{eq:Haus_final}
\max\Bigl\{ \Exp\bigl[d(\hat{\class{C}}_r, \class{C})^2\bigr], \, \Exp\bigl[d(\cC,\hat{\class{C}}_r)^2\bigr] \Bigr\} & \leq
 \frac{3}{2T_r} + C_{\gamma_0,\gamma_1} \Exp[\alpha(T_r)^{2/3}]\,.
\end{align}
To get the desired property (L2-Hausdorff), we only need to make sure that the right hand side of (\ref{eq:Haus_final}) can be upper bounded by $\beta_r^2$ where
$(\beta_r)$ is non-increasing and summable.
Recall that our proof also relied on the assumption~\eqref{eq:probOmega}.
We now illustrate that indeed, $\alpha_1(T_r) \leq 1$ and $\alpha_2(T_r) \leq 1$ may be set in a way such
that all these facts hold. For the sake of simplicity, we provide the illustration
for the case of finite set $\cX$.

\paragraph{Choice of $\alpha_1(T_r)$ and $\alpha_2(T_r)$, part 2: illustration for finite sets $\cX$.}
Based on the $T$--sample $(x_t,s_t)_{1 \leq t \leq T}$ with distribution $\bfQ$, we denote by
\[
N_{s,T} = \sum_{t = 1}^T \ind{s_t = s}
\]
the number of occurrences of the value $s \in \{0,1\}$ of the sensitive context, and
consider the empirical frequencies $\hat \gamma_{0, T} = N_{0,T}/T$ and $\hat \gamma_{1, T} = N_{1,T}/T$
to estimate the frequencies $\gamma_0$ and $\gamma_1$ of the sensitive contexts.

The choice of $\widehat M_T$, and hence, the one of $\alpha_1(T)$, depend heavily on the possibly additional assumptions on the marginal distributions
$\bfQ^0$ and $\bfQ^1$. We illustrate such a choice for the case where $\cX$ is a finite set.
In that case, we may consider the empirical distributions $\hat{\bfQ}_T^0$ and $\hat{\bfQ}_T^1$ for these marginals: for each $s \in \{0,1\}$,
$\hat{\bfQ}_T^s$ is some arbitrary distribution over $\cX$ (say, the uniform distribution) when $N_{s,T} = 0$, and otherwise,
for each $x \in \cX$,
\begin{align*}
\hat{\bfQ}_T^s(x) =
{\frac{1}{N_{s,T}}{\sum_{t = 1}^T \ind{x_t = x, s_t = s}}}\,.
\end{align*}
Then, we consider the plug-in estimate $\widehat M_T:=\TV(\hat\bfQ^0_T,\hat\bfQ^1_T)$ of $\TV(\bfQ^0,\bfQ^1)$.

\underline{Proof of~\eqref{eq:probOmega}, part I.}
We set
\[
\alpha_2(T) = 1 \wedge \sqrt{\log(8T) \over 2 T}
\]
and note that by Hoeffding's inequality (and the fact that we only have two classes and
that probabilities sum up to~$1$), for those $T$ for which $\alpha_2(T) < 1$,
\begin{align}
\nonumber
\P \Bigl( \forall s \in \{0, 1\}, \ \abs{{\hat \gamma_{s, T}} - {\gamma_s}} > \alpha_2(T) \Bigr)
& = \Prob\big( \abs{{\hat \gamma_{0, T}} - {\gamma_0}} > \alpha_2(T) \big) \\
\label{eq:borne:Omega2}
& \leq 2 \exp \bigl( - 2 T \, \alpha_2(T)^2 \bigr) = {1\over 4T}\,.
\end{align}
For $T$ such that $\alpha_2(T) = 1$, the probability above is null, as
$\abs{{\hat \gamma_{s, T}} - {\gamma_s}} \leq 1$ a.s., and therefore, the final $1/(4T)$ bound holds in particular.

\underline{Proof of~\eqref{eq:probOmega}, part II.}
We set $\theta(0)=1$ and  $\theta(n) \eqdef \displaystyle{\sqrt{\frac{|\cX| + \log(8T)}{2n}}}$ for $n \geq 1$,
and define
\[
\alpha_1(T) \eqdef 1 \wedge \bigl( \theta(N_{0, T}) + \theta(N_{1, T}) \bigr) \,.
\]
We now prove that
\begin{equation}
\label{eq:borne:Omega1}
\P \Bigl( \bigl| \widehat M_T - \TV(\bfQ^0, \bfQ^1) \bigr| > \alpha_1(T) \Bigr) \leq \frac{1}{4T}\,.
\end{equation}
The property~\eqref{eq:probOmega} then follows from the bounds
\eqref{eq:borne:Omega2} and~\eqref{eq:borne:Omega1} at $T = T_r$.

Using that $|\widehat M_{T} -  \TV(\bfQ^0, \bfQ^1)|\leq 1$ a.s.
(for the first inequality in the display below) and the triangle inequality
\[
\bigl| \widehat M_T - \TV(\bfQ^0, \bfQ^1) \bigr| \leq
\TV(\bfQ^0, \hat{\bfQ}^0_T)+ \TV(\bfQ^1, \hat{\bfQ}^1_T)
\]
(for the second inequality in the display below), we have
\begin{equation}
\label{eq:TV_bound}
\begin{aligned}
\Prob\ \Bigl( \bigl| \widehat M_T - \TV(\bfQ^0, \bfQ^1) \bigr| > \alpha_1(T) \Bigr)
& = \Prob \Bigl( \bigl| \widehat M_T - \TV(\bfQ^0, \bfQ^1) \bigr| > \theta(N_{0, T}) + \theta(N_{1, T}) \Bigr)  \\
& \leq \Prob \Bigl( \TV(\bfQ^0, \hat{\bfQ}^0_T)+ \TV(\bfQ^1, \hat{\bfQ}^1_T) > \theta(N_{0, T}) + \theta(N_{1, T}) \Bigr) \\
& \leq \sum_{s \in \{0,1\}} \P\Bigl(\TV(\bfQ^s, \hat{\bfQ}^s_T) > \theta(N_{s, T}) \Bigr) \,.
\end{aligned}
\end{equation}
The conclusion~\eqref{eq:borne:Omega1} follows from showing that for each $s \in \{0,1\}$,
\[
\Prob\Bigl(\TV(\bfQ^s, \hat{\bfQ}^s_T) > \theta(N_{s, T}) \Bigr) \leq \frac{1}{8T}\,.
\]

A useful auxiliary result to that end is the following.
Denote by $\hat\bfP^n$ the empirical frequencies of some probability distribution $\bfP$ on $\cX$ based
on a sample of deterministic size $n \geq 1$. Hoeffding's inequality and a union bound over the $\leq 2^{|\cX|}$ subsets
of $\cX$ ensure that for all $\theta\geq 0$,
\begin{equation}
\label{eq:HoeffdingTV}
\Prob\bigl( \TV(\bfP, \hat{\bfP}^n)\geq \theta \bigr) =
\Prob \Bigl( \max_{A\subset \cX} \big(\bfP(A)- \hat{\bfP}^n(A)\big) \geq \theta \Bigr) \leq 2^{|\cX|} \exp(-2n\theta^2)\enspace.
\end{equation}

In our case, note however that the estimators $\hat{\bfQ}_T^s$ at time $T$ are built on a random number
$N_{s,T}$ of samples. We therefore decompose the probability of interest according to the values of $N_{s,T}$:
for each $s \in \{0,1\}$,
\begin{align*}
\Prob \Big( \TV(\bfQ^s, \hat{\bfQ}^s_T) > \theta(N_{s, T}) \Big)
& =
\sum_{n = 0}^T \Prob \Bigl( N_{s, T} = n \ \ \mbox{and} \ \ \TV(\bfQ^s, \hat{\bfQ}^s_T) > \theta(n) \Bigr) \\
& = \sum_{n = 1}^T \Prob \Bigl( N_{s, T} = n \ \ \mbox{and} \ \ \TV(\bfQ^s, \hat{\bfQ}^s_T) > \theta(n) \Bigr) \\
& = \sum_{n = 1}^T \Prob( N_{s, T} = n) \,\, \P \Bigl( \TV(\bfQ^s, \hat{\bfQ}^{s,n}) > \theta(n) \Bigr)\,,
\end{align*}
where the second equality follows from the choice $\theta(0) = 1$ and the fact that a total variation
is always smaller than~$1$, and where the third equality follows by conditional independence
with $\hat{\bfQ}^{s,n}$ denoting the empirical distribution based on a $\bfQ^s$--sample of size $n$.
Substituting the bound~\eqref{eq:HoeffdingTV} and the definition of the $\theta(n)$, we get
\begin{align*}
\Prob \Big( \TV(\bfQ^s, \hat{\bfQ}^s_T) > \theta(N_{s, T}) \Big)
& \leq
2^{|\cX|} \sum_{n = 1}^T \Prob(N_{s, T} = n) \, \exp \bigl(-2n\theta(n)^2 \bigr) \\
& = 2^{|\cX|} \sum_{n = 1}^T \Prob(N_{s, T} = n) \, \exp \bigl( - |\cX| - \log(8T) \bigr) \\
& \leq \underbrace{(2/\mathrm{e})^{|\cX|}}_{\leq 1} \, \frac{1}{8T} \underbrace{\sum_{n = 1}^T \Prob(N_{s, T} = n)}_{\leq 1}
\leq \frac{1}{8T}\,,
\end{align*}
which is exactly what remained to be proven.

\underline{Control of the right-hand side of~\eqref{eq:Haus_final}.}
It involves $\Exp\bigl[\alpha(T_r)^{{2}/{3}}\bigr]$, where
\begin{align*}
\alpha(T_r) = \alpha_1(T_r) \vee \alpha_2(T_r) =
\alpha_1(T_r)
& = 1 \wedge \left(
\sqrt{\frac{|\cX| + \log(8T)}{2 N_{0, T}}} + \sqrt{\frac{|\cX| + \log(8T)}{2 N_{1, T}}} \right) \\
& \leq \sum_{s \in \{0,1\}} 1 \wedge \sqrt{\frac{|\cX| + \log(8T)}{2 N_{s, T}}}\,.
\end{align*}
Now, note that $N_{s, T}$ follows the binomial distribution with parameters $\gamma_s$ and $T$. Thus, for each $s \in \{0, 1\}$,
\begin{align*}
  \Exp\Bigg[\pa{\frac{|\cX| + \log(8T)}{2 N_{s, T}}}^{\! 1/3} \wedge 1\Bigg]
  &\leq
  \Prob(N_{s, T} \leq T \gamma_s / 2) + \pa{\frac{|\cX| + \log(8T)}{\gamma_s T}}^{\! 1/3}\\
  &\leq
  \exp\pa{-\gamma_s^2T/2} + \pa{\frac{|\cX| + \log(8T)}{\gamma_sT}}^{\! 1/3}\enspace,
\end{align*}
where we applied Hoeffding's inequality to get the last bound.
Therefore, recalling that $T_r = 2^r$, the bound of Eq.~\eqref{eq:Haus_final} may be further bounded as:
for all $r \geq 2$,
\[
\frac{3}{2 T_r} + C_{\gamma_0,\gamma_1} \Exp\bigl[\alpha(T_r)^{{2}/{3}}\bigr] \leq
C'_{\gamma_0,\gamma_1} \pa{\frac{r}{2^r}}^{\! 1/3} =: \beta_r^2\,,
\]
for some constant $C'_{\gamma_0,\gamma_1}$ depending only on $\gamma_0$ and $\gamma_1$.
We observe that $\beta_{r}= \sqrt{C'_{\gamma_0,\gamma_1}} \, (r \, 2^{-r})^{1/6}$ is non-increasing for $r \geq 2$
and summable, as required.

We emphasize that, while exact values of $\alpha_1(T_r)$ and $\alpha_2(T_r)$ are needed for the construction of the set-estimate $\hat{\cC}_r$, the knowledge of $\beta_r$ is not required by the algorithm (its choice is required for the sake of the theoretical analysis only).

\subsection{Proof of Theorem~\ref{th:unknown}}
\label{sec:proofthunknown}

We actually prove a more complete and more precise version of Theorem~\ref{th:unknown}.

\begin{theorem}[contains Theorem~\ref{th:unknown}]
\label{th:unknownbis}
Under Assumption~\ref{ass:set_estimation} and the assumptions of Theorem~\ref{thm:approachability_main}, a convex closed set $\cC$,
unknown to the Player, is $\bm$--approachable if and only if Blackwell's condition in Eq.~\eqref{eq:Blackwell_condition} is satisfied.
In this case, the strategy of Eq.~\eqref{eq:strategy_unknown} is an approachability strategy. It achieves the following rates for
$L^2$ convergence: for all $r \geq 1$ and all $t\in [T_{r},T_{r+1}-1]$,
\[
\sqrt{\Exp[d_{t}^2 ]} \leq
\frac{\sqrt{6B^2 + 8B\|\bm\|_{\infty, 2}}}{(\sqrt{2} - 1)\sqrt{t}}
    + \frac{4\|\bm\|_{\infty, 2}}{t}\sum_{t' = 1}^{t - 1} \sqrt{\Exp[\TV^2(\bfQ, \hat{\bfQ}_{t'})]}
      + \frac{4}{t}\sum_{r' = 0}^r T_{r'}\beta_{r' }\enspace.
\]
It also achieves the following rates for almost-sure convergence: for all $r \geq 1$,
\[
\P \! \left( \sup_{t \geq T_r} d_t \geq 2\varepsilon \! \right) \leq \frac{\Xi_r}{\epsilon^2} + \frac{1}{\epsilon^2} \sum_{r' \geq r} \beta^2_{r'}\,,
\]
where $\Xi_r$ is defined in Eq.~\eqref{eq:super_to_zero}, page \pageref{eq:super_to_zero}, and converges to~$0$.
\end{theorem}

Comments after Assumption~\ref{ass:consistent_estimator} explain why the middle term in the $L^2$ bound
vanishes. Assumption~\ref{ass:set_estimation} indicates that the series $(\beta_r)_{r \geq 1}$ is summable, hence
the following sequence of Cesaro averages built on it also vanishes:
\begin{equation}
\label{eq:betacesaro}
\bar{\beta}_r \eqdef \frac{1}{T_{r+1}} \sum_{r' = 0}^r T_{r'}\beta_{r' } \to 0\,.
\end{equation}
Finally, the series $(\beta^2_r)_{r \geq 1}$ is also summable, hence its associated sequence of remainder sums
also vanishes:
\[
\sum_{r' \geq r} \beta^2_{r'} \to 0\,.
\]

We now move to the proof.
We simply note at this stage that the condition
$\|\bv - \Proj_{\hat{\class{C}}_r} (\bv)\| \leq B$ for all $\bv \in \bm(\cA, \cB, \cX, \{0,1\})$
of Assumption~\ref{ass:set_estimation} also holds, by convexity, for
all $\bv$ in the convex hull of $\bm(\cA, \cB, \cX, \{0,1\})$.

\begin{proof}
The proof is required only for the sufficiency, since the necessity was proven in Theorem~\ref{thm:approachability_main}.

Recall that, from the perspective of the Player, the game proceeds in phases lasting from $T_r \eqdef 2^r$ to $T_{r + 1} - 1 \eqdef 2^{r + 1} - 1$.
  For each time $t \in [T_{r}, T_{r+1} - 1]$, the Player uses $\hat{\cC}_r$ as an estimate of the true target set $\cC$, and updates to $\hat{\cC}_{r + 1}$ only at $t = T_{r + 1}$.
  The initial stage of the proof is split into two parts: first, we closely follow the proof of Theorem~\ref{thm:approachability_main} and analyze the game for $t \in [T_{r}, T_{r+1} - 1]$; then, we handle the case of transition from $\hat{\cC}_r$ to $\hat{\cC}_{r+1}$.

   We introduce the following short-hand notation:
  \begin{align*}
    &\hat d_{t} \eqdef \|\bar\bm_{t} - \hat\bc_t\| \qquad\text{and}\qquad
    \Omega_{r}=\big\{\class{C}\subset\hat{\class{C}}_{r}\big\}\enspace.
  \end{align*}
  Note that unlike the quantity of interest $d_t = \|\bar\bm_{t} - \bar{\bc}_t\|$, which is equal to the distance from the average payoff $\bar\bm_{t}$ along the trajectory to the \emph{true} target set $\bar{\bc}_t$, the distance $\hat d_{t}$ is with respect to the currently used estimate $\hat{\class{C}}_{r}$. The key insight of the proof is hidden in the fact that, if $\Omega_{r}$ occurs, then the approachability condition, which is met by $\class{C}$, is also met by the super-set estimate $\hat{\class{C}}_{r}$.

  \paragraph{Convergence in $L^2$.}
  Let us start with the following observation, which relates $d_t$ to $\hat d_{t}$, based on Assumption~\ref{ass:set_estimation}.
We have for $t \in [T_{r}, T_{r+1} - 1]$,
 \begin{align}
 d_{t}=\|\bar\bm_{t}- \Proj_{{\class{C}}}\bar\bm_{t}\|&\leq  \|\bar\bm_{t}- \Proj_{{\class{C}}}\Proj_{\hat{\class{C}}_r}\bar\bm_{t}\|\nonumber\\
 &\leq  \|\bar\bm_{t}- \Proj_{\hat{\class{C}}_r}\bar\bm_{t}\|+ \|\Proj_{\hat{\class{C}}_r}\bar\bm_{t}- \Proj_{{\class{C}}}\Proj_{\hat{\class{C}}_r}\bar\bm_{t}\|\nonumber\\
 &\leq \hat d_{t} + d(\hat{\class{C}}_{r},\class{C}).\label{eq:moving:set}
 \end{align}
Hence, according to the fourth item  of Assumption~\ref{ass:set_estimation} and the $L^2$-triangular inequality, we have
 \begin{align}
    \label{eq:dt_to_hat}
    \sqrt{\Exp[d_{t}^2 ]} \leq \sqrt{\Exp[\hat d_{t}^2]} + \beta_r\enspace.
  \end{align}
 Since $\beta_{r}\to 0$ according to Assumption~\ref{ass:set_estimation}, the latter implies that, if $\Exp[\hat d_{t}^2]\rightarrow 0$, then $\Exp[d_t^2 ] \rightarrow 0$.\smallskip

  As already mentioned, to prove the $L^2$-convergence, we consider two cases. In the first case, we study the evolution of the game withing one phase, that is for $t \in [T_{t}, T_{r+1} - 2]$ -- the case where we project onto $\hat{\cC}_r$. The second case is when $t = T_{r+1} - 1$, that is, when in the next round we are going to update the estimate $\hat{\cC}_r$.\smallskip

  \noindent{\emph{Case $T_{r}\leq t\leq T_{r+1}-2$}:}
  Defining
  \begin{align}
    &Z_{t + 1} \eqdef \scalar{\bar\bm_{t} - \hat\bc_t}{\bm_{t+1} - \int_{\cX \times \cS} \bm \bigl( \bp^x_{t+1}, \bq^{G(x, s)}_{t+1}, x, s \bigr) \d\bfQ(x, s)},\label{eq:unknown_Z}\\
 \textrm{and}\ \   &B_t \eqdef \scalar{\bar\bm_{t} - \hat\bc_t}{\int_{\cX \times \cS} \bm \bigl( \bp^x_{t+1}, \bq^{G(x, s)}_{t+1}, x, s \bigr) \d\bfQ(x, s) - \hat\bc_t}\label{eq:unknown_B},
  \end{align}
  we can write
  \begin{align}
      \label{eq:rec_unknown}
      \hat d_{t+1}^2 \leq \,&\|\bar\bm_{t+1}-\hat\bc_{t}\|\nonumber\\
      \leq \,&\frac{t^2}{(t+1)^2} \hat d_t^2 + \frac{1}{(t+1)^2}\|\bm_{t+1} - \hat\bc_{t}\|^2 + \frac{2t}{(t+1)^2} (Z_{t+1}+B_{t})\enspace.
  \end{align}
  As in the proof of Theorem~\ref{thm:approachability_main}, the main non-standard analysis is connected with the treatment of $B_t$. Observe that thanks to Assumption~\ref{ass:set_estimation}, we always have $|B_t| \leq B(B + 2\|\bm\|_{\infty, 2})$, hence
  \begin{align}
    \label{eq:Bt1}
    B_t \leq B_t\ind{\Omega_r} + B(B + 2\|\bm\|_{\infty, 2})\ind{\Omega^c_r}\enspace.
  \end{align}
  Furthermore, similarly as for Eq.~\eqref{eq:cross_2}, we have on $\Omega_{r}$
  \begin{align}
    B_t
    &\leq
    4 \|\bm\|_{\infty, 2}\, \hat d_t \cdot \TV(\bfQ, \hat{\bfQ}_t) \nonumber\\
    &\phantom{\leq}+
    \min_{(\bp^x)_{x \in \class{X}}}\max_{(\bq^{G(x, s)})_{(x, s) \in \class{X} \times \cS}}\scalar{\bar\bm_{t} - \hat\bc_t}{\int_{\cX \times \cS} \bm \bigl( \bp^x, \bq^{G(x, s)}, x, s \bigr) \d\bfQ(x, s) - \hat\bc_t}\label{eq:Bt2}\enspace.
  \end{align}
  Since, by definition of $\Omega_r$, we have the inclusion $\cC \subset \hat{\cC}_r$ on  $\Omega_r$,  Blackwell's condition~\eqref{eq:Blackwell_condition} implies that, on $\Omega_r$,
  \begin{align*}
    \forall (\bq^{G(x, s)})_{(x, s) \in \class{X} \times \{0, 1\}}\,\, \exists (\bp^{x})_{x \in \class{X}} \quad\text{s.t.}
\quad \int_{\class{X} \times \cS} \bm\big(\bp^x, \bq^{G(x, s)}, x, s\big) \d\bfQ(x, s) \in \cC \subset \hat{\class{C}}_r\enspace.
  \end{align*}
  The first item of Assumption~\ref{ass:set_estimation} requires $\hat{\class{C}}_r$ to be closed convex almost surely. Hence, using the property of Euclidean projection onto a convex closed set, in conjunction with von Neumann's minmax theorem,  we conclude that, on $\Omega_r$, it holds that
  \begin{align*}
    \min_{(\bp^x)_{x \in \class{X}}}\max_{(\bq^{G(x, s)})_{(x, s) \in \class{X} \times \cS}}\scalar{\bar\bm_{t} - \hat\bc_t}{\int_{\cX \times \cS} \bm \bigl( \bp^x, \bq^{G(x, s)}, x, s \bigr) \d\bfQ(x, s) - \hat\bc_t} \leq 0\enspace.
  \end{align*}
  The above inequality, combined with Eqs.~\eqref{eq:rec_unknown}--\eqref{eq:Bt2}, yields
  \begin{equation}
  \label{eq:unknown_hatd}
  \begin{aligned}
    \hat d_{t+1}^2
    \leq \,&\frac{t^2}{(t+1)^2} \hat d_t^2 + \frac{1}{(t+1)^2}\|\bm_{t+1} - \hat\bc_{t}\|^2\\
    &+ \frac{2t}{(t+1)^2} \pa{Z_{t+1} + 4 \|\bm\|_{\infty, 2}\, \hat d_t \cdot \TV(\bfQ, \hat{\bfQ}_t)\ind{\Omega_r} + B(B + 2\|\bm\|_{\infty, 2})\ind{\Omega^c_r}}\enspace.
  \end{aligned}
  \end{equation}
  Since $(Z_t)_{t \geq 1}$ is martingale difference (by the same arguments as for the proof of Theorem~\ref{thm:approachability_main}), taking expectations from both sides of the above inequality, in conjunction with the condition on $\Prob(\Omega_r)$ of Assumption~\ref{ass:set_estimation} and the Cauchy-Schwartz inequality, yields
  \begin{align*}
    \Exp[\hat d_{t+1}^2]
    \leq\,& \frac{t^2}{(t+1)^2} \Exp[\hat d_t^2] + \frac{B^2}{(t+1)^2}\\
    &+ \frac{2t}{(t+1)^2} \pa{4 \|\bm\|_{\infty, 2}  \sqrt{\Exp[\hat d^2_t]} \cdot \sqrt{\Exp[\TV^2(\bfQ, \hat{\bfQ}_t)]} + \frac{B(B + 2\|\bm\|_{\infty, 2})}{2T_r}}\enspace.
  \end{align*}
  We deduce from the above that, for all $t \in [T_r, T_{r+1} - 2]$, since $t / (2 T_r) \leq 1$,
  \begin{align}
    \Exp[\hat d_{t+1}^2]
    \leq \,&\frac{t^2}{(t+1)^2} \Exp[\hat d_t^2] + \frac{3B^2 + 4B\|\bm\|_{\infty, 2}}{(t+1)^2}\nonumber\\
    &+ \frac{2t}{(t+1)^2} \pa{4 \|\bm\|_{\infty, 2}  \sqrt{\Exp[\hat d^2_t]} \cdot \sqrt{\Exp[\TV^2(\bfQ, \hat{\bfQ}_t)]}}\enspace.
    \label{eq:recursion:unknown}
  \end{align}
  Applying Lemma~\ref{lem:recursion} with $t^* = T_r$, $K = 3B^2 + 4B\|\bm\|_{\infty, 2}$, and $\delta_t = 4 \|\bm\|_{\infty, 2} \cdot \sqrt{\Exp[\TV^2(\bfQ, \hat{\bfQ}_t)]}$, we obtain that, for all $t \in [T_r, T_{r+1} - 1]$,
  \begin{equation}
    \label{eq:rec_case1_final}
  \begin{aligned}
    \sqrt{\Exp[\hat d_{t}^2]} \leq \frac{\sqrt{(3B^2 + 4B\|\bm\|_{\infty, 2})(t - T_r)}}{t} &+ \frac{4}{t}\sum_{t' = T_r}^{t - 1}\|\bm\|_{\infty, 2} \cdot \sqrt{\Exp[\TV^2(\bfQ, \hat{\bfQ}_{t'})]}\\
    &+ \frac{T_r}{t}\sqrt{\Exp[\hat d_{T_r}^2]}\enspace.
  \end{aligned}
  \end{equation}

  \noindent{\emph{Case} $t = T_{r+1}-1$:} In this case, when passing from $t$ to $t+1$, the Player updates the estimate of the target set $\cC$, which incurs additional price.
  In particular, the established recursion in Eq~\eqref{eq:rec_case1_final} does not hold, since by definition $\hat d_{T_{r+1}} = \|\bar\bm_{T_{r+1}} - \hat\bc_{T_{r+1}}\|$, where $\hat\bc_{T_{r+1}}$ is the projection onto $\hat\cC_{r+1}$. However, note that the argument of the first case still holds if we fix the set onto which we project.
    More formally, the inequality (\ref{eq:recursion:unknown}) still holds at $t=T_{r+1}-1$, if we
  replace $\hat d_{T_{r+1}}$ in the left-hand side  by
  \begin{align*}
    \tilde d_{T_{r+1}} \eqdef \|\bar\bm_{T_{r+1}} - \Proj_{\hat\cC_r}(\bar\bm_{T_{r+1}})\|\enspace.
  \end{align*}
  Hence $\sqrt{\Exp[\tilde d_{T_{r+1}}^2]}$ is smaller than the right-hand side of Eq.~\eqref{eq:rec_case1_final} with $t = T_{r+1}$.
    Applying the same argument as in (\ref{eq:moving:set}), and applying Minkowski's inequality, we get
  \begin{align*}
    \sqrt{\Exp[\hat d_{T_{r+1}}^2]}
    &=
    \sqrt{\Exp\big[\|\bar\bm_{T_{r+1}} - \Proj_{\hat\cC_{r + 1}}(\bar\bm_{T_{r+1}})\|^2\big]}\\
    &\leq
    \sqrt{\Exp\big[\|\bar\bm_{T_{r+1}} - \Proj_{\hat\cC_{r}}(\bar\bm_{T_{r+1}})\|^2\big]} + \sqrt{\Exp\big[d(\hat\cC_{r},\hat\cC_{r+1})^2\big]}\\
    &=
    \sqrt{\Exp[\tilde{d}_{T_{r + 1}}^2]} + \sqrt{\Exp\big[d(\hat\cC_{r},\hat\cC_{r+1})^2\big]} \enspace.
  \end{align*}
  Recalling that the bound in Eq.~\eqref{eq:rec_case1_final} holds
  for $\sqrt{\Exp[\tilde d_{T_{r+1}}^2]}$, and using the above derived relation, we get for all $r \geq 0$
  \begin{equation}
    \label{eq:rec_case2_almost_final}
  \begin{aligned}
    \sqrt{\Exp[\hat d_{T_{r + 1}}^2]} \leq \,&\frac{\sqrt{(3B^2 + 4B\|\bm\|_{\infty, 2})(T_{r + 1} - T_r)}}{T_{r + 1}} + \frac{4}{T_{r + 1}}\sum_{t' = T_r}^{T_{r + 1} - 1}\|\bm\|_{\infty, 2} \cdot \sqrt{\Exp[\TV^2(\bfQ, \hat{\bfQ}_{t'})]}\\
    &+ \frac{T_r}{T_{r + 1}}\sqrt{\Exp[\hat d_{T_r}^2]} + \sqrt{\Exp\big[d(\hat\cC_{r},\hat\cC_{r+1})^2\big]}\enspace.
  \end{aligned}
  \end{equation}
  Multiplying Eq.~\eqref{eq:rec_case2_almost_final} by $T_{r+1}$ on both sides and rearranging, we deduce that for all $r \geq 0$
  \begin{align*}
    \parent{T_{r + 1}\sqrt{\Exp[\hat d_{T_{r + 1}}^2]} - T_r\sqrt{\Exp[\hat d_{T_r}^2]}} \leq  \,&\sqrt{(3B^2 + 4B\|\bm\|_{\infty, 2})(T_{r + 1} - T_r)}\\
    &+ 4\sum_{t' = T_r}^{T_{r + 1} - 1}\|\bm\|_{\infty, 2} \cdot \sqrt{\Exp[\TV^2(\bfQ, \hat{\bfQ}_{t'})]}\nonumber\\
    & + T_{r+1}\sqrt{\Exp\big[d(\hat\cC_{r},\hat\cC_{r+1})^2\big]}\enspace.
  \end{align*}
  Summing up the above inequalities over $r \geq 0$, and using the fact that, by Assumption~\ref{ass:set_estimation}, $\hat{d}_{1} \leq B$, we obtain, with the convention $T_{-1}=0$,
  \begin{equation}
  \label{eq:rec_unknown_before_final}
  \begin{aligned}
    \sqrt{\Exp[\hat d_{T_r}^2]} \leq  \,&\sqrt{3B^2 + 4B\|\bm\|_{\infty, 2}}\,\frac{1}{T_r}\sum_{r' = 0}^r\sqrt{T_{r'} - T_{r' - 1}}\\
    &+ 4\|\bm\|_{\infty, 2}\frac{1}{T_r}\sum_{t' = 1}^{T_{r} - 1} \sqrt{\Exp[\TV^2(\bfQ, \hat{\bfQ}_{t'})]}
     + \frac{1}{T_r}\sum_{r' = 1}^rT_{r'} \sqrt{\Exp\big[d(\hat\cC_{r'-1},\hat\cC_{r'})^2\big]}\enspace .
  \end{aligned}
  \end{equation}
  To conclude the convergence in $L^2$,
 we observe that $d(\hat\cC_{r'-1},\hat\cC_{r'})\leq d(\hat\cC_{r'-1},\cC)+d(\cC,\hat\cC_{r'})$ and hence,
 the triangle inequality for $L^2$-norms and
 Assumption~\ref{ass:set_estimation} yield
 \begin{equation}
 \label{eq:deuxbetar}
 \sqrt{\Exp\big[d(\hat\cC_{r'-1},\hat\cC_{r'})^2\big]}\leq  \sqrt{\Exp\big[d(\hat\cC_{r'-1},\cC)^2\big]}+ \sqrt{\Exp\big[d(\cC,\hat\cC_{r'})^2\big]}\leq \beta_{r'-1}+\beta_{r'}\leq 2 \beta_{r'-1}.
 \end{equation}
   Substituting the above bound in Eq.~\eqref{eq:rec_unknown_before_final} (and reindexing, using that $T_{r'} = 2 T_{r'-1}$),
   we get for all $r\geq 1$
\begin{equation}
\label{eq:rec_unknown_final}
\begin{aligned}
    \sqrt{\Exp[\hat d_{T_{r}}^2 ]}  \leq  \,&\frac{\sqrt{3B^2 + 4B\|\bm\|_{\infty, 2}}}{{T_r}}\sum_{r' = 0}^r\sqrt{T_{r'} - T_{r' - 1}}
    + 4\|\bm\|_{\infty, 2}{\frac{1}{T_r}\sum_{t' = 1}^{T_{r} - 1} \sqrt{\Exp[\TV^2(\bfQ, \hat{\bfQ}_{t'})]}}\\
    &  + \frac{4}{T_r}\sum_{r' = 0}^{r-1}T_{r'}\beta_{r' }\\
 \leq  \,&\frac{\sqrt{3B^2 + 4B\|\bm\|_{\infty, 2}}}{(\sqrt{2} - 1)\sqrt{T_r}}
    + 4\|\bm\|_{\infty, 2}{\frac{1}{T_r}\sum_{t' = 1}^{T_{r} - 1} \sqrt{\Exp[\TV^2(\bfQ, \hat{\bfQ}_{t'})]}}
      + \frac{4}{T_r}\sum_{r' = 0}^{r-1}T_{r'}\beta_{r' }\enspace.
  \end{aligned}
  \end{equation}
The $(\sqrt{2}-1)\sqrt{T_r}$ factor in the denominator of the first term of the final bound was obtained as follows:
\[
\sum_{r' = 0}^r\sqrt{T_{r'} - T_{r' - 1}} =
1 + \sum_{r'=1}^r \sqrt{2^{r'-1}} = 1 + \frac{\sqrt{2^r}-1}{\sqrt{2}-1} \leq \frac{\sqrt{2^r}}{\sqrt{2}-1} = \frac{\sqrt{T_r}}{\sqrt{2}-1}\,.
\]
Combining the first inequality of the two inequalities of
Eqs.~\eqref{eq:rec_unknown_final} with Eq.~(\ref{eq:rec_case1_final}), we get for all $r \geq 1$ and $t\in [T_{r},T_{r+1}-1]$,
\begin{equation}
  \label{eq:hatd0}
  \begin{aligned}
    \sqrt{\Exp[\hat d_{t}^2 ]} \leq
\,& \frac{\sqrt{3B^2 + 4B\|\bm\|_{\infty, 2}}}{{t}}\pa{\sqrt{t-T_{r}}+\sum_{r' = 0}^r\sqrt{T_{r'} - T_{r' - 1}} }  \\
    & + \frac{4\|\bm\|_{\infty, 2}}{t}\sum_{t' = 1}^{t - 1} \sqrt{\Exp[\TV^2(\bfQ, \hat{\bfQ}_{t'})]}
      + \frac{4}{t}\sum_{r' = 0}^{r-1}T_{r'}\beta_{r' }\\
     \leq  \,&\frac{\sqrt{6B^2 + 8B\|\bm\|_{\infty, 2}}}{(\sqrt{2} - 1)\sqrt{t}}
    + \frac{4\|\bm\|_{\infty, 2}}{t}\sum_{t' = 1}^{t - 1} \sqrt{\Exp[\TV^2(\bfQ, \hat{\bfQ}_{t'})]}
      + \frac{4}{t}\sum_{r' = 0}^{r-1}T_{r'}\beta_{r' }\enspace,
  \end{aligned}
  \end{equation}
 where the last inequality follows from
  \[
  \sqrt{t-T_{r}} +
  \sum_{r' = 0}^r\sqrt{T_{r'} - T_{r' - 1}} \leq
  \sum_{r' = 0}^{r+1} \sqrt{T_{r'} - T_{r' - 1}} \leq \frac{\sqrt{T_{r+1}}}{\sqrt{2}-1} = \frac{\sqrt{2T_{r}}}{\sqrt{2}-1}\,.
  \]
  Combining inequality~\eqref{eq:hatd0} with (\ref{eq:dt_to_hat}), i.e., adding $\beta_r$ to the bound above,
  and using $T_r/t \leq 1$, we conclude the stated bound for the $L^2$ convergence.

\paragraph{Almost-sure convergence.}
We observe that, according to (\ref{eq:moving:set}), by union bounds, Markov's inequality,
and the third item of Assumption~\ref{ass:set_estimation}, we have
 \begin{align*}
 \Prob\!\cro{\sup_{t\geq T_{r}}d_{t}\geq 2\varepsilon} &\leq   \Prob\!\cro{\sup_{t\geq T_{r}}\hat d_{t}\geq \varepsilon} +
 \Prob\!\cro{\sup_{r'\geq r} d(\hat\cC_{r'},\cC)\geq \varepsilon}\\
& \leq     \Prob\!\cro{\sup_{t\geq T_{r}}\hat d_{t}\geq \varepsilon} + {1\over \varepsilon^2} \sum_{r'\geq r} \beta_{r'}^2.
 \end{align*}
In what follows, we bound $\Prob\!\cro{\sup_{t\geq T_{r}} \hat d_{t}\geq \varepsilon}$ by $\Xi_r/\epsilon^2$, where $\Xi_r$ is defined
in Eq.~\eqref{eq:super_to_zero}.

As in Theorem~\ref{thm:approachability_main}, we introduce a super-martingale $S_t$ bounding $\hat d^2_{t}$ and whose expectation
vanishes; however, the analysis is more involved here due to additional difficulties connected to handling the switches between regimes.
More precisely, let us define, for $t\in[T_{r},T_{r}-1]$,
 \begin{align*}
 V_{t}=\,& \frac{B^2}{(t+1)^2}+ \frac{2t}{(t+1)^2} \pa{ 4 \|\bm\|_{\infty, 2}\, \hat d_t \cdot \TV(\bfQ, \hat{\bfQ}_t)+ B(B + 2\|\bm\|_{\infty, 2}) \, \ind{\Omega^c_r}}\\&+2B \, d(\hat C_{r},\hat C_{r+1}) \, \ind{t=T_{r+1}-1}\enspace.
 \end{align*}
Using the above defined $V_t$, we additionally introduce the process
\begin{align}
  \label{eq:unknown_super}
  S_{T}=\hat d_{T}^2+\sum_{t\geq T} \Exp[V_{t}|H_{T}]\enspace.
\end{align}
We observe that, by Assumption~\ref{ass:set_estimation} and the triangle inequality,
\begin{align*}
  \hat d_{T_{r+1}}^2 - \tilde d_{T_{r+1}}^2
  &=
  \bigl( \hat d_{T_{r+1}} - \tilde d_{T_{r+1}} \bigr) \bigl( \overbrace{\hat d_{T_{r+1}} + \tilde d_{T_{r+1}}}^{\leq 2B} \bigr) \\
  &\leq 2B \Bigl( \bigl \|\bar{\bm}_{T_{r+1}}- \Proj_{\hat\cC_{r+1}}(\bar{\bm}_{T_{r+1}})\|-\|\bar{\bm}_{T_{r+1}}-\Proj_{\hat\cC_r}(\bar{\bm}_{T_{r+1}}) \bigr\| \Bigr)\\
  &\leq 2B \Bigl( \bigl\|\bar{\bm}_{T_{r+1}}- \Proj_{\hat\cC_{r+1}}(\Proj_{\hat\cC_r}(\bar{\bm}_{T_{r+1}}))\|-\|\bar{\bm}_{T_{r+1}}-\Proj_{\hat\cC_r}(\bar{\bm}_{T_{r+1}}) \bigr\| \Bigr)\\
    &\leq2B \, \bigl\|\Proj_{\hat\cC_r}(\bar{\bm}_{T_{r+1}}) - \Proj_{\hat\cC_{r+1}}(\Proj_{\hat\cC_r}(\bar{\bm}_{T_{r+1}})) \bigr\|\\
&  \leq 2B\cdot d(\hat C_{r},\hat C_{r+1})\enspace.
\end{align*}
Thus, in view of Eq.~\eqref{eq:unknown_hatd}, and recalling that the right-hand side of Eq.~\eqref{eq:unknown_hatd} bounds
rather $\tilde d_{T_{r+1}}$ at $t=T_{r+1}-1$, the following recursive relation holds for any $t\in[T_{r},T_{r+1}-1]$:
\begin{align*}
\hat d_{t+1}^2&\leq  \hat d_{t}^2 +V_{t}+{2t\over (t+1)^2} Z_{t+1}.
\end{align*}
Recalling that $\Exp[Z_{t+1}|H_t]=0$, we deduce
\begin{align*}
\Exp[S_{T+1}|H_{T}] &= \Exp[\hat d_{T+1}^2|H_{T}] + \sum_{t\geq T+1} \Exp[V_{t}|H_{T}]
\leq \hat d_{T}^2+ \sum_{t\geq T} \Exp[V_{t}|H_{T}]=S_{T},
\end{align*}
  which means that $(S_{T})_{T\geq 1}$ is a super-martingale.

  Since, by definition of $S_T$, it holds that
   $\hat d^2_{T} \leq S_T$,  Doob's maximal inequality for non-negative super-martingales (Lemma~\ref{lem:super-martingale}) gives
    \[
      \Prob\parent{\sup_{t \geq T_{r}} \hat d_{t} \geq \varepsilon}
       \leq \Prob\parent{\sup_{t \geq T_{r}} S_{t} \geq \varepsilon^2}
    \leq \frac{\E[S_{T_{r}}]}{\varepsilon^2}\,.
    \]
    It only remains to bound $\E[S_{T_{r}}]$ by $\Xi_r$.

Note that by the Cauchy-Schwarz inequality and the bound of Eq.~\eqref{eq:deuxbetar},
\[
\E\bigl[d(\hat \cC_{r'},\hat\cC_{r'+1})\bigr] \leq
\sqrt{\Exp\big[d(\hat\cC_{r'},\hat\cC_{r'+1})^2\big]}
\leq 2 \beta_{r'}\,.
\]
Thanks to this inequality, to $t \, \P(\Omega^c_r) \leq t/(2T_r) \leq 1$ for $t\in[T_{r},T_{r+1}-1]$,
and other manipulations that are standard by now,
the expectation of the sum appearing in the definition~\eqref{eq:unknown_super} of the super-martingale $S_T$ can be bounded as
  \begin{equation}
    \label{eq:bound_sumV}
 \begin{aligned}
 \sum_{t\geq T_{r}} \Exp\cro{V_{t}}  \leq& \sum_{t\geq T_{r}} {3B^2+4B \|\bm\|_{\infty, 2}\over (t+1)^2} +   \sum_{t\geq T_{r}} {8t  \|\bm\|_{\infty, 2} \over (t+1)^2} \sqrt{\Exp[\hat d_{t}^2]} \sqrt{\Exp[\TV^2(\bfQ, \hat{\bfQ}_{t})]}\\
 &+ 2B \sum_{r'\geq r} \E\bigl[d(\hat \cC_{r'},\hat\cC_{r'+1})\bigr] \\
 \leq &\  {3B^2+4B \|\bm\|_{\infty, 2}\over T_{r}}+4B \sum_{r'\geq r}\beta_{r'}
   +\sum_{t\geq T_{r}} {8  \|\bm\|_{\infty, 2} \over t} \sqrt{\Exp[\hat d_{t}^2]} \sqrt{\Exp[\TV^2(\bfQ, \hat{\bfQ}_{t})]}.
 \end{aligned}
 \end{equation}
  To bound the right hand side of the above inequality, we observe that for $t\geq T_{r}$, by Eq.~\eqref{eq:hatd0},
 we have
 \begin{align*}
  \sqrt{\Exp[\hat d_{t}^2]} \leq \frac{\sqrt{6B^2 + 8B\|\bm\|_{\infty, 2}}}{(\sqrt{2} - 1)\sqrt{t}}
    & + 4\|\bm\|_{\infty, 2}\overbrace{\max_{t\geq T_{r}}\frac{1}{t}\sum_{t' = 1}^{t - 1} \sqrt{\Exp[\TV^2(\bfQ, \hat{\bfQ}_{t'})]}}^{=:\bar\Delta^*_{T_{r}}} \\
      & + 4\underbrace{\max_{r''\geq r} \frac{1}{T_{r''}}\sum_{r' = 0}^{r''-1}T_{r'}\beta_{r' }}_{=:\beta^*_{r}}.
 \end{align*}
Substituting the above bound into Eq.~\eqref{eq:bound_sumV}, using
$\smash{\sum_{t \geq T_r} t^{-3/2} \leq 2/\sqrt{T_r - 1}}$ and $\TV(\bfQ, \hat{\bfQ}_{t'})\leq 1$, we obtain
\begin{align*}
 \sum_{t\geq T_{r}} \Exp\cro{V_{t}}   \leq &\  {3B^2+4B \|\bm\|_{\infty, 2}\over T_{r}}+ 16  \|\bm\|_{\infty, 2}\frac{\sqrt{6B^2 + 8B\|\bm\|_{\infty, 2}}}{(\sqrt{2} - 1)\sqrt{T_{r}-1}}+4B \sum_{r'\geq r}\beta_{r'}\\
 &  \quad +32  \|\bm\|_{\infty, 2} \bigl( \|\bm\|_{\infty, 2}\bar\Delta^*_{T_{r}}+\beta^*_{r} \bigr) \sum_{t\geq T_{r}} {1 \over t}  \sqrt{\Exp[\TV^2(\bfQ, \hat{\bfQ}_{t'})]}\enspace.
  \end{align*}
Finally, we take into account the definition of the super martingale $S_T$ in Eq.~\eqref{eq:unknown_super} and
the upper bound of Eq.~\eqref{eq:rec_unknown_final}, which we square, using that $(x+y+z)^2 \leq
2x^2 + 2(y + z)^2$. Doing so, and performing some crude boundings for the sake of readability,
we get the final bound $\Exp[S_{T_{r}}] \leq \Xi_r$, where
\begin{equation}
  \label{eq:super_to_zero}
    \begin{aligned}
 \Xi_r \eqdef &\  (1+2(\sqrt{2}-1)^{-2}) {3B^2+4B \|\bm\|_{\infty, 2}\over T_{r}}+ 16  \|\bm\|_{\infty, 2}\frac{\sqrt{6B^2 + 8B\|\bm\|_{\infty, 2}}}{(\sqrt{2} - 1)\sqrt{T_{r}-1}}+4B \sum_{r'\geq r}\beta_{r'}\\
 &  +32 \bigl( \|\bm\|_{\infty, 2}\bar\Delta^*_{T_{r}}+\beta^*_{r} \bigr)
 \pa{ \|\bm\|_{\infty, 2}\bar\Delta^*_{T_{r}}+\beta^*_{r}+ \|\bm\|_{\infty, 2}\sum_{t\geq T_{r}} {1 \over t}  \sqrt{\Exp[\TV^2(\bfQ, \hat{\bfQ}_{t'})]}}.
  \end{aligned}
  \end{equation}
As indicated in Eq.~\eqref{eq:betacesaro}, the Cesaro averages $\bar{\beta}_r$, which are positive, tend to~$0$; therefore,
we also have $\beta^*_{r} \to 0$.
For similar reasons, and as already noted for Theorem~\ref{thm:approachability_main},
the term $\bar\Delta^*_{T_{r}}$ also vanishes under Assumption~\ref{ass:consistent_estimator}.
The latter also implies that the final term in Eq.~\eqref{eq:super_to_zero} vanishes.
Other terms clearly vanish or were already discussed for the $L^2$-convergence.
All in all, $\Xi_r \to 0$, as claimed.
\end{proof}

\end{document}